\documentclass[11pt]{article}

\usepackage{fullpage} 
\usepackage{comment}
\usepackage{amsthm}
\usepackage{tikz}
\usetikzlibrary{arrows.meta}
\usepackage{pgfplots}

\usepackage{csquotes}

\usepackage{bigints}
\usepackage{amsmath,amssymb}
\usepackage{xcolor}
\usepackage{tcolorbox}

\usepackage{thm-restate}

\usepackage{natbib}
\bibpunct{(}{)}{;}{a}{,}{,}
 
\usepackage[utf8]{inputenc} % allow utf-8 input
\usepackage[T1]{fontenc}    % use 8-bit T1 fonts
\usepackage[colorlinks=true, linkcolor= blue!70!black, citecolor=blue!70!black,urlcolor=black,breaklinks=true]{hyperref}

\usepackage{comment}
\newcommand{\xhdr}[1]{\vspace{2mm}\noindent{\bf {#1.}\ }}

\usepackage{url}            % simple URL typesetting
\usepackage{booktabs}       % professional-quality tables
\usepackage{amsfonts}       % blackboard math symbols
\usepackage{nicefrac}       % compact symbols for 1/2, etc.
\usepackage{microtype}      % microtypography

\usepackage{mathtools, bm}

\usepackage{algorithm}
\usepackage{algpseudocode}

\usepackage{subcaption}
\usepackage{multirow}
\usepackage{float}
\usepackage{xspace}

\usepackage{enumitem}

\usepackage{xfrac}

\usepackage{cleveref}

\theoremstyle{plain}
\newtheorem{theorem}{Theorem}[section]
\newtheorem{lemma}[theorem]{Lemma}
\newtheorem{claim}[theorem]{Claim}

\newtheorem{proposition}[theorem]{Proposition}
\newtheorem{corollary}[theorem]{Corollary}% reset theorem numbering for each chapter

\theoremstyle{plain}
\newtheorem{definition}{Definition}[section] % definition numbers are dependent on theorem numbers
\newtheorem{example}[definition]{Example}

\theoremstyle{plain}

\newtheorem{assumption}{Assumption}

\newcounter{relctr} %% <- counter for relations
\everydisplay\expandafter{\the\everydisplay\setcounter{relctr}{0}} %% <- reset every eq
\AtBeginDocument{} %% <- store original definition

\newcommand{\squishlist}{
\begin{list}{{$\bullet$}}
{\setlength{\itemsep}{3pt}      \setlength{\parsep}{1pt}
\setlength{\topsep}{1pt}       \setlength{\partopsep}{0pt}
\setlength{\leftmargin}{1em} \setlength{\labelwidth}{1em}
\setlength{\labelsep}{0.5em} } }
\newcommand{\squishend}{  \end{list}}

\newcommand{\ECE}[2][]{\text{\bf ECE}\ifthenelse{\not\equal{}{#1}}{_{#1}}{}\!\left[{\def\givenn{\middle|}#2}\right]}

\newcommand{\ECEnorm}[2][]{\text{\bf ECE}_t\ifthenelse{\not\equal{}{#1}}{_{#1}}{}\!\left[{\def\givenn{\middle|}#2}\right]}

\newcommand{\Payoff}[2][]{\text{\bf Payoff}\ifthenelse{\not\equal{}{#1}}{_{#1}}{}\!\left[{\def\givenn{\middle|}#2}\right]}
\newcommand{\Reg}[2][]{\text{\bf REG}\ifthenelse{\not\equal{}{#1}}{_{#1}}{}\!\left[{\def\givenn{\middle|}#2}\right]}

\newcommand{\OPTobj}[2][]{\cc{ OPT}\ifthenelse{\not\equal{}{#1}}{_{#1}}{}\!\left[{\def\givenn{\middle|}#2}\right]}

\newcommand{\OBJ}[2][]{\text{\bf OBJ}\ifthenelse{\not\equal{}{#1}}{_{#1}}{}\!\left[{\def\givenn{\middle|}#2}\right]}

\newcommand{\supp}{\cc{supp}}

\newcommand{\numData}{n}

\newcommand{\outcomeSpace}{\mathcal{Y}}

\newcommand{\actionSpace}{\mathcal{A}}

\newcommand{\senderU}{u}
\newcommand{\receiverU}{v}

\newcommand{\outcome}{y}
\newcommand{\bestr}{a^*}
\newcommand{\Bern}{\cc{Bern}}

\newcommand{\predicSpace}{\mathcal{F}}
\newcommand{\ECEbudget}{\text{ECE budget}}
\newcommand{\prior}{\lambda}
\newcommand{\truemean}{\theta}

\newcommand{\prediction}{p}
\newcommand{\bayesprediction}{q}

\newcommand{\predictor}{f}

\renewcommand{\predictor}{\tilde{f}}
\newcommand{\marginalpredictor}{f}
\newcommand{\bayespredictor}{\tilde{g}}
\newcommand{\bayesmarginalpredictor}{g}
\newcommand{\actionpredictor}{\pi}

\newcommand{\MPC}{\cc{MPC}}
\newcommand{\newpredictor}{\predictor\primed}
\newcommand{\newmarginalpredictor}{f\primed}

\newcommand{\newbayesmarginalpredictor}{\bayesmarginalpredictor\primed}
\newcommand{\truePosterior}{\kappa}

\newcommand{\actionpredictortoPos}{\prediction}
\newcommand{\bias}{\posteriorErr}

\newcommand{\thresholdState}{i^*}

\newcommand{\convexEnv}{\Gamma}
\newcommand{\miscali}{\chi}
\newcommand{\newmiscali}{\miscali\primed}

\newcommand{\newq}{q\primed}

\newcommand{\caliBudgetDual}{\alpha}

\newcommand{\Lagrange}{\mathcal{L}}

\renewcommand{\d}{{\mathrm{d}}}

\newcommand{\event}{{event}}
\newcommand{\priormean}{\bar{\truemean}}

\newcommand{\dualvarsupply}{\beta_{\cc{supp}}}
\newcommand{\dualvarmiscal}{\beta_{\miscali}}

\newcommand{\dualvarbayesup}{\bar{\beta}_\bayesmarginalpredictor}
\newcommand{\dualvarbayes}{\beta_\bayesmarginalpredictor}
\newcommand{\dualvarbayesmean}{\bar{\beta}_\prior}
\newcommand{\dualvarMPC}{\beta_{\cc{MPC}}}

\newcommand{\poly}{\cc{poly}}

\newcommand{\normexponent}{t}
\newcommand{\tnorm}{\ell_t}

\newcommand{\action}{a}

\newcommand{\convexcombin}{b}
\newcommand{\convexcombinbf}{\mathbf{\convexcombin}}

\newcommand{\optpredictor}{\predictor^*}
\newcommand{\predicSpaceEI}{\predicSpace^{\rm EI}}

\newcommand{\bayespredictorCDF}{G}
\newcommand{\priorCDF}{\Lambda}

\newcommand{\BP}{\cc{BP}}
\newcommand{\posteriorErr}{b}

\newcommand{\stateSpace}{\Theta}
\newcommand{\state}{\truemean}

\newcommand{\priorBP}{\prior^\BP}
\newcommand{\stateSpaceBP}{\Theta^\BP}
\newcommand{\senderUBP}{\senderU^\BP}
\newcommand{\receiverUBP}{\receiverU^\BP}
\newcommand{\actionSpaceBP}{\actionSpace^\BP}
\newcommand{\signal}{\sigma}
\newcommand{\signalSpace}{\Sigma}

\newcommand{\bipoolingprob}{\omega}
\newcommand{\mixratio}{\bipoolingprob}

\newcommand{\numAction}{m}
\newcommand{\caliBudget}{\varepsilon}

\newcommand{\thresholdi}{i^\dagger}

\newcommand{\indirectsenderU}{U}

\newcommand{\lowq}{\prediction_{\cc{L}}}
\newcommand{\highq}{\prediction_{\cc{H}}}
\newcommand{\TPF}{{true expected outcome function}}

\newcommand{\lowx}{x_{\cc{L}}}
\newcommand{\highx}{x_{\cc{H}}}

\newcommand{\setShigh}{\mathcal{S}_{\cc{H}}}
\newcommand{\setSlow}{\mathcal{S}_{\cc{L}}}

\newcommand{\sgselect}{\xi}

\newcommand{\discrepara}{\delta}
\newcommand{\discreparaBudget}{\delta_0}
\newcommand{\scaledownfrac}{\discrepara\primed}
\newcommand{\discresupp}{\mathcal{S}_{\discrepara}}

\newcommand{\numdeltas}{S}

\newcommand{\barECEnorm}[2][]{\overline{\text{\bf ECE}}_t\ifthenelse{\not\equal{}{#1}}{_{#1}}{}\!\left[{\def\givenn{\middle|}#2}\right]}

\newcommand{\discontinuityset}{\mathcal{Z}}

\newcommand{\bed}{\text{bi-event}}

\newcommand{\contributionratio}{{contribution ratio}}
\newcommand{\contributionratios}{{contribution ratios}}

\newcommand{\optsignaling}{\actionpredictor^*}
\newcommand{\optbias}{\bias^*}
\newcommand{\signalingscheme}{\pi}

\newcommand{\pthreshold}{\prediction\primed}

\newcommand{\optpoolingpredic}{\pthreshold}
\newcommand{\optpoolingprob}{t}

\newcommand{\condition}{\,\mid\,}

\newcommand{\prob}[2][]{\mathbb{P}\ifthenelse{\not\equal{}{#1}}{_{#1}}{}\!\left[{\def\givenn{\middle|}#2}\right]}
\newcommand{\indicator}[2][]{\mathbf{1}\ifthenelse{\not\equal{}{#1}}{_{#1}}{}\!\left\{{\def\givenn{\middle|}#2}\right\}}
\newcommand{\expect}[2][]{\mathbb{E}\ifthenelse{\not\equal{}{#1}}{_{#1}}{}\!\left[{\def\givenn{\middle|}#2}\right]}

\newcommand{\abs}[1]{\left| #1 \right|}

\newcommand{\cc}[1]{\ensuremath{\mathsf{#1}}} % algorithm class

\DeclareMathOperator*{\argmax}{arg\,max}

\newcommand{\R}{\mathbb{R}}

\newcommand{\primed}{^{\dagger}}
\newcommand{\doubleprimed}{^{\ddagger}}
\usepackage{fix-cm}
\allowdisplaybreaks

\title{Persuasive Calibration}
\author{
Yiding Feng\thanks{Hong Kong University of Science and Technology. Email: {\tt ydfeng@ust.hk}}
\and
Wei Tang\thanks{Chinese University of Hong Kong. Email: {\tt weitang@cuhk.edu.hk}}
}
\date{}

\begin{document}

\maketitle
\begin{abstract}

We introduce and study the \emph{persuasive calibration} problem, where a principal aims to provide trustworthy predictions about underlying events to a downstream agent to make desired decisions.
We adopt the standard calibration framework that regulates predictions to be unbiased conditional on their own value,
and thus, they can reliably be interpreted at the face value by the agent.
Allowing a small calibration error budget, we aim to answer the following question: what is and how to compute the optimal predictor under this calibration error budget, especially when there exists incentive misalignment between the principal and the agent?
We focus on standard $\tnorm$-norm Expected Calibration Error (ECE) metric.

We develop a general framework by viewing predictors as post-processed versions of perfectly calibrated predictors.
Using this framework, we first characterize the structure of the optimal predictor. 
Specifically, when the principal's utility is event-independent and for $\ell_1$-norm ECE, we show: (1) 
the optimal predictor is over-(resp.\ under-) confident for high (resp.\ low) true expected outcomes, while remaining perfectly calibrated in the middle;
(2) the miscalibrated predictions exhibit a collinearity structure with the principal's utility function.
On the algorithmic side, we provide a FPTAS for computing approximately optimal predictor for general principal utility and general $\tnorm$-norm ECE.
Moreover, for the $\ell_1$- and $\ell_\infty$-norm ECE, we provide polynomial-time algorithms that compute the exact optimal predictor.

\end{abstract}

\setcounter{page}{0}
\thispagestyle{empty}
 \clearpage

{\hypersetup{linkcolor=black}\tableofcontents}
\setcounter{page}{0}
\thispagestyle{empty}
 \clearpage

\section{Introduction}
\label{sec:intro}

Over the past decade, machine learning models have grown remarkably powerful, with recent large-scale models (like LLMs) often containing billions of parameters and providing good predictions in a wide variety of domains. 
Yet these models/algorithms are frequently regarded as black-box systems: their internal mechanisms remain proprietary, or are considered trade secrets and thus, are not observed by outsiders.
Consequently, end users and downstream decision-makers are often left to trust predictions without transparent insight into how those predictions are generated.

A prominent approach to bolstering trust in such predictions is {\em calibration} \citep{D-82,FV-98,HKRR-18}.
A calibrated predictor, in simple terms, regulates that predicted probabilities align with the true (conditional) probability of the outcome. For example, predictions of ``70\% likelihood'' materialize approximately 70\% outcome realizations of the time. 
Calibration thus provides a guarantee that predictions are trustworthy and can be reliably taken at face value for use in downstream decision-making.
Indeed, it is well established that agents who 
naively
best respond to calibrated predictions (i.e., those with low calibration error -- a measure of how far the predictions deviate from the true conditional probability) achieve low regret \citep{FV-98,KLST-23,HPY-23,RS-24,HW-24}, {whose rate is even no larger than the inherent calibration error of the predictor.}
At the same time, however, the principal -- the entity designing or deploying the prediction model -- may have incentives that differ from those of the downstream agent who relies on these predictions. For example, the principal might wish to skew predictions ever so slightly to influence the agent's actions in a beneficial (to the principal) way.

The above tension naturally leads to the following question: 
If a small calibration error is permissible, 
how might the principal design an optimal predictor that balances the principal's objectives while preserving the agent's trust by not miscalibrating so much?
To study this question, in this paper, we formulate and analyze the {\em persuasive calibration} problem. 
Specifically, we introduce a framework in which the principal aims to provide predictions that must remain acceptably calibrated while ensuring that the downstream agent, who only sees the final predictions, would maintain sufficient trust to best-respond to them.
We aim to understand the following two question:
\begin{displayquote}
    \emph{Can we characterize the optimal predictor under a calibration error budget, especially when there exists incentive misalignment between the principal and the agent? In addition, can we compute this optimal predictor or an approximately optimal predictor in polynomial time?}
\end{displayquote}

\subsection{Our Contributions and Techniques}

In this paper, we provide compelling answers to both of these questions. In a nutshell, we introduce the \emph{persuasive calibration} problem which captures the interaction between the principal who decides the predictor and the agent who takes the action based on the prediction. In the first part of the paper, we focus on a canonical setting where the principal's utility is event-independent and the predictor is evaluated based on the most classic $\ell_1$-norm calibration error. For this setting, we provide a comprehensive characterization of the optimal predictor. In the second part of the paper, we move to the more general setting where the principal's utility can be arbitrary and the predictor can be evaluated by any $\tnorm$-norm calibration error. For this rich setting, we develop an FPTAS (fully polynomial-time approximation scheme) algorithm for all instances of our problem, and an optimal polynomial-time algorithm for $\ell_1$-norm or $\ell_\infty$-norm calibration errors. Below, we discuss our model, results, and techniques in detail.

\subsection*{The Principal-Agent Model (\Cref{sec:prelim})} 
The \emph{persuasive calibration} problem models the interaction between a predictor designer (the principal, she) and a downstream decision maker (the agent, he). In this setting, there is an underlying randomized binary outcome, which is unobserved by both players. However, the principal observes a realized \emph{event} (e.g., a feature or context) that is correlated with the outcome and uses this information to generate a prediction. This prediction is a scalar value representing the probability that the binary outcome equals one.

The agent, whose utility depends on both his action and the realized outcome, \emph{trusts} the prediction provided by the principal
and chooses an action which  best responses to the provided prediction.
In contrast, the principal's utility depends not only on the agent's action but also on the realized event and outcome, and may be misaligned with the agent's utility. The principal’s goal is to design a predictor that maximizes her own utility, given the agent's best response.

We adopt the calibration framework to ensure the trustworthiness of the predictors --
that is, conditional on a given prediction, the expected true outcome should be close to the predicted probability. In this work, we focus on the classic $\tnorm$-norm expected calibration error (ECE), which quantifies the expected $\tnorm$-norm difference between the prediction and the true expected outcome. More formally, the \emph{persuasive calibration} problem seeks the optimal predictor that maximizes the principal's payoff, subject to an exogenously specified ECE budget.

\subsection*{Result I: Structural Characterization of the Optimal Predictor (\Cref{sec:optimal structure})}
We begin our exploration of the persuasive calibration problem under the \emph{$\ell_1$-norm ECE} (also known as the $K_1$ ECE), which is the most standard calibration error metric. 
We focus on a canonical setting where the principal's \emph{indirect utility}---mapping predictions to her utilities under the assumption that the agent best responds---is event-independent.\footnote{This assumption is common in various principal-agent models \citep[e.g.,][]{DM-19,LR-20,FTX-22,ABSY-23,CD-23,FHT-24}.}  

\xhdr{The structure results of the optimal predictor} As the first main result of this paper, we provide a comprehensive characterization of the optimal predictor under any exogenously specified $K_1$ ECE budget. In \Cref{subsec:characterization:K1 ECE}, we establish two key properties: (1) the \emph{miscalibration structure} where the optimal predictor is over- (resp.\ under-) confident for high (resp.\ low) true expected outcomes while remaining perfectly calibrated in the middle; and (2) the \emph{payoff structure}, where the miscalibrated predictions in the optimal predictor exhibit a collinearity pattern with the principal's indirect utility function.

We illustrate both properties in \Cref{fig:general}. In this figure, the black solid curve represents the principal's indirect utility, while each blue dot corresponds to a prediction generated by the optimal predictor. Specifically, the x-coordinate of each blue dot indicates the prediction value, and the y-coordinate represents the corresponding payoff for the principal. The brown squares denote the true expected outcomes conditional on the predictions.  

Our \emph{miscalibration structure} (in \Cref{thm:event-inde}) suggests that the prediction space $[0,1]$ can be partitioned into three sub-intervals:  
(1) the \emph{under-confidence interval} $[0, \lowq]$, where predictions (blue dots) are lower (smaller x-coordinate) than their corresponding true expected outcomes (brown squares);  
(2) the \emph{over-confidence interval} $[\highq,1]$, where predictions (blue dots) are higher (larger x-coordinate) than their corresponding true expected outcomes (brown squares); and  
(3) the \emph{perfectly calibrated interval} $[\lowq, \highq]$, where predictions (blue dots) exactly match their corresponding true expected outcomes (brown squares). 
We remark that this three-interval miscalibration pattern appears to be a well-observed phenomenon in the machine learning community. For example, \cite{GPSW-17} show that uncalibrated ResNet tends to be overconfident in its predictions, while \cite{KFE-18} show that neural networks consistently generate both underestimated and overestimated predictions/forecasts at the two extremes of the probability spectrum.\footnote{Also see Figure 4 in \citet{GPSW-17} and Figures 2, 3 in \citet{KFE-18}.} By interpreting the loss function in these machine learning tasks as the principal's utility, our results provide a theoretical explanation for this practical observation.

Our \emph{payoff structure} (in \Cref{thm:event-inde} and \Cref{prop:convex env}) suggests that all blue dots (representing predictions and their induced payoffs for the principal) lie on a \emph{symmetric linear-tailed convex function $\convexEnv$} (blue curve), which is pointwise (weakly) greater than the principal's indirect utility (black curve). Moreover,  all blue dots in the under-confidence interval $[0, \lowq]$ (resp.\ over-confidence interval $[\highq, 1]$) are collinear, with the slopes of these two linear tails having the same absolute value.

Besides the characterization above, we also (i) establish (almost) tight bounds on the number of predictions that can be generated by the optimal predictor -- overall, per event, and per true expected outcome (\Cref{prop:num predictions}), and (ii) provide a verification tool (\Cref{prop:opt verify}) that enables us to theoretically verify the optimality of a given predictor.

\begin{figure}[t]
  \centering
  \begin{subfigure}{0.45\textwidth}
    \centering
    \resizebox{\linewidth}{!}{\begin{tikzpicture}[baseline=(current axis.south)]
\def\yminval{-0.05}
\begin{axis}[
    axis lines=left,
    axis line shift=0pt,
    width=10cm,
    height=8cm,
    xtick=\empty, 
    ytick=\empty,
    xmin=0, xmax=1,
    ymin=\yminval, ymax=0.65,
]

% Define the parameters
\pgfmathsetmacro{\k}{13}
\pgfmathsetmacro{\m}{0.5}
\pgfmathsetmacro{\v}{0.9}
\pgfmathsetmacro{\kthree}{13}
\pgfmathsetmacro{\mthree}{0.3}
\pgfmathsetmacro{\vthree}{1.73}

% Define F(x) and F3(x)
\pgfmathdeclarefunction{F}{1}{\pgfmathparse{1/(1+exp(-\k*(#1-\m)))^(1/\v)}}
\pgfmathdeclarefunction{Fthree}{1}{\pgfmathparse{1/(1+exp(-\kthree*(#1-\mthree)))^(1/\vthree)}}

% Plot U_0(x) for x <= 0.25
% \addplot[black, thick, domain=0:0.25, samples=100] {0.4*(1-F(2*x+0.5)) - (1-F(1))};

\def\Uzero(#1){0.4*(1 - F(4*((#1)))) - (1 - F(1)) + 0.2}  % EXAMPLE 
\def\Uone(#1){0.2*(1 - F(4*(#1) - 1) - (1 - F(1)))}              % EXAMPLE
\def\Utwo(#1){0.2*F(4*(#1) -2) - 0.3*F(0)}                   % EXAMPLE
\def\Uthree(#1){0.3*(Fthree(4*((#1)) -3) - Fthree(0)) + 0.2} % EXAMPLE

% -------------------------------------------------
% 3) Plot each piece over its sub-interval
% -------------------------------------------------
% U_0(x) for x in [0, 0.25]
\addplot[domain=0:0.25, samples=200, thick, black] { \Uzero(x) };

% U_1(x) for x in [0.25, 0.5]
\addplot[domain=0.25:0.5, samples=200, thick, black] { \Uone(x) };

% U_2(x) for x in [0.5, 0.75]
\addplot[domain=0.5:0.75, samples=200, thick, black] { \Utwo(x) };

% U_3(x) for x in [0.75, 1]
\addplot[domain=0.75:1, samples=200, thick, black] { \Uthree(x) };

\def\xpone{0.0814}
\def\xptwo{0.3477}
\def\xpthree{0.49}
\def\xpfour{0.521245} 
\def\xpfive{0.663}
\def\xpsix{0.853}

\def\xqone{0.1562}
\def\xqtwo{0.361}
\def\xqthree{0.49}
\def\xqfour{0.521245}
\def\xqfive{0.663}
\def\xqsix{0.81}

\def\alphanega{-1.5115}
\def\alphapos{1.5115}

\addplot[
  only marks,
  mark=square*,
  mark size=2.2pt,
  color={brown},
  opacity=0.9
] coordinates {(\xqone, {\Uzero(\xqone)})
               (\xqtwo, {\Uone(\xqtwo)})
               (\xqthree, {\Uone(\xqthree)})
               (\xqfour, {\Utwo(\xqfour)})
               (\xqfive, {\Utwo(\xqfive)})
               (\xqsix, {\Uthree(\xqsix)})
               };

\addplot[
  only marks,
  mark=*,
  mark size=2pt,
  color=blue
] coordinates {(\xpone, {\Uzero(\xpone)})
               (\xptwo, {\Uone(\xptwo)})
               (\xpthree, {\Uone(\xpthree)})
               (\xpfour, {\Utwo(\xpfour)})
               (\xpfive, {\Utwo(\xpfive)})
               (\xpsix, {\Uthree(\xpsix)})
               };

\addplot[
    domain=0.048309:0.456648,
    samples=2,
    line width=1.5pt,
    blue,
    opacity=0.6
] {0.569983 - \alphapos * (x - \xpone)};               

\addplot[
    domain=0.521245:0.751199,
    samples=2,
    line width=1.5pt,
    blue,
    opacity=0.6
] {0.172916668316 + 1.21786989679 * (x - \xqfive)}; 

\addplot[
    domain=0.751199:0.976076,
    samples=2,
    line width=1.5pt,
    blue,
    opacity=0.6
] {0.43409824403 + \alphapos * (x - \xpsix)};

\draw[dashed, thick, gray]
  (axis cs:0.456648,\yminval) -- (axis cs:0.456648, 0.002803);

\draw[dashed, thick, gray]
  (axis cs:0.751199,\yminval) -- (axis cs:0.751199, 0.280331);

\draw[dashed, thick, gray]
  (axis cs:\xptwo,\yminval) -- (axis cs:\xptwo, { \Uone(\xptwo)});
  
\draw[dashed, thick, gray]
  (axis cs:\xpsix,\yminval) -- (axis cs:\xpsix, { \Uthree(\xpsix)});

\addplot[
  domain=0.456648:0.5,
  samples=200,
  line width=1.5pt,
  blue,
  opacity=0.6
] { \Uone(x) };

\addplot[
  domain=0.5:0.521245,
  samples=200,
  line width=1.5pt,
  blue,
  opacity=0.6
] { \Utwo(x) };

\node[anchor=north] at (axis cs:\xptwo-0.04,0) {$\lowq$};
\node[anchor=north] at (axis cs:\xpsix-0.04,0) {$\highq$};

\node[anchor=north] at (axis cs:0.456648-0.04,0) {$\lowx$};
\node[anchor=north] at (axis cs:0.751199-0.04,0) {$\highx$};

\end{axis}
\end{tikzpicture}}
    \caption{}
    \label{fig:general}
  \end{subfigure}
\hspace{0.8cm}
  \begin{subfigure}{0.45\textwidth}
    \centering
    \resizebox{\linewidth}{!}{\begin{tikzpicture}[baseline=(current axis.south)]
\def\yminval{0}
\def\ymaxval{0.85}

\def\xpone{0.0814}
\def\xptwo{0.3377}
\def\xpthree{0.49}
\def\xpfour{0.521245} 
\def\xpfive{0.663}
\def\xpsix{0.823}

\def\xqone{0.1562}
\def\xqtwo{0.361}
\def\xqthree{0.49}
\def\xqfour{0.521245}
\def\xqfive{0.663}
\def\xqsix{0.78}

\begin{axis}[
    axis lines=left,
    axis line shift=0pt,
    % xtick={\xptwo, \xpthree},  % positions for q_L and q_H
    % xticklabels={$q_L$, $q_H$},
    width=10cm,
    height=8cm,
    xtick=\empty, 
    ytick=\empty,
    xmin=0, xmax=1,
    ymin=\yminval, ymax=\ymaxval,
]

\node[anchor=south west, rotate=0] at (rel axis cs:0.01,0.865) {$\truePosterior(\prediction)$};

\def\diagonal(#1){(#1)}
\addplot[domain=0:\ymaxval, samples=200, thick, black] { \diagonal(x) };

\def\alphanega{-1.5115}
\def\alphapos{1.5115}

\addplot[
  only marks,
  mark=square*,
  mark size=2.2pt,
  color={brown},
  opacity=0.9
] coordinates {(\xpone, \xqone)
                (\xptwo, \xqtwo)
                (\xpthree, \xqthree)
                (\xpfour, \xqfour)
                (\xpfive, \xqfive)
                (\xpsix, \xqsix)
               };

\addplot[
  only marks,
  mark=*,
  mark size=2pt,
  color=blue
] coordinates {(\xpone, \xpone)
                (\xptwo, \xptwo)
                (\xpthree, \xpthree)
                (\xpfour, \xpfour)
                (\xpfive, \xpfive)
                (\xpsix, \xpsix)
               };

% \draw[dashed, thick, gray]
%   (axis cs:\xpone,0) -- (axis cs:\xpone,\xqone);

\draw[dashed, thick, gray]
  (axis cs:\xptwo,0) -- (axis cs:\xptwo,\xqtwo);

% \draw[dashed, thick, gray]
%   (axis cs:\xpthree,0) -- (axis cs:\xpthree,\xqthree);

% \draw[dashed, thick, gray]
%   (axis cs:\xpfour,0) -- (axis cs:\xpfour,\xqfour);

% \draw[dashed, thick, gray]
%   (axis cs:\xpfive,0) -- (axis cs:\xpfive,\xqfive);

\draw[dashed, thick, gray]
  (axis cs:\xpsix,0) -- (axis cs:\xpsix,\xpsix);

\node[anchor=north] at (axis cs:\xptwo-0.04,0.06) {$\lowq$};
\node[anchor=north] at (axis cs:\xpsix-0.04,0.06) {$\highq$};

\end{axis}
\end{tikzpicture}}
    \caption{}
    \label{fig:reliable-diagram-general}
  \end{subfigure}
  \caption{
  In both figures, blue dots are the predictions $\prediction$ generated from the optimal predictor $\optpredictor$, 
  brown squares are the corresponding true expected outcomes $\truePosterior(\prediction)$ conditional on the prediction $\prediction$, 
  \Cref{fig:general} is an illustration of \Cref{thm:event-inde} -- black solid curve is indirect utility $\indirectsenderU$ and blue curve is the convex function $\convexEnv$ in \Cref{prop:convex env}.
  \Cref{fig:reliable-diagram-general} is the reliable diagram (a standard visualization of the miscalibration structure) -- black solid line is diagonal line.} 
  \label{fig:main}
\end{figure}
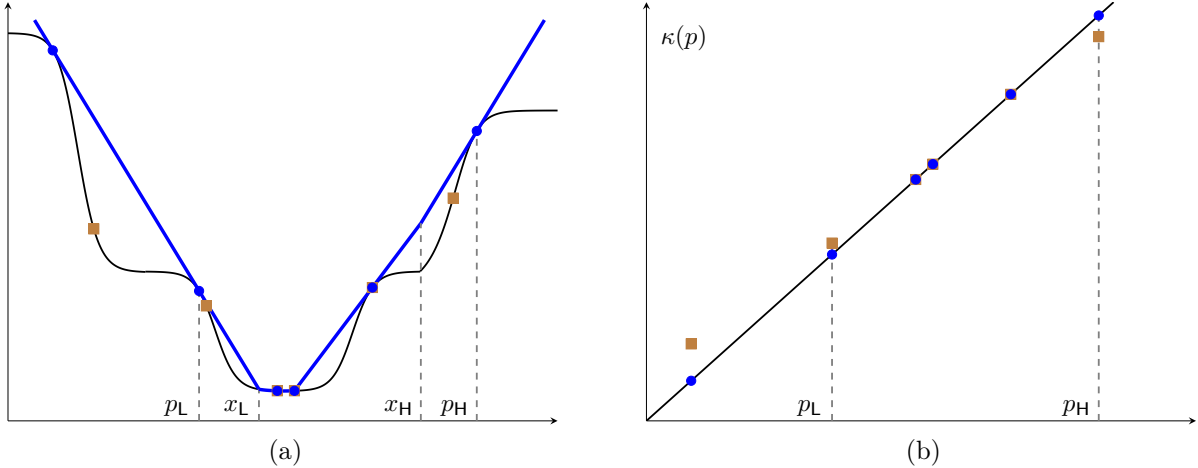

\xhdr{Two-step framework to generate optimal predictors} 
To prove the aforementioned structural results, we develop a two-step framework for the principal to design and analyze the optimal predictor. At a high level, the framework treats the predictor as a post-processed version of a perfectly calibrated predictor (i.e., zero calibration error). Specifically, we introduce the \emph{post-processing plan}, which ``miscalibrates'' a perfectly calibrated predictor into an imperfectly calibrated one (\Cref{defn:miscal plan,defn:miscal procedure} and \Cref{prop:miscal procedure error guarantee}). 
Using this concept, identifying the optimal predictor subject to an ECE budget constraint can be reformulated as optimizing a pair consisting of a perfectly calibrated predictor and an ``ECE-budget-feasible'' miscalibration plan. 
This reformulation allows us to leverage well-established properties of perfectly calibrated predictors in the analysis.

A particularly interesting subclass of post-processing plans is the event-independent ones.
However, not every imperfectly calibrated predictor can be obtained by event-independently post-processing a perfectly calibrated predictor (see \Cref{fig:hierarch} for an illustration and \Cref{ex:special predictor} for an example). 
Perhaps interestingly, 
we find that when the principal's indirect utility is event-independent, 
it suffices to consider event-independent post-processing plans to obtain the optimal predictor (\Cref{prop:prog equivalence}). 
Furthermore, optimizing over perfectly calibrated predictors and event-independent miscalibration plans can be characterized as a linear program (see \ref{eq:decoupled opt}).

\xhdr{Primal-dual analysis of \ref{eq:decoupled opt}}
Utilizing the two-step framework, we know that the optimal predictor can be transformed into an optimal solution of \ref{eq:decoupled opt} and vice versa. Therefore, proving the structural results for the optimal predictor is equivalent to proving their analogs for the optimal solution in \ref{eq:decoupled opt}. By developing both primal-based and duality-based arguments, we establish the aforementioned structural results in \Cref{subsec:proofs event-inde}. Notably, the miscalibration structure is primarily proved through a primal-based argument, which leverages the structure (i.e., mean-preserving contraction (MPC)) of perfectly calibrated predictors, while the payoff structure is primarily proved using duality-based arguments. In particular, the symmetric linear-tailed convex function $\convexEnv$ (see \Cref{fig:general} and \Cref{defn:symmetric linear-tailed}) arises as a consequence of the optimal dual solutions. For example, the slope of its linear tail corresponds to the optimal dual variable associated with the ECE budget constraint in \ref{eq:decoupled opt}.

\subsection*{Result II: FPTAS for Computing Approximately Optimal Predictor (\Cref{sec:fptas})}
Moving forward, we explore the algorithmic aspects of the persuasive calibration problem. We consider a more general setting where the principal's utility can be arbitrary and the calibration metrics can be general $\tnorm$-norm ECE with any $\normexponent \geq 1$.
As the second main result of this paper, we provide a FPTAS (\Cref{algo:fptas} and \Cref{thm:fptas}) for computing an approximately optimal predictor subject to any exogenously specified ECE budget.

To obtain a time-efficient algorithm for computing an approximately optimal predictor, a natural approach is to formulate the principal's problem as a computationally tractable optimization program. However, due to the calibration error constraint defined on the predictor $\predictor$, this approach is not straightforward.\footnote{In \Cref{apx:Failure}, we explain the failure of two natural optimization program formulations that directly use the predictor as the decision variable and optimize over the entire space of feasible predictors.}

Our FPTAS contains two technical ingredients: it combines a generalized two-step framework, which extends the two-step approach (used in our structural characterization) from the event-independent setting to the general setting, with a carefully designed discretization scheme.

\xhdr{Generalized two-step framework}
For the general setting where the principal's utility can depend on the realized event, our two-step framework with the \emph{event-independent} post-processing plan established for the structural results may not be sufficient. Motivated by this, we introduce a generalized two-step framework in \Cref{subsec:generalizeed two-step framework}. However, if we consider an event-dependent post-processing plan that \emph{fully decouples} the miscalibration across all events, additional non-linear constraints---seemingly intractable---are required. To bypass this issue, we introduce a generalized two-step framework with an \emph{(event-dependent) bi-event post-processing plan} (\Cref{defn:bed miscal plan,defn:bed miscal procedure}), which allows us to formulate an infinite-dimensional (but tractable) linear program~\ref{eq:decoupled opt general}. Our bi-event post-processing plan is rich enough to ensure that there always exists an optimal predictor that can be generated within our generalized framework (\Cref{prop:LP general}).\footnote{The foundational reason why the bi-event post-processing plan suffices is that our current model assumes a binary outcome. In fact, we believe that our FPTAS admits a relatively straightforward generalization to an extended model with a constant number of outcomes, which we leave for future exploration.}

\xhdr{Instance-dependent two-layer discretization \& rounding-based analysis} 
Equipped with an infinite-dimensional linear program \ref{eq:decoupled opt general}, a natural approach is to consider its discretization. However, common discretization schemes, such as uniform discretization, are not directly applicable. Specifically, due to the ECE budget constraint, standard discretization may either render the discretized LP infeasible or significantly degrade the objective value.  

To overcome this challenge, in \Cref{subsec:discretization and rounding}, we introduce a carefully designed discretization scheme (\Cref{defn:discre}) that is instance-dependent and has a two-layer structure. First, {given a discretization parameter $\discrepara>0$,} our scheme constructs a non-uniform $\Theta(\discrepara)$-net, where the discretized points depend on the problem primitives, ensuring feasibility of the discretized program. Then, for each point in this $\Theta(\discrepara)$-net, we introduce a finer $\Theta(\discreparaBudget)$-net within a small neighborhood. Here, $\discreparaBudget$ depends not only on the discretization precision $\discrepara$ but also on the ECE budget $\caliBudget$ and the norm exponent parameter $\normexponent$. This second layer, which adapts to the ECE budget, allows us to carefully control changes in ECE when transitioning from a solution in the infinite-dimensional LP to one in the discretized LP. Crucially, the second-layer discretized points are introduced only locally, ensuring that the overall discretization set remains of polynomial size.  

Finally, we develop a rounding scheme (\Cref{algo:rounding}) to analyze both the feasibility and the objective value of the discretized program \ref{eq:decoupled opt general discre}. We believe that both our discretization scheme and rounding argument might be of independent interest in the algorithmic information design literature.

\subsection*{Result III: Computing Optimal Predictor under $K_1$ or $K_\infty$ ECE (\Cref{sec:polytime})}

As the third main result of this paper, we focus on the most standard ECE metrics, namely the $K_1$ ECE (i.e., $\ell_1$-norm ECE). We present a polynomial-time algorithm (\Cref{algo:poly-time} and \Cref{thm:opt general}) that computes an optimal predictor while adhering to any exogenously specified $K_1$ ECE budget. Furthermore, we demonstrate that the same algorithm is also applicable to the $K_\infty$ ECE (i.e., $\ell_\infty$-norm ECE).

For $K_1$ or $K_\infty$ ECE, it is not difficult to express the calibration constraint as a linear constraint of the predictor. However, to obtain an optimal algorithm with polynomial running time, we also need to restrict the continuous space of predictions to a discrete polynomial-size set. To bypass this challenge, we build on an idea from the algorithmic information design literature. 

\xhdr{Revelation principle and (Bayesian) persuasion with signal-dependent bias}
In the classic Bayesian persuasion problem \citep{KG-11}, the \emph{revelation principle} assures that it is sufficient to consider signaling schemes that recommend an action. Hence, instead of searching over a (possibly) continuous signal space, it suffices to construct a signaling scheme with a signal space equal to the agent's action space, which is polynomial-sized. However, such an approach seems difficult to apply to our model, as the design space in our model is restricted to predictions rather than arbitrary signals, and the agent in our model follows a much simpler behavior---i.e., always trusting the prediction---rather than engaging in strategic reasoning. Therefore, it is unclear a priori whether a revelation principle can be established.

Although at first glance, our persuasive calibration problem and the classic Bayesian persuasion problem may seem different, we demonstrate that the persuasive calibration problem can also be interpreted as a new variant of the Bayesian persuasion problem, which we refer to as \emph{(Bayesian) persuasion with signal-dependent bias}. In this model, the receiver (equivalent to the agent) (1) has a utility function that depends linearly on the payoff-relevant state, and (2) exhibits a signal-dependent bias when updating their belief, rather than updating their belief in a fully Bayesian manner.
The sender (equivalent to the principal) now determines both the signaling scheme and the bias assignment for the receiver, subject to an exogenously specified constraints (Eqn.~\eqref{eq:aggregated IC}) on the total bias.\footnote{By allowing the sender to design the bias assignment, our persuasion problem with signal-dependent bias shares similarities with the literature on $\caliBudget$-IC mechanism design \citep{BBHM-05, HL-15, CDW-12,BBC-24}. Notably, even when selling a single item to a single buyer, the optimal $\caliBudget$-IC mechanism lacks a simple structural characterization \citep{BBC-24}. In contrast, due to the equivalence between the persuasive calibration and this variant, we obtain a clean structural characterization for this persuasion problem.} In the special case where the total bias is restricted to zero, this variant reduces to the classic Bayesian persuasion problem.

With this equivalent interpretation, it becomes relatively more straightforward for us to identify a suitable revelation principle (\Cref{lem:revelation}) and formulate the problem of optimizing the pair consisting of the signaling scheme and the bias assignment (equivalent to the predictor) as an action-recommendation program \ref{eq:opt general actionRem} (\Cref{cor:actionRem BP}). Notably, both the equivalence between two models and \ref{eq:opt general actionRem} holds for a general $\tnorm$-norm ECE budget constraint with any $t\geq 1$. For $K_1$ and $K_\infty$ ECE, the corresponding \ref{eq:opt general actionRem} becomes a linear program. Since the program has polynomial size and the transformation between the signaling scheme (with the bias assignment) and the predictor is also polynomial-time computable, we obtain \Cref{algo:poly-time}, which computes the optimal predictor under the $K_1$ or $K_\infty$ ECE budget in polynomial time.

Given the significance of the Bayesian persuasion and information design literature, we believe that our new variant of persuasion problem with signal-dependent bias, together with its results, may be of independent interest.

\subsection{Discussions}
\label{subsec:discussion}

\xhdr{How to choose ECE budget $\caliBudget$}
In this work, we consider an exogenously given ECE budget 
$\caliBudget$, under which the agent best responds to predictions as long as they are generated by an $\caliBudget$-calibrated predictor.
A natural question arises: does the agent's expected utility under the optimal $\caliBudget$-calibrated predictor always weakly decrease as $\caliBudget$ increases?

Perhaps a bit surprisingly,  the answer is no. 
in \Cref{apx:win-win}, we provide an example showing that, due to the strategic interaction between the principal and the agent, the agent's expected utility can increase as the ECE budget $\caliBudget$ increases within a certain range.
In fact, the agent's utility under the optimal $\caliBudget$-calibrated predictor  can be strictly higher than under a perfectly calibrated predictor ($\caliBudget = 0$).
However, if the ECE budget becomes too large, the agent may begin to doubt the reliability of the predictions and stop treating them as trustworthy predictions.
This introduces a key behavioral dimension to the problem: the appropriate level of calibration error depends not only on optimization or statistical considerations but also on how the agent perceives and responds to prediction quality.

\xhdr{Calibration to achieve weak commitment}
In information design literature, the designer (or sender) can commit ahead of time to a particular signaling or information disclosure policy, and the receiver knows all the details of this information policy and trusts this commitment.
Additionally, agents are assumed to be fully Bayesian and able to perform complex Bayesian reasoning based on any signal they receive.
While elegant in theory, these assumptions are often unrealistic in practice, as they require agents to have perfect knowledge and unlimited computational capacity.

Calibration, by contrast, offers a simpler and more practical alternative.
It can be viewed as a relaxation of both the strong commitment assumption and the demanding cognitive requirements placed on the agent.
Given a pre-specified ECE budget,
the principal does not need to commit to a specific predictor in advance -- she may generate predictions in any manner, as long as the calibration error, evaluated ex ante, remains within the budget.
On the agent's side, there is no need to know the prior, the predictor details, or to perform Bayesian updates.
The agent simply acts based on the predictions received. 
Thus, calibration enables a weaker form of commitment that is more behaviorally and computationally plausible.\footnote{Similar discussions between calibration and Bayesian reasoning has also been mentioned in \cite{CGGR-24}.}

As an overarching goal of this work, we seek to establish and encourage future research that bridges the gap between the literature on algorithmic information design and the literature on calibration -- two areas that have been extensively studied within the computer science community.

\subsection{Open Questions and Future Directions}
\label{subsec:open}

The persuasive calibration problem studied in this work opens up several interesting directions for future research, some of which we outline below.

\xhdr{Structural characterizations for general utility}
A natural question is whether our structural results characterized in \Cref{sec:optimal structure}  can be (at least partially) extended to more general utility functions. 
We conjecture that the miscalibration structure may still hold under broader conditions, although this likely requires a more fine-grained duality-based analysis.

\xhdr{Deterministic predictor}
Throughout the paper, we focus on potentially randomized predictors. While randomization is allowed in our framework, \Cref{prop:num predictions} shows that, conditional on each event, the number of possible predictions is bounded by a constant (at most four).\footnote{In fact, we can further strengthen this result: there exists at most one event that can generate four distinct predictions.}
These observations suggest that optimal predictors exhibit very limited randomization. 
With that being said, it is still interesting to explore the design of the optimal $(\caliBudget,\tnorm)$-calibrated deterministic predictor, which have been widely adopted in practice.

\xhdr{Extensions with multicalibration, multi-class calibration, and beyond}
Beyond the ECE metric considered in this work, there is a rich body of literature on broader notions of calibration.
We believe that our persuasive calibration framework naturally extends to these settings.
For example, by viewing each event as a data point, and assuming a membership structure among events, multicalibration becomes a more appropriate metric. It would be interesting to extend our results to predictors subject to multicalibration constraints.
In fact, the computational results and the connections to Bayesian persuasion established in \Cref{thm:opt general} can be extended to the multicalibration setting.
The key difference would be in the receiver's behavior -- defined in \Cref{defn:receiver behavior}  -- which, under multicalibration, would depend on the group membership of the events.

Moreover, it is also interesting to explore the setting with multi-outcome (i.e., multi-class) calibration. This problem would be particularly technically interesting as predictions must now lie in a high-dimensional simplex.
It closely relates to the high-dimensional information design, a notoriously difficult and unresolved area in the theoretical economics literature.

Lastly, in this work, we focus on a setting where the principal has full knowledge of each event's true expected outcome. While this is already a foundational setup, a more realistic and practical scenario may involve the principal only observing noisy realizations of outcomes. In such a case, one must design a predictor based on empirical observations, which we leave it as an important direction for future work.

\subsection{Additional Related Work}

Our work contributes to a growing line of research on calibration in decision-making settings.
Beginning with \cite{D-82,FV-98}, it is well established that agents who best respond to calibrated predictions (i.e., those with low $K_1$ ECE) achieve diminishing swap regret.
Recent works have extended this foundation by introducing new decision-theoretic calibration measures that offer finer-grained guarantees on the regret incurred by downstream agents \citep{KLST-23,RS-24,HW-24,QZ-25}.

For example, \citet{KLST-23} show that
a decision-maker cannot fare worse by trusting well-calibrated predictions -- the external regret incurred by best-responding to such predictions is linearly (up to a small factor) bounded by the $K_1$ ECE.
They introduce ``U-Calibration'' -- a measure bounded by a small constant times the $K_1$ ECE, and show that sublinear  U-Calibration is both
necessary and sufficient for ensuring sublinear regret for all downstream agents.
\citet{RS-24,HW-24} consider a similar setting but focusing on swap regret.
In particular, \citet{RS-24} show that
by ensuring unbiasedness relative to a carefully chosen collection of events, one can achieve diminishing swap regret for arbitrary downstream agents with an improved rate compared to using calibrated predictions.
\citet{HW-24} then introduce ``Calibration Decision Loss (CDL)'', defined as the maximum improvement (over all decision tasks) in decision payoff obtainable by recalibrating the predictions.
The show that CDL is upper bounded by twice the $K_1$ ECE and that a vanishing CDL guarantees vanishing payoff loss from miscalibration across all decision tasks,
which also removes the regret dependence on the number of actions appeared in the results of \citet{RS-24}.
\cite{QZ-25}
propose a new calibration measure called ``subsampled step calibration'', which has both decision-theoretic implications and a truthfulness property.

Other conceptually related works include \citet{NRRX-23,HPY-23} who focus on single decision task, and \citet{CHJ-20,CRS-24} who explore repeated Principal-Agent interactions in a prior-free setting.
\citet{ZKS-21} introduce ``decision calibration'' that requires the predicted distribution and true distribution to be indistinguishable to a set of downstream decision-makers. \citet{GKRSW-22} develop an omniprediction framework which aims to design a single predictor that ensures every downstream decision-maker's loss is no worse than some benchmark family of functions \citep{GHKRW-23,GJRR-24}. 
\cite{GS-21} analyze a sender-receiver game where the sender, facing miscalibration costs, provides predictions to the receiver; they show that the game's Nash equilibrium crucially depends on the magnitude of the miscalibration costs.

In contrast to these works, which primarily adopt a worst-case regret minimization framework and aim to achieve robustness against all decision tasks, our work focuses on characterizing the structure and algorithmic properties of the optimal predictor. 
In addition, many of the aforementioned works consider a specific principal's goal -- aiming to minimize regret across all downstream agents. 
By contrast, our framework accommodates arbitrary misaligned incentives between the principal and the agent.

In this work, we adopt the $\tnorm$-norm ECE as a framework to ensure the predictions to be trustworthy to the agents.
This $\tnorm$-norm calibration metric is broadly accepted in the forecasting community \cite{D-82,FV-98,GBR-07,RG-10}, within theoretical computer science (see, e.g., \citealp{QV-21,DDFM-24,CDV-24}), empirical machine learning community (see, e.g., \citealp{GK-14,R-21}). 
Beyond  $\tnorm$-norm ECE and the above mentioned decision-driven calibration metrics, other calibration measures have been proposed to capture different desired aspects, e.g., the multicalibration \citep{HKRR-18,GJRR-24}, multi-class calibration \citep{GHR-24}, distance to calibration \citep{BGHN-23,QZ-24,ACRS-25}.

\section{Preliminaries}
\label{sec:prelim}

In this work, we introduce and study the \emph{``persuasive calibration''} problem between a principal and a downstream decision-maker. Below, we outline the key components of our model.

\xhdr{Environment}
Consider a stochastic environment that randomly generates an \emph{event} from a finite set of $\numData$ events indexed by $[\numData] \triangleq \{1, 2, \dots, \numData\}$.\footnote{Equivalently, event distribution $\prior$ can be viewed as the a continuum population of events.} We denote by $\prior_i \in [0, 1]$ the probability of event $i$, satisfying $\sum_{i \in [\numData]} \prior_i = 1$. Once an event is realized, it further induces a randomized \emph{binary outcome}. Specifically, for each event $i \in [\numData]$, the binary outcome $\outcome \in \outcomeSpace \triangleq \{0, 1\}$ is drawn from a Bernoulli distribution with mean $\truemean_i \in [0, 1]$.  Without loss of generality, we assume that the events are sorted in non-decreasing order of their Bernoulli means: $\truemean_1 \leq \truemean_2 \leq \dots \leq \truemean_\numData$.  

There is a \emph{principal} (she) and a downstream decision-maker (he), referred to as the \emph{agent}. The principal observes the realized event $i$ but not the final binary outcome $\outcome$, while the agent observes neither the event nor the outcome. 
We detail the interaction between the principal and the agent below.

\xhdr{Calibrated predictors}
The principal provides predictions of the final binary outcome (which is unknown to both herself and the agent) to the agent, who then makes a decision that affects the payoffs of both the principal and himself. A \emph{prediction} $\prediction \in [0, 1]$ is a scalar indicating the probability that the binary outcome is one. The principal's predictions are generated by a (possibly randomized) \emph{predictor} $\predictor = \{\predictor_i\}_{i \in [\numData]}$, which specifies a profile of $\numData$ conditional distributions.\footnote{Since the principal observes the realized event $i$ but not the outcome $\outcome$, her predictor is a function of the realized event that does not depend on the outcome.} Specifically, for each event $i$, the conditional predictor $\predictor_i$ represents the distribution over predictions given that event $i$ occurs. We define the conditional prediction space $\supp(\predictor_i)$ as the support of conditional distribution $\predictor_i$ and (unconditional) prediction space $\supp(\predictor) \triangleq \cup_{i\in[\numData]}\supp(\predictor_i)$. With slight abuse of notation, we use $\predictor_i(\prediction)$ to denote the probability mass (or probability density if $\supp(\predictor_i)$ is continuous) of generating prediction $\prediction$ when event $i$ is realized.\footnote{While predictors with continuous prediction space are considered, our results show that there exists optimal predictors with finite prediction space.} We are also interested in the \emph{marginalized predictor}, denoted by $\marginalpredictor \in \Delta([0, 1])$, which marginalizes the predictor $\predictor$ over the events. Specifically, we define $\marginalpredictor(\prediction) \triangleq \sum_{i \in [\numData]} \prior_i \predictor_i(\prediction)$ to represent the marginal probability of generating the prediction $\prediction$.

The predictor provided by the principal is required to be \emph{trustworthy}. In our model, the trustworthiness of predictors follows the calibration framework: given a fixed predictor $\predictor$, for every prediction $\prediction \in \supp(\predictor)$, we denote by $\truePosterior(\prediction)\triangleq \expect{\outcome\condition \prediction}$ the \emph{true expected outcome} conditional on the realized prediction~$\prediction$, where the randomness is taken over event and outcome. Since the outcome is binary, it also represents the true probability that binary outcome is one, conditional on the realized prediction~$\prediction$. As a sanity check, if the prediction space $\supp(\prediction)$ is discrete, the true expected outcome $\truePosterior(\prediction)$ can be expressed according to Bayes' rule as 
\begin{align}
\label{eq:true expected outcome}
    \truePosterior(\prediction) \triangleq 
    \expect{\outcome\condition \prediction}
    =
    \frac{
    \sum_{i\in[\numData]}
    \prior_i \cdot 
    \predictor_i(\prediction)
    \cdot 
    \truemean_i
    }{
    \sum_{i\in[\numData]}
    \prior_i \cdot 
    \predictor_i(\prediction)}
\end{align}
A predictor $\bayespredictor$ is \emph{perfectly calibrated} if, for every prediction $\bayesprediction \in \supp(\bayespredictor)$, the true expected outcome $\truePosterior(\bayesprediction)$ equals the prediction itself, i.e., $\truePosterior(\bayesprediction) = \bayesprediction$.\footnote{\label{footnote:perfectly calibrated predictor}Throughout the paper, we use $\bayespredictor$ to denote a perfectly calibrated predictor and $\predictor$ to denote a general (possibly imperfectly calibrated) predictor. Similarly, we use $\bayesprediction$ to denote the prediction generated by perfectly calibrated predictors and $\prediction$ to denote the prediction generated by general (possibly imperfectly calibrated) predictors.} For illustration, \Cref{example:prelim:prefect calibration} presents two distinct perfectly calibrated predictors. 

\begin{example}[Perfect calibration]
    \label{example:prelim:prefect calibration}
    Suppose there are $\numData = 2$ events with $\prior_1 = \prior_2 = 0.5$ and $\truemean_1 = 0.3,\truemean_2 = 0.9$. Two perfectly calibrated predictors $\bayespredictor\primed$ and $\bayespredictor\doubleprimed$ can be constructed as follows:
    \begin{itemize}
        \item In the first predictor $\bayespredictor\primed$, deterministic prediction $\bayesprediction = 0.6$ is generated regardless of the realized event. Specifically, $\bayespredictor_i\primed(\bayesprediction) \triangleq \indicator{\bayesprediction = 0.6}$ for both $i\in[2]$.\footnote{We use $\indicator{\cdot}$ to denote the indicator function.} Note that the ex ante probability of outcome being one is $0.5\cdot 0.3 + 0.5\cdot 0.9 = 0.6$. Hence, the construction ensures that $\truePosterior(0.6) = 0.6$ and thus predictor $\bayespredictor\primed$ is perfectly calibrated.
        \item In the second predictor $\bayespredictor\doubleprimed$, prediction $\bayesprediction = \truemean_i$ is generated when event $i$ is realized. Specifically, $\bayespredictor_i\doubleprimed(\bayesprediction) \triangleq \indicator{\bayesprediction = \truemean_i}$ for both $i \in [2]$. By construction, predictor $\bayespredictor\doubleprimed$ is also perfectly calibrated.
    \end{itemize}
\end{example}
We adopt the standard $\tnorm$-norm expected calibration error metric to quantify the deviation of a predictor from being perfectly calibrated.
\begin{definition}[Expected calibration error]
\label{defn:ECE}
Fix any $\normexponent \in[1,\infty)$. The \emph{$\tnorm$-norm expected calibration error} (\emph{ECE}) $\ECE[\normexponent]{\predictor}$ of predictor $\predictor$ is 
\begin{align*}
    \ECEnorm{\predictor} \triangleq 
    \left(\expect{\abs{\truePosterior(\prediction) - \prediction}^{\normexponent}}\right)^{\frac{1}{\normexponent}}~,
\end{align*}
where the expectation is taken over the randomness of both the event and the prediction. Similarly, the \emph{$\ell_\infty$-norm ECE} of predictor $\predictor$ is 
    ${\ECE[\infty]{\predictor} \triangleq 
    \max\nolimits_{\prediction\in\supp(\predictor)}~ 
    |\truePosterior(\prediction) - \prediction|}$.
\end{definition}
Given any $\caliBudget\geq 0$ and $\normexponent\geq 1$,
we say a predictor $\predictor$ is \emph{$(\caliBudget, \tnorm)$-calibrated} if its $\tnorm$-norm ECE is at most $\caliBudget$, i.e., $\ECEnorm{\predictor}\leq \caliBudget$. We refer to $\caliBudget$ as the \emph{{\ECEbudget}}, and denote by $\predicSpace_{(\caliBudget,\tnorm)} \triangleq \{\predictor: \ECEnorm{\predictor}\le\caliBudget\}$ the space of all $(\caliBudget,\tnorm)$-calibrated predictors. Note that when the {\ECEbudget} is set to zero, the $(0, \tnorm)$-calibration recovers perfect calibration for all $\tnorm$ norms. Thus, we write $\predicSpace_0\equiv \predicSpace_{(0,\tnorm)}$.

Among all $\tnorm$-norm ECEs, the most classic and important one is the $\ell_1$-norm ECE, also known as the $K_1$ ECE. Additionally, the $\ell_2$-norm ECE and $\ell_\infty$-norm ECE, referred to as the $K_2$ and $K_\infty$ ECEs, respectively, are also standard in the literature and commonly used in practice. While our results apply to general $\tnorm$-norm ECEs, more refined and improved results are derived for these classic ECEs, particularly the $K_1$ ECE.
When $\tnorm$ norm is clear from the context, we sometimes simplify notations and omit them, e.g., writing $\ECE{\predictor}$, $\predicSpace_\caliBudget$ instead of $\ECEnorm{\predictor}$, $\predicSpace_{(\caliBudget,\tnorm)}$, and saying $\caliBudget$-calibrated instead of $(\caliBudget,\tnorm)$-calibrated. 

\begin{example}[$\caliBudget$-calibrated predictor]
\label{example:prelim:eps calibration}
    Consider the same two-event setting as in \Cref{example:prelim:prefect calibration} and two predictors, $\predictor\primed$ and $\predictor\doubleprimed$, constructed as follows:
    \begin{itemize}
        \item In the first predictor $\predictor\primed$, the deterministic prediction $\prediction = 0.4$ is generated regardless of the realized event. Specifically, $\predictor\primed_i(\prediction) = \indicator{\prediction = 0.4}$ for both $i\in[2]$. For all $\normexponent\geq 1$, its $\tnorm$-norm ECEs are identical and equal to $\ECEnorm{\predictor\primed} = 0.2$.
        \item In the second predictor $\predictor\primed$, predictions $\prediction = 0.4$ and $\prediction = 0.7$ are generated under events 1 and 2, respectively. Specifically, $\predictor\doubleprimed_1(\prediction) = \indicator{\prediction = 0.4}$ and $\predictor\doubleprimed_2(\prediction) = \indicator{\prediction = 0.7}$. For this predictor, the $K_1$, $K_2$ and $K_\infty$ ECEs are $\ECE[1]{\predictor\doubleprimed} = 0.15$, $\ECE[2]{\predictor\doubleprimed} = \sqrt{0.025}\approx 0.158$, $\ECE[\infty]{\predictor\doubleprimed} = 0.2$, respectively.
    \end{itemize}
\end{example}

\xhdr{Agent's problem}
The agent (downstream decision-maker) has a finite set of actions, denoted by $\actionSpace$ with $\numAction \triangleq |\actionSpace|$. His utility function is given by $\receiverU(\cdot, \cdot): \actionSpace \times \outcomeSpace \rightarrow \R$, which depends on both the chosen action and the binary outcome of the event.\footnote{\label{footnote:event independent agent utility}Notably, the agent's utility does not depend on the event, while the principal's utility may depend on the event, besides the agent's action and the binary outcome.} We use $\receiverU(\action, \outcome)$ to denote the agent's utility when he takes action $\action \in \actionSpace$ and the outcome $\outcome\in\outcomeSpace$ is realized.

The agent has no knowledge of the environment, the realized outcome, or the details of the predictor. He can only use the realized prediction $\prediction$ generated by the principal's predictor to inform his decision.
Knowing that the predictions are generated by a predictor with small calibration error, the agent \emph{trusts} the principal's prediction $\prediction$ and chooses the action that maximizes his expected utility assuming the outcome is equal to one with probability $\prediction$.\footnote{The behavior of agent directly best responding to the provided predictions that are generated from a low-ECE predictor has been widely adopted and considered in calibration literature, see, e.g., \cite{HPY-23,GJRR-24,KLST-23,HW-24,RS-24,QZ-25}.}
More formally, for any prediction $\prediction$ generated by an $\caliBudget$-calibrated predictor $\predictor$, the agent takes the action $\bestr(\prediction) \in \actionSpace$ that satisfies\footnote{\label{fn:tie-breaking}Following the literature, we assume the agent breaks ties in favor of the principal when he faces indifferent actions.
As a result, the principal's indirect utility function under each event preserves the upper-semi continuity. This property guarantees the principal's problem always feasible.
}
\begin{align}
    \label{eq:agent bestr}
    \bestr(\prediction) \in \argmax\nolimits_{\action\in\actionSpace}~
    \expect[\outcome\sim \Bern(\prediction)]{\receiverU(\action, \outcome)}
    =
    \argmax\nolimits_{\action\in\actionSpace}~
    \prediction\cdot 
    \receiverU(\action, 1)
    +
    (1-\prediction)\cdot 
    \receiverU(\action, 0)
\end{align}

\xhdr{Principal's problem}
The principal's utility depends on the agent's action, the realized outcome, and the realized event. Specifically, for each event $i \in [n]$, agent's action $\action \in \actionSpace$, and outcome $\outcome \in \outcomeSpace$, we denote the principal's utility by $\senderU_i(\action, \outcome){\in\R^+}$.\footnote{{Throughout the paper, we consider non-negative principal's utility, and we do not make any other assumptions unless explicitly stated.}} To simplify the presentation, we also introduce the \emph{indirect utility function} $\indirectsenderU_i$, which maps a given prediction $\prediction$ to the principal's expected utility when the agent selects his best action $\bestr(\prediction)$ upon the realization of event $i$.
\begin{definition}[Indirect utility]
    For each event $i\in[n]$ and prediction $\prediction\in \supp(\predictor)$, the principal's \emph{indirect utility} is 
    $\indirectsenderU_i(\prediction) \triangleq \truemean_i \cdot \senderU_i(\bestr(\prediction), 1)
    +
    (1-\truemean_i)\cdot \senderU_i(\bestr(\prediction), 0)$, where $\bestr(\prediction)$ is defined in Eqn.~\eqref{eq:agent bestr}.\footnote{Notably, when the principal evaluates her indirect utility $\indirectsenderU_i(\prediction)$, the outcome follows the Bernoulli distribution with mean $\truemean_i$ instead of prediction $\prediction$.}
\end{definition}
\noindent Given a predictor $\predictor$, the principal's expected utility is expressed as\footnote{When $\predictor_i$ is continuous, the inner summation in the expression of $\Payoff{\predictor}$ becomes integral, i.e., $\int_0^1 
    \prior_i\cdot  \predictor_i(\prediction)\cdot 
    \indirectsenderU_i(\prediction)\,\d\prediction$. To avoid redundancy, in the remainder of the paper, we present formulas for either the continuous or discrete distribution, as the analogous version can be easily derived.}
\begin{align*}
    \Payoff{\predictor} &\triangleq \expect[i\sim\prior, \prediction\sim\predictor_i,\outcome\sim \Bern(\truemean_i)]{\senderU_i(\bestr(\prediction), \outcome)}
    = \sum\nolimits_{i\in[\numData]}\sum\nolimits_{\prediction\in\supp(\predictor_i)}
    \prior_i\cdot  \predictor_i(\prediction)\cdot 
    \indirectsenderU_i(\prediction)
\end{align*}
where the second equality holds due to the definition of the indirect utility. 

Fixing an exogenously given {\ECEbudget} $\caliBudget$ and $\tnorm$-norm calibration error, we say an $(\caliBudget,\tnorm)$-calibrated predictor $\optpredictor$ is \emph{optimal} if it maximizes the principal's expected utility among all $(\caliBudget,\tnorm)$-calibrated predictors, i.e., 
\begin{align*}
    \optpredictor \in \argmax\nolimits_{\predictor\in \predicSpace_{(\caliBudget,\tnorm)}}
    ~\Payoff{\predictor}
\end{align*}
We use the tuple $(\stateSpace, \prior, \senderU, \receiverU, \actionSpace, (\caliBudget, \tnorm))$ to refer to as a persuasive calibration instance, where $\stateSpace = \{\truemean_i\}_{i\in[\numData]}$ represents the profile of expected outcomes for all events.
 
\xhdr{Timeline} We formalize the timeline of the model as follows: 
\begin{enumerate}
    \item The principal designs a $(\caliBudget, \tnorm)$-calibrated predictor $\predictor$.\footnote{Note the ECE of a predictor is evaluated ex ante, thus, it does not matter when the principal determines the predictor details.}
    \item A randomized event $i \in [\numData]$ is realized with probability $\prior_i$, and 
    a (randomized) prediction $\prediction$ is sent to the agent according to the conditional distribution $\predictor_i(\cdot)$.
    \item The agent selects his best action $\bestr(\prediction)$ as defined in Eqn.~\eqref{eq:agent bestr}.
    \item The binary outcome $\outcome$ is realized from a Bernoulli distribution with mean $\truemean_i$. The principal and the agent receive utilities $\senderU_i(\bestr(\prediction), \outcome)$ and $\receiverU(\bestr(\prediction), \outcome)$, respectively.
\end{enumerate}

\section{Structural Characterization for the Event-Independent Setting}
\label{sec:optimal structure}

In this section, we focus on the \emph{event-independent setting}, where the principal's indirect utility does not depend on the realized event. Specifically, we assume that $\indirectsenderU_i(\prediction) \equiv \indirectsenderU_j(\prediction)$ for all events $i, j \in [\numData]$ and predictions $\prediction\in[0, 1]$.\footnote{Such event-independent indirect utility assumptions are standard and widely studied in the algorithmic game theory and economics literature \citep[e.g.,][]{DM-19,LR-20,FTX-22,ABSY-23,CD-23,FHT-24}. In our model, it is equivalent to assume the principal's utility $\senderU$ does not depend on the event and outcome.} Within this section, we omit the subscript $i$ in the principal's indirect utility function, i.e., write $\indirectsenderU(\prediction)$ instead of $\indirectsenderU_i(\prediction)$.

In \Cref{subsec:characterization:K1 ECE}, we focus on the 
$\ell_1$-norm ECE (a.k.a., $K_1$ calibration error) and present the structural characterization of the optimal $\caliBudget$-calibrated predictor. 
In \Cref{subsec:characterization:two-step perspective}, we introduce a two-step framework for generating $(\caliBudget,\tnorm)$-calibrated predictors for general $\normexponent \ge 1$. 
Our proofs, provided in \Cref{subsec:proofs event-inde}, for the structural characterization of the optimal $\caliBudget$-calibrated predictor heavily relies on the this two-step framework with a dual-based approach.

\subsection{Structural Results for Optimal $(\caliBudget,\ell_{1})$-Calibrated Predictor}
\label{subsec:characterization:K1 ECE}

In this section, we focus on the $\ell_1$-norm ECE, which is the most classic and important calibration error definition, and we present the structural characterization of the optimal $(\caliBudget,\ell_1)$-calibrated predictor. All the proofs are deferred to \Cref{subsec:proofs event-inde}.
Except for the theorem statements, we omit $\ell_1$ for the presentation simplicity in this subsection.

We first introduce the following concepts of under-confidence and over-confidence, which are introduced and observed in the machine learning literature \citep[e.g.,][]{GPSW-17,KFE-18,FSZ-23}. 

\begin{definition}[Under-/over-confidence]
    Fix any predictor $\predictor$. A prediction $\prediction\in\supp(\predictor)$ is \emph{under-confident} (resp.\ \emph{over-confident}, \emph{perfectly calibrated}) if $\prediction \leq \truePosterior(\prediction)$ (resp. $\prediction\geq \truePosterior(\prediction)$, $\prediction= \truePosterior(\prediction)$), where $\truePosterior$ is the {\TPF} of predictor $\predictor$ defined in Eqn.~\eqref{eq:true expected outcome}.
\end{definition}

We now present the main structural characterization of the optimal predictor.\footnote{\Cref{thm:event-inde} and its analysis also holds when there is continuum population of events.}

\begin{theorem}[Structure of optimal $\caliBudget$-calibrated predictor]
\label{thm:event-inde}
For every persuasive calibration instance in the event-independent setting, there exists $0 \leq \lowq \leq \highq\leq 1$ that partitions 
the prediction space $[0, 1]$ into \emph{under-confidence interval} $[0, \lowq]$, \emph{perfectly calibrated internal} $[\lowq, \highq]$, and \emph{over-confidence} interval $[\highq, 1]$. Then, there exists an optimal $(\caliBudget,\ell_1)$-calibrated predictor $\optpredictor$ that satisfies the following three properties:
\begin{enumerate}
    \item[(1)] \emph{(Miscalibration Structure)} Every prediction $\prediction$ in the under-confidence interval (resp.\ over-confidence interval, perfectly calibrated interval) is under-confident (resp.\ over-confident, perfectly calibrated), i.e., 
    \begin{align*}
    \begin{array}{ll}
      \forall \prediction \in \supp(\optpredictor)\cap[0, \lowq]:   &\quad \prediction\leq \truePosterior(\prediction)  \\
         \forall \prediction \in \supp(\optpredictor)\cap[\lowq, \highq]:& \quad 
        \prediction= \truePosterior(\prediction)
        \\
        \forall \prediction \in \supp(\optpredictor)\cap[\highq, 1]:& \quad 
        \prediction \geq \truePosterior(\prediction)
    \end{array}
    \end{align*}
    \item[(2)] \emph{(Payoff Structure-I)} There exists a function $\sgselect:\supp(\optpredictor) \rightarrow \R$ that  selects a subgradient at each prediction $\prediction\in\supp(\optpredictor)$, i.e., $\sgselect(\prediction) \in \partial \indirectsenderU(\prediction)$, such that function $\sgselect$ is increasing, and it satisfies $\sgselect(\prediction) \equiv \caliBudgetDual$ for all predictions in over-confidence interval, and $\sgselect(\prediction) \equiv -\caliBudgetDual$ for all predictions in under-confidence interval for some $\caliBudgetDual\ge 0$.

    \item[(3)] \emph{(Payoff Structure-II)} All points $(\prediction, \indirectsenderU(\prediction))_{\prediction\in\supp(\optpredictor)}$ form a convex function.
    Moreover, points $(\prediction, \indirectsenderU(\prediction))$ for all predictions in over-confidence interval are collinear (with slope $\caliBudgetDual$), and points $(\prediction, \indirectsenderU(\prediction))$ for all predictions in under-confidence interval are collinear (with slope $-\caliBudgetDual$).
\end{enumerate}
\end{theorem}

\xhdr{Implications on the miscalibration of neural networks}
\Cref{thm:event-inde} sheds light on a well-observed phenomenon in the machine learning community: the miscalibration of modern neural networks.
It is well-known that neural networks have been found to be poorly calibrated.
For example, in the seminal prior work, \cite{GPSW-17} show that uncalibrated ResNet tends to be overconfident in its predictions (see Figure 4 in \citealp{GPSW-17}), and \cite{KFE-18} show that the neural networks consistently generate underestimated and overestimated predictions/forecasts at the two extremes of the probability spectrum (see Figures 2 and 3 in \citealp{KFE-18}).

Our framework provides a theoretical explanation for this behavior.
Specifically, the principal's problem can be viewed as minimizing a particular loss function.
The miscalibration structure established in \Cref{thm:event-inde} implies that, whenever there exist calibration errors in the output predictions, they must consistently occur in the tails of the prediction spectrum. 
That is, there exist $\lowq \le \highq$ such that the model underestimates probabilities in the under-confidence interval $[0, \lowq]$ and overestimates them in the over-confidence interval $[\highq, 1]$, no matter what the loss function is. 
This structural insight aligns closely with empirical observations in machine learning.
Although we acknowledge that here we simplify the principal's utility function to be event-independent, it may not fully capture all loss functions that are being considered in training neural networks, we conjecture that our findings can be generalized to more complex utility functions (see \Cref{subsec:discussion} for more detailed discussions).

\xhdr{Interpretation of the payoff structure} 
\Cref{thm:event-inde} also provides insights into the principal's indirect utility and its corresponding derivative (or subgradient if the derivative does not exist) for predictions generated by the optimal predictor. Specifically, the ``per-unit'' indirect utility derived from predictions in the under-confidence interval $[0,\lowq]$ and over-confidence interval $[\highq, 1]$ is higher than the ``per-unit'' indirect utility derived from predictions in the perfectly calibrated interval $[\lowq, \highq]$. This aligns with the intuition that the optimal predictor is willing to incur ECE to generate predictions with higher payoff. Furthermore, our result suggests that the marginal indirect utility is the same for all predictions in the under-confidence interval $[0,\lowq]$ and the over-confidence interval $[\highq, 1]$, and is larger than that for predictions in the perfectly calibrated interval $[\lowq, \highq]$.

We now illustrate \Cref{thm:event-inde} with two examples. \Cref{example:binary action} considers instances where the agent has binary actions. \Cref{example:s-shape} considers instances where the principal's indirect utility function $\indirectsenderU$ follows an S-shaped curve (i.e., first convex then concave). Both setups are widely studied in the literature \citep[e.g.,][]{DX-17,BB-17,xu-20,FTX-22,FHT-24,KLZ-24}.

\begin{example}[Binary-action setting]
\label{example:binary action}
Consider a persuasive calibration instance where the agent has binary actions (i.e., $\actionSpace = \{0, 1\}$) and the principal has utility function $\senderU_i(\action, \outcome) \equiv \action$ (i.e., always preferring the agent taking action 1). Let $\pthreshold\in[0, 1]$ be the prediction such that the agent is indifferent between two actions. Suppose the agent takes action 1, i.e., $\bestr(\prediction) = 1$, for every prediction $\prediction\geq \pthreshold$.
Let $\priormean\triangleq \expect[i\in\prior]{\truemean_i}$ be the ex ante expected outcome.
The optimal $\caliBudget$-calibrated predictor $\optpredictor$ satisfying properties in \Cref{thm:event-inde} has the following two cases:

\xhdr{{Case 1. Low ECE budget $\caliBudget< \pthreshold - \priormean$}} In this case, let $\thresholdi\in[\numData]$ be the threshold event such that 
    $\sum_{i\in[\thresholdi+1:\numData]}\prior_i(\pthreshold-\truemean_i) \leq \caliBudget < \sum_{i\in[\thresholdi:\numData]}\prior_i(\pthreshold-\truemean_i)$. The optimal $\caliBudget$-calibrated predictor has degenerate under-confidence interval $[0,0]$, perfectly calibrated interval $[0,\pthreshold]$, and over-confidence interval $[\pthreshold, 1]$. For every realized event $i< \thresholdi$, predictor $\optpredictor$ generates a deterministic and perfectly calibrated prediction $\prediction = \truemean_i$. For every realized event $i > \thresholdi$, predictor $\optpredictor$ generates a deterministic but over-confident prediction $\prediction = \pthreshold$. Finally, for the threshold event $\thresholdi$, predictor $\optpredictor$ generates an over-confident prediction $\prediction = \pthreshold$ with probability 
    ${(\caliBudget - \sum_{i\in[\thresholdi+1:\numData]} \prior_i (\pthreshold-\truemean_i))}{(\prior_{\thresholdi} (\pthreshold - \truemean_{\thresholdi}))^{-1}}$ and a
    perfectly calibrated prediction $\prediction = \truemean_{\thresholdi}$ with the remaining probability. The construction of the threshold event $\thresholdi$ and its prediction distribution ensures that the predictor $\optpredictor$ fully exhausts the ECE budget, i.e., $\ECE{\optpredictor} = \caliBudget$.
 
    \xhdr{{Case 2. High ECE budget $\caliBudget\geq \pthreshold - \priormean$}}
    In this case, the optimal $\caliBudget$-calibrated predictor generates a deterministic (but possibly over-confident) prediction $\prediction = \max\{\priormean,\pthreshold\}$ regardless of the realized event. Its induced ECE budget is $\ECE{\predictor} = \max\{\pthreshold - \priormean,0\}$.
\end{example}

\begin{example}[S-shaped indirect utility]
\label{example:s-shape}
Consider a persuasive calibration instance where the principal's has an event-independent S-shaped (i.e., first convex then concave) indirect utility function~$\indirectsenderU$. For simplicity, we assume function $\indirectsenderU$ is differentiable with $\partial\indirectsenderU(0) = \partial\indirectsenderU(1) = 0$ and $\indirectsenderU(0) < \indirectsenderU(1)$.
Let $\priormean\triangleq \expect[i\in\prior]{\truemean_i}$ be the ex ante expected outcome.
The optimal $\caliBudget$-calibrated predictor $\optpredictor$ satisfying properties in \Cref{thm:event-inde} has the following two cases:\footnote{The structure of the optimal predictor $\optpredictor$ for the S-shaped indirect utility closely resembles that for the binary action case. This similarity arises because the indirect utility function under binary actions can be seen as a ``discretized'' version of a S-shaped function.} 

\xhdr{{Case 1. Low ECE budget $\caliBudget< 1 - \priormean$}} In this case,
there exists a threshold event $\thresholdState\in[\numData]$, a threshold prediction $\optpoolingpredic \in [0, 1]$, and a probability $\optpoolingprob\in(0, 1]$ such that\footnote{Here the first equality means that if $\optpoolingprob < 1$, then the line through points $(\truemean_{\thresholdi},\indirectsenderU_i(\truemean_{\thresholdi}))$ and $(\pthreshold,\indirectsenderU(\pthreshold))$ is tangent to function $\indirectsenderU$ at point $(\pthreshold,\indirectsenderU(\pthreshold))$.}
\begin{align*}
    (1-\optpoolingprob)\cdot \left(\indirectsenderU(\truemean_{\thresholdi}) + \partial\indirectsenderU(\pthreshold)- \indirectsenderU(\pthreshold)\right)=0
    \;\;
    \mbox{and}
    \;\;
    \optpoolingprob\cdot\prior_{\thresholdi}(\pthreshold - \truemean_{\thresholdi})+
    \sum\nolimits_{i\in[\thresholdi+1:\numData]}\prior_i(\pthreshold - \truemean_i)
    =
    \caliBudget
\end{align*}
Then the optimal $\caliBudget$-calibrated predictor has degenerate under-confidence interval $[0,0]$, perfectly calibrated interval $[0,\pthreshold]$, and over-confidence interval $[\pthreshold, 1]$.
For every realized event $i< \thresholdi$, predictor $\optpredictor$ generates a deterministic and perfectly calibrated prediction $\prediction = \truemean_i$. For every realized event $i > \thresholdi$, predictor $\optpredictor$ generates a deterministic but over-confident prediction $\prediction = \pthreshold$. Finally, for the threshold event $\thresholdi$, predictor $\optpredictor$ generates an over-confident prediction $\prediction = \pthreshold$ with probability 
$\optpoolingprob$ and a
perfectly calibrated prediction $\prediction = \truemean_{\thresholdi}$ with the remaining probability. The construction of the threshold event $\thresholdi$ and its prediction distribution ensures that the predictor $\optpredictor$ fully exhausts the ECE budget, i.e., $\ECE{\optpredictor} = \caliBudget$. Also see \Cref{fig:s-shaped} for an illustration.

\xhdr{{Case 2. High ECE budget $\caliBudget\geq 1 - \priormean$}}
    In this case, the optimal $\caliBudget$-calibrated predictor generates a deterministic (but possibly over-confident) prediction $\prediction = 1$ regardless of the realized event. Its induced ECE budget is $\ECE{\predictor} = 1 - \priormean$.
\end{example}

\xhdr{Graphical characterization of the optimal predictor} \Cref{thm:event-inde} and its analysis also admit a graphical interpretation as follows.

\begin{definition}[Symmetric linear-tailed convex function]
\label{defn:symmetric linear-tailed}
A function $\convexEnv:[0, 1]\rightarrow \R$ is {\em symmetric linear-tailed convex} if it is convex, and there exists $0 \le \lowx \le \highx \le 1$ such that (i) $\convexEnv$ is linear on both intervals $[0, \lowx]$ and $[\highx, 1]$; and (ii) the absolute values of the slopes of these linear parts are equal. 
We refer to the intervals $[0, \lowx]$ and $[\highx, 1]$ as the \emph{linear tails} of the function $\convexEnv$.
\end{definition}

\begin{proposition}
\label{prop:convex env}
For every persuasive calibration instance in the event-independent setting,
fix an optimal $(\caliBudget,\ell_1)$-calibrated predictor $\optpredictor$ satisfying the properties in \Cref{thm:event-inde}. There exists a symmetric linear-tailed convex function $\convexEnv:[0,1]\rightarrow\R$ 
such that 
(i) $\convexEnv(\prediction) \ge \indirectsenderU(\prediction)$ for all $\prediction\in[0, 1]$; 
(ii) $\supp(\optpredictor) \subseteq \{\prediction\in[0, 1]: \convexEnv(\prediction) = \indirectsenderU(\prediction)\}$; 
(iii) it is linear over $[0, \lowq]$ and $[\highq, 1]$.
\end{proposition}

As we will show in \Cref{subsec:proofs event-inde}, the symmetric linear-tailed convex function $\convexEnv$ and its properties in \Cref{prop:convex env} are derived from the optimal dual solution of a linear program \ref{eq:decoupled opt} that characterizes the optimal predictor.

In \Cref{fig:main}, we illustrate the optimal predictor characterized in \Cref{thm:event-inde} along with the corresponding symmetric linear-tailed convex function $\convexEnv$ from \Cref{prop:convex env}. When the principal's indirect utility function is general, as shown in \Cref{fig:main}, we observe the (non-degenerate) under-confidence, perfectly calibrated, and over-confidence intervals structure. The symmetric linear-tailed convex function $\convexEnv$ exhibits linear tails at both ends, covering the under-confidence and over-confidence intervals.

In contrast, as illustrated in \Cref{example:s-shape}, when the principal's indirect utility function follows an S-shaped structure, the optimal predictor only miscalibrates at the highest prediction. Consequently, in \Cref{fig:s-shaped}, the under-confidence interval degenerates, and the symmetric linear-tailed convex function $\convexEnv$ retains only a single linear tail, corresponding to the over-confidence interval.

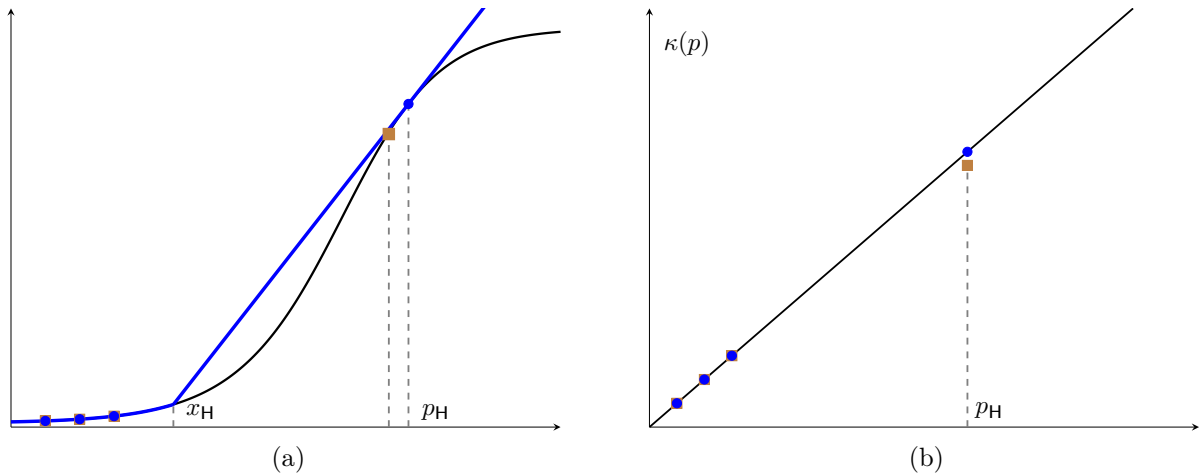
\begin{figure}[H]
  \centering
  \begin{subfigure}{0.45\textwidth}
    \centering
    \resizebox{\linewidth}{!}{\begin{tikzpicture}[baseline=(current axis.south)]
\def\yminval{-0.01}
\begin{axis}[
    axis lines=left,
    axis line shift=0pt,
    width=10cm,
    height=8cm,
    xtick=\empty, 
    ytick=\empty,
    xmin=0, xmax=0.8,
    ymin=\yminval, ymax=1.05,
]

% Define parameters k, m, v
\def\k{14.75020921892135}
\def\m{0.5}
\def\v{1.2800549434321995}

% Define the function u(x)
\def\u(#1){1/(1 + exp(-\k*(#1-\m)))^(1/\v)}

% Define its derivative F'(x)
\def\up(#1){(\k/\v)*exp(-\k*(#1-\m))*(1 + exp(-\k*(#1-\m)))^(-1/\v - 1)}

% Plot u(x) as a black dashed curve
\addplot[
  domain=0:0.8,
  samples=200,
  line width=1pt,
  black
] { \u(x) };

% The point where we draw the tangent
\def\xqone{0.05}
\def\xqtwo{0.1}
\def\xqthree{0.15}
\def\xp{0.5786}
\def\xq{0.55}
\def\xintersect{0.236392}

\addplot[
  only marks,
  mark=square*,
  mark size=2.2pt,
  color={brown},
  opacity=0.9
] coordinates {(\xqone, {\u(\xqone)})
               (\xqtwo, {\u(\xqtwo)})
               (\xqthree, {\u(\xqthree)})
               };

\addplot[
  only marks,
  mark=*,
  mark size=2pt,
  color=blue
] coordinates {(\xqone, {\u(\xqone)})
               (\xqtwo, {\u(\xqtwo)})
               (\xqthree, {\u(\xqthree)})
               };

% The tangent line y = F(xT) + F'(xT)*(x - xT)
\addplot[
  domain=\xintersect:0.8,
  % dashed,
  % dash pattern=on 3pt off 3pt,
  line width=1.5pt,
  blue,
  opacity=0.6
] { \u(\xp) + \up(\xp)*(x - \xp) };

\addplot[
  domain=0:\xintersect,
  samples=200,
  line width=1.5pt,
  blue,
  opacity=0.6
] { \u(x) };

% The vertical line at x = 0.683753916245
\draw[dashed, thick, gray]
  (axis cs:\xq,\yminval) -- (axis cs:\xq,{ \u(\xq) });

\draw[dashed, thick, gray]
  (axis cs:\xp,\yminval) -- (axis cs:\xp,{ \u(\xp) });

\draw[dashed, thick, gray]
  (axis cs:\xintersect,\yminval) -- (axis cs:\xintersect,{ \u(\xintersect) });

\addplot[
  only marks,
  mark=square*,
  mark size=2.5pt,
  color = {brown},
  opacity=0.9
] coordinates {(\xq, {\u(\xq)-0.005})};

\addplot[
  only marks,
  mark=*,
  mark size=2pt,
  color=blue
] coordinates {(\xp, {\u(\xp)})};

\node[anchor=north] at (axis cs:\xp+0.04,0.07) {$\highq$};

\node[anchor=north] at (axis cs:\xintersect+0.04,0.07) {$\highx$};

% Scatter point at (xV, F(xV))

\end{axis}
\end{tikzpicture}}
    \caption{}
    \label{fig:s-shaped indirect utility}
  \end{subfigure}
  \hspace{0.8cm}
  \begin{subfigure}{0.45\textwidth}
    \centering
    \resizebox{\linewidth}{!}{\begin{tikzpicture}[baseline=(current axis.south)]
\def\yminval{0}
\begin{axis}[
    axis lines=left,
    axis line shift=0pt,
    width=10cm,
    height=8cm,
    xtick=\empty, 
    ytick=\empty,
    xmin=0, xmax=1,
    ymin=\yminval, ymax=0.88,
]

\def\diagonal(#1){(#1)}
\addplot[domain=0:0.88, samples=200, thick, black] { \diagonal(x) };

% The point where we draw the tangent
\def\xqone{0.05}
\def\xqtwo{0.1}
\def\xqthree{0.15}
\def\xp{0.5786}
\def\xq{0.55}
\def\xintersect{0.236392}

\addplot[
  only marks,
  mark=square*,
  mark size=2.2pt,
  color={brown},
  opacity=0.9
] coordinates {(\xqone, \xqone)
                (\xqtwo, \xqtwo)
                (\xqthree, \xqthree)
                (\xp, \xq)
               };

\addplot[
  only marks,
  mark=*,
  mark size=2pt,
  color=blue
] coordinates {(\xqone, \xqone)
                (\xqtwo, \xqtwo)
                (\xqthree, \xqthree)
                (\xp, \xp)
               };

% \draw[dashed, thick, gray]
%   (axis cs:\xqone,0) -- (axis cs:\xqone,\xqone);

% \draw[dashed, thick, gray]
%   (axis cs:\xqtwo,0) -- (axis cs:\xqtwo,\xqtwo);

% \draw[dashed, thick, gray]
%   (axis cs:\xqthree,0) -- (axis cs:\xqthree,\xqthree);
  
\draw[dashed, thick, gray]
  (axis cs:\xp,0) -- (axis cs:\xp,\xq);

\node[anchor=north] at (axis cs:\xp+0.04,0.07) {$\highq$};

\node[anchor=south west, rotate=0] at (rel axis cs:0.01,0.865) {$\truePosterior(\prediction)$};

\end{axis}
\end{tikzpicture}}
    \caption{}
    \label{fig:reliable-diagram-s}
  \end{subfigure}
  \caption{Illustration of the principal with a S-shaped indirect utility (\Cref{example:s-shape}).  
  In both figures, blue dots are the predictions $\prediction$ generated from the optimal predictor $\optpredictor$, 
  brown squares are the corresponding true expected outcomes $\truePosterior(\prediction)$ conditional on the prediction $\prediction$. 
  \Cref{fig:general} is an illustration of \Cref{thm:event-inde} -- black solid curve is indirect utility $\indirectsenderU$ and blue curve is the convex function $\convexEnv$ in \Cref{prop:convex env}.
  \Cref{fig:reliable-diagram-general} is the reliable diagram -- black solid line is diagonal line.}
  \label{fig:s-shaped}
\end{figure}

In \Cref{example:binary action,example:s-shape}, we observe that the number of predictions used in the optimal predictor scales linearly with the number of events, while the number of predictions generated for each realized event remains constant. The following proposition formally establishes that both observations hold for all persuasive calibration instances in the event-independent setting.
\begin{proposition}[How many predictions are necessary?]
\label{prop:num predictions}
For every persuasive calibration instance in the event-independent setting, 
there exists an optimal $(\caliBudget,\ell_1)$-calibrated predictor $\optpredictor$ satisfying all structures in \Cref{thm:event-inde}, such that
\begin{enumerate}
    \item[(i)] \emph{(Total predictions)} The total number of predictions induced by the optimal predictor $\optpredictor$ is at most $n + 2$, i.e., $|\supp(\optpredictor)| \le \numData + 2$, where $\numData$ is the number of events. 
    \item[(ii)] \emph{(Predictions per event)} For each event $i\in[\numData]$, the number of predictions induced by the optimal predictor $\optpredictor$ for this event is at most 4, i.e., $|\supp(\optpredictor_i)|\leq 4$.
    \item[(iii)] \emph{(Predictions per true expected outcome)} For every true expected outcome $\bayesprediction\in[0, 1]$, the number of predictions which are induced by the optimal predictor $\optpredictor$ and have true expected outcome $\bayesprediction$ is at most 2, i.e., $|\truePosterior^{-1}(\bayesprediction)| \leq 2$.\footnote{Given any predictor $\predictor$, the inverse of true expected outcome function $\truePosterior^{-1}$ is defined as $\truePosterior^{-1}(\bayesprediction) \triangleq \{\prediction\in \supp(\predictor):\truePosterior(\prediction) = \bayesprediction\}$.} Moreover, if the equality holds (i.e., $|\truePosterior^{-1}(\bayesprediction)| = 2$), at least one prediction $\prediction\in \truePosterior^{-1}(\bayesprediction)$ belongs to the under-confidence $[0,\lowq]$ or over-confidence interval $[\highq,1]$ of $\optpredictor$, i.e., $\prediction\in [0,\lowq]\cup[\highq,1]$.
\end{enumerate}
\end{proposition}
The following example shows that the Statement (i) in above \Cref{prop:num predictions} is indeed tight.
\begin{example}
Given an ECE budget $\caliBudget > 0$ with $\caliBudget < 0.05$,
consider a two-event setting where $\truemean_1 = 0.1, \truemean_2 = 0.9$ with $\prior_1 =\prior_2 = 0.5$, and the principal's indirect utility is $\indirectsenderU(\prediction) 
= \indicator{\prediction = \truemean_1 - 0.5\caliBudget} + 
2\cdot \indicator{\prediction = \truemean_1 - 2\caliBudget}
+ 
\indicator{\prediction = \truemean_2 + 0.5\caliBudget} + 
2\cdot \indicator{\prediction = \truemean_2 + 2\caliBudget}$.
The optimal $\caliBudget$-calibrated predictor $\optpredictor$  is unique and it is 
$\optpredictor_1(\truemean_1 - 0.5\caliBudget) =\optpredictor_2(\truemean_2 + 0.5\caliBudget) = \frac{2}{3}$ and  
$\optpredictor_1(\truemean_1 - 2\caliBudget) = \optpredictor_2(\truemean_2 + 2\caliBudget)= \frac{1}{3}$.
Namely, we have $\supp(\optpredictor) = 4 = \numData + 2$.\footnote{For readers who are familiar with information design literature, it is well-known that the number of necessary signals in optimal signaling scheme in Bayesian persuasion is not larger than the number of events. Our example here shows a strict separation between our persuasive calibration problem and the canonical Bayesian persuasion problem.}
\end{example}
We conclude this subsection by noting that the proofs of all results are provided in \Cref{subsec:proofs event-inde}. In the next subsection, we introduce a general framework that serves as the foundation for our proofs.

\subsection{Two-Step Framework to Generate $\varepsilon$-Calibrated Predictors}
\label{subsec:decouple}
\label{subsec:characterization:two-step perspective}
In this section, we provide a systematic way for the principal to design and analyze the predictor. At a high level, the process consists of two steps. First, a perfectly calibrated predictor is constructed. Second, we apply a simple transformation from this perfectly calibrated predictor to obtain an $\caliBudget$-calibrated predictor.\footnote{All results and discussions in \Cref{subsec:characterization:two-step perspective} apply to all $\tnorm$-norm ECEs. Thus, except for the theorem statements, we omit $\tnorm$ for simplicity in this subsection.} As we will illustrate in later sections, this two-step framework and its generalization play an important role in computing and analyzing the optimal $\caliBudget$-calibrated predictor.

Our two-step framework relies on the \emph{post-processing plan} defined as follow:

\begin{definition}[Event-independent post-processing plan]
\label{defn:miscal plan}
An \emph{(event-independent) post-processing plan} $\miscali$ is a joint distribution over $[0, 1]\times[0, 1]$. We denote by $\miscali(\bayesprediction, \prediction)$ the probability density (or probability mass for discrete distribution) of $(\bayesprediction, \prediction)\in[0, 1]^2$.

Fix any perfectly calibrated predictor $\bayespredictor$ and any $\normexponent\geq 1,\caliBudget\geq 0$. A post-processing plan $\miscali$ is \emph{feasible} under perfectly calibrated predictor $\bayespredictor$ and $\tnorm$-norm {\ECEbudget} $\caliBudget$ if it satisfies 
\begin{align}
    \label{eq:feasible miscal cal budget}
    \tag{\textsc{Calibration-Feasibility}}
    & \left(\int_0^1 \int_{0}^1 \miscali(\bayesprediction, \prediction) \cdot |\bayesprediction - \prediction|^{\normexponent}~\d  \prediction\d  \bayesprediction\right)^{\frac{1}{\normexponent}} 
    \le \caliBudget
    \\
    \label{eq:feasible miscal supp}
    \tag{\textsc{Supply-Feasibility}}
    \forall \bayesprediction\in [0, 1]:\quad
    & \int_0^1\miscali(\bayesprediction, \prediction)~\d  \prediction 
    = \bayesmarginalpredictor(\bayesprediction)
\end{align}
where $\bayesmarginalpredictor$ is marginalized predictor induced by $\bayespredictor$, defined as $\bayesmarginalpredictor(\bayesprediction)\triangleq \sum_{i\in[\numData]} \prior_i\bayespredictor_i(\bayesprediction)$ for every $\bayesprediction\in\supp(\bayespredictor)$.
When $\bayespredictor, \normexponent,\caliBudget$ are clear from the context, we simply refer to $\miscali$ as a feasible post-processing plan.
\end{definition}
Intuitively, one can view the post-processing plan $\miscali$ as a joint distribution over the prediction $\bayesprediction\in\supp(\bayespredictor)$ and the prediction $\prediction$ that \emph{``miscalibrates''} the true expected outcome $\bayesprediction$ to be $\prediction$. 
We say that $\miscali$ is event-independent because this joint distribution does not depend on the realized event. We formalize this intuition through the following concept of a \emph{post-processing procedure}. We also illustrate it using \Cref{example:structural opt:miscalibration procedure}.\footnote{\label{footnote:necessity of g in two-step framework}In fact, due to \ref{eq:feasible miscal supp} and \Cref{lem:bayes predictor is in MPC}, a perfectly calibrated predictor $\bayespredictor$ can be determined by a feasible post-processing plan $\miscali$. Hence, such a post-processing plan $\miscali$ is sufficient to generate an imperfectly calibrated predictor via \Cref{defn:miscal procedure}. Nevertheless, we choose  to explicitly express the perfectly calibrated predictor $\bayespredictor$ (or its marginalized version, $\bayesmarginalpredictor$), as it reveals useful structural properties--such as the MPC representation in \Cref{defn:mpc} and \Cref{lem:bayes predictor is in MPC}--which play an important role in our analysis.}

\begin{definition}[Event-independent post-processing procedure]
\label{defn:miscal procedure}
Fix any perfectly calibrated predictor~$\bayespredictor$ and any $\normexponent \geq 1, \caliBudget\geq 0$. Given any feasible post-processing plan $\miscali$, a new (possibly imperfectly calibrated) predictor $\predictor$ is \emph{generated} as follows: for every {\event} $i\in[\numData]$, and prediction $\prediction\in[0, 1]$,
\begin{align}
    \label{eq:construct f}
    \predictor_i(\prediction) \triangleq 
    \int_0^1 \frac{\bayespredictor_i(\bayesprediction)}{\bayesmarginalpredictor(\bayesprediction)}\cdot \miscali(\bayesprediction,\prediction)\,\d\bayesprediction~.
\end{align}
To simplify the presentation, we also say predictor $\predictor$ is \emph{generated by} $\bayespredictor$ and $\miscali$.
\end{definition}

\begin{example}[Predictors generated from post-processing procedure]
\label{example:structural opt:miscalibration procedure}
    Consider the same two-event setting and four predictors $\bayespredictor\primed,\bayespredictor\doubleprimed,\predictor\primed,\predictor\doubleprimed$ as in \Cref{example:prelim:prefect calibration,example:prelim:eps calibration}. Predictors $\predictor\primed,\predictor\doubleprimed$ can be generated by the perfect calibrated predictors $\bayespredictor\primed,\bayespredictor\doubleprimed$ via \Cref{defn:miscal procedure} with the following post-processing plans:
    \begin{itemize}
        \item Construct post-processing plan $\miscali\primed$ with $\miscali\primed(\bayesprediction,\prediction) = \indicator{\bayesprediction = 0.6,\prediction = 0.4}$. Then, predictor $\predictor\primed$ is generated by perfectly calibrated predictor $\bayespredictor\primed$ and post-processing plan $\miscali\primed$.
        \item Construct post-processing plan $\miscali\doubleprimed$ with $\miscali\doubleprimed(\bayesprediction,\prediction) = 0.5\cdot \indicator{\text{$\bayesprediction = \truemean_1,\prediction = 0.4$ or $\bayesprediction = \truemean_2,\prediction = 0.7$}}$. Then, predictor $\predictor\doubleprimed$ is generated by perfectly calibrated predictor $\bayespredictor\doubleprimed$ and post-processing plan $\miscali\doubleprimed$.
    \end{itemize}
\end{example}

With \ref{eq:feasible miscal supp} of the feasible post-processing plan $\miscali$, the predictor $\predictor = (\predictor_i)_{i\in[\numData]}$ generated in Eqn.~\eqref{eq:construct f} is indeed well-defined as each $\predictor_i$ is a valid distribution over $[0, 1]$. In addition, \ref{eq:feasible miscal cal budget} ensures that predictor $\predictor$ is $\caliBudget$-calibrated.

\begin{proposition}
\label{prop:miscal procedure error guarantee}
Fix any perfectly calibrated predictor $\bayespredictor$, any $\normexponent\geq 1,\caliBudget\geq 0$, and any feasible post-processing plan $\miscali$. The predictor $\predictor$ generated by $\bayespredictor$ and $\miscali$ (in \Cref{defn:miscal procedure}) is well-defined and $(\caliBudget,\tnorm)$-calibrated, i.e., $\ECEnorm{\predictor} \leq \caliBudget$.
\end{proposition}
\begin{proof}
We first verify that for each event $i\in[\numData]$, $\predictor_i$ is a valid distribution on $[0, 1]$:
\begin{align*}
    \int_0^1 \predictor_i(\prediction)~\d\prediction &\overset{(a)}{=}
    \int_0^1\int_0^1 \frac{\bayespredictor_i(\bayesprediction)}{\bayesmarginalpredictor(\bayesprediction)}\cdot \miscali(\bayesprediction,\prediction)\,\d\bayesprediction\d\prediction
    =
    \int_0^1 \frac{\bayespredictor_i(\bayesprediction)}{\bayesmarginalpredictor(\bayesprediction)} \int_0^1\miscali(\bayesprediction,\prediction)\,\d\prediction\d\bayesprediction
    \overset{(b)}{=}
    \int_0^1 \bayespredictor_i(\bayesprediction)\,\d\bayesprediction = 1
\end{align*}
where equality~(a) holds due to the construction of predictor $\predictor$ in \Cref{defn:miscal procedure}, and equality~(b) holds due to \ref{eq:feasible miscal supp} of post-processing plan $\miscali$.

We next bound the $\tnorm$-norm ECE of predictor $\predictor$:
\begin{align*}
    \left(\ECEnorm{\predictor}\right)^{\normexponent}
    & = \expect{\abs{\truePosterior(\prediction) - \prediction}^{\normexponent}}\\
    & = 
    \int_0^1 \abs{\sum\nolimits_{i\in[\numData]} \prior_i  \predictor_i(\prediction)\cdot (\prediction-\truemean_i)}^{\normexponent}~\d \prediction\\
    & 
    \overset{(a)}{=} 
    \int_0^1 \abs{\sum\nolimits_{i\in[\numData]} \prior_i  \cdot (\prediction -\truemean_i) 
    \int_0^1\frac{\bayespredictor_i(\bayesprediction)}{\bayesmarginalpredictor(\bayesprediction)}\cdot \miscali(\bayesprediction,\prediction)\,\d\bayesprediction
    }^{\normexponent}~\d \prediction
    \\
    & \overset{(b)}{=} 
    \int_0^1 \abs{ \int_0^1\prediction\cdot\miscali(\bayesprediction, \prediction)\,\d\bayesprediction
    - 
    \int_0^1 \sum\nolimits_{i\in[\numData]} \prior_i \truemean_i \cdot 
    \frac{\bayespredictor_i(\bayesprediction)} {\bayesmarginalpredictor(\bayesprediction)}\cdot \miscali(\bayesprediction,\prediction)\,\d\bayesprediction
    }^{\normexponent}~\d \prediction\\
    & \overset{(c)}{=} 
    \int_0^1 \abs{ \int_0^1\prediction\cdot \miscali(\bayesprediction, \prediction)\,\d\bayesprediction
    - 
    \int_0^1 \bayesprediction\cdot \miscali(\bayesprediction,\prediction) \,\d\bayesprediction
    }^{\normexponent}~\d \prediction\\
    & \overset{(d)}{\le} 
    \int_0^1 \int_0^1\miscali(\bayesprediction, \prediction) \cdot 
    \abs{\prediction -\bayesprediction}^{\normexponent} \,\d\bayesprediction \d \prediction 
   \overset{(e)}{\leq} \caliBudget^{\normexponent}~,
\end{align*}
where equality (a) holds due to the construction of predictor $\predictor$ in \Cref{defn:miscal procedure}, 
equality (b) holds due to the definition of marginalized perfectly calibrated predictor $\bayesmarginalpredictor$,
equality (c) holds due to the fact that perfectly calibrated predictor $\bayespredictor$ has zero ECE, 
inequality (d) holds due to the Jensen's inequality and the convexity of $\abs{\cdot}^{\normexponent}$, and 
inequality (e) holds due to \ref{eq:feasible miscal cal budget} of post-processing plan $\miscali$. This completes the proof of \Cref{prop:miscal procedure error guarantee}.
\end{proof}

We remark that various post-processing procedures have been introduced in the calibration literature. However, the majority of these methods focus on transforming a miscalibrated predictor into a perfectly calibrated one or reducing its ECE \citep[e.g.,][]{KGZ-19,FSZ-23,BGHN-23b,GHR-24}. That said, as we will demonstrate in \Cref{apx:win-win example}, when the principal and agent have misaligned incentives, allowing the principal to use a predictor with higher ECE may improve both her and the agent's utility--creating a ``win-win'' scenario.
Hence, we believe our post-processing procedure is not only valuable as an analytical tool but may also have practical implications.

\xhdr{Predictor space $\predicSpaceEI_\caliBudget$}
We denote by $\predicSpaceEI_{\caliBudget}$ the space of all predictors generated by perfectly calibrated predictors and feasible miscalibration plans via the event-independent post-processing procedure (\Cref{defn:miscal procedure}). 
\Cref{prop:miscal procedure error guarantee} guarantees that every predictor $\predictor\in \predicSpaceEI_{\caliBudget}$ is ${\caliBudget}$-calibrated, and thus 
\begin{align*}
    \predicSpaceEI_{\caliBudget}\subseteq \predicSpace_{\caliBudget}~.
\end{align*} 
\Cref{ex:special predictor} shows that the inclusion relation between the two spaces is strict, i.e.,  there exists a predictor $\predictor\in \predicSpace_{\caliBudget}$ such that $\predictor\not\in\predicSpaceEI_{\caliBudget}$, showing a strict inclusion relation between the two spaces.\footnote{In \Cref{sec:fptas}, we extend the event-independent post-processing procedure to the event-dependent post-processing procedure, which generates every ${\caliBudget}$-calibrated predictors from perfectly calibrated predictor.} (See \Cref{fig:hierarch} for the hierarch illustration.) However, when the principal's indirect utility is event-independent, we argue that restricting to predictors in $\predicSpaceEI_{\caliBudget}$ is without loss of generality.

\begin{example}[Strict inclusion relation between  $\predicSpaceEI_{\caliBudget}$ and $\predicSpace_{\caliBudget}$]
\label{ex:special predictor}
\label{example:structural opt:strict inclusion}
There are two events and let $\prior_1=\prior_2 = 0.5$ with $\truemean_1 = \truemean_2 = 0.5$. Consider predictor $\predictor$ with
$\predictor_1(0.125) = \predictor_2(0.875) = 1$.
It can be verified that such predictor $\predictor\in \predicSpace_{\caliBudget}\backslash\predicSpaceEI_{\caliBudget}$ for every $\caliBudget\geq 0.375$ and $\normexponent \geq 1$.
\end{example}

\begin{figure}[t]
  \centering
  \begin{tikzpicture}[>=stealth,scale=1]

  % Define some custom colors (adjust to taste):
  \definecolor{mygray}{gray}{0.85}    % Light gray for the outer ring
  \definecolor{myorange}{RGB}{255,220,100} % Lighter orange for middle ring
  \definecolor{myorangeDark}{RGB}{255,200,70} % Darker orange for inner ellipse

  %--- 1) Fill the largest ellipse (F_\epsilon) in gray ---
  \fill[mygray] (0.3,0) ellipse (4 and 2.2);
  \draw[dashed,thick] (0.3,0) ellipse (4 and 2.2);

  \begin{scope}
    \clip (0,0) ellipse (4 and 2.2); % Outer boundary
    \fill[myorange] (0,0) ellipse (3 and 1.8); % Middle ellipse
    \draw[dashed,thick] (0,0) ellipse (3 and 1.8);
  \end{scope}

  %--- 3) Fill the innermost ellipse (F_0) in a darker orange ---
  \fill[myorangeDark] (-0.5,0) ellipse (2 and 1);
  \draw[thick] (-0.5,0) ellipse (2 and 1);

  %--- 5) A small green “center” point labeled \tilde{g} ---
  \fill[blue] (0.7,0) circle (3pt);
  \node[blue,anchor=west] at (0.2,0) {$\bayespredictor$};

  %--- 6) Two blue points with dashed arrows from \tilde{g} ---
  %     (Adjust the coordinates to place them as desired.)
  \coordinate (f)  at (3, -1.2);
  \coordinate (fd) at (2.3, 1.0);

  \fill[blue] (f) circle (3pt);
  \node[blue,anchor=west] at (f) {$\predictor$};

  \fill[blue] (2.1, 0.8) circle (3pt);
  \node[blue,anchor=south] at (1.8, 0.5) {$\newpredictor$};

  \draw[dashed,ultra thick, ->] (2.8, -1.1) -- (0.8,-0.1);
  \draw[dashed,ultra thick, ->] (0.85,0.1) -- (2.0, 0.7);

  %--- 7) Labels for the sets (adjust positions as needed) ---
  \node at (3.6, -0.1)    {$\predicSpace_{\caliBudget}$};
  \node at (2.25, -0.1)          {$\predicSpaceEI_{\caliBudget}$};
  \node at (-0.5, -0.1)   {$\predicSpace_{0}$};
\end{tikzpicture}
  \caption{An illustration of hierarchical relation between $\predicSpace_{\caliBudget}$, $\predicSpaceEI_{\caliBudget}$ and $\predicSpace_{0}$.
  The dashed arrow line from predictor $\predictor$ to perfectly calibrated predictor $\bayespredictor$ is established by \Cref{prop:two-step equivalence}, and the dashed arrow line from  $\bayespredictor$ to the predictor $\newpredictor$ is established in \Cref{defn:miscal procedure} and \Cref{prop:miscal procedure error guarantee}.}
  \label{fig:hierarch}
\end{figure}

\begin{proposition}
\label{prop:two-step equivalence}
For every persuasive calibration instance in the event-independent setting, 
for every $(\caliBudget,\tnorm)$-calibrated predictor $\predictor\in \predicSpace_{(\caliBudget,\tnorm)}$, there exists a pair of perfectly calibrated predictor $\bayespredictor$ and feasible miscalibration plan $\miscali$ that can generate a (possibly different) predictor $\newpredictor\in\predicSpaceEI_{(\caliBudget,\tnorm)}$ with the same expected utility for the principal, i.e.,  
$\Payoff{\newpredictor} = \Payoff{\predictor}$.

Specifically, such a pair of perfectly calibrated predictor $\bayespredictor$ and feasible miscalibration plan $\miscali$ can be constructed as follows: 
for every event $i\in[\numData]$, and predictions $\bayesprediction, \prediction\in[0, 1]$,
\begin{align}
    \label{eq:bayes predic construction}
    \bayespredictor_i(\bayesprediction) \triangleq 
    \sum\nolimits_{\prediction\in\supp(\predictor_i):\truePosterior(\prediction) =\bayesprediction} \predictor_i(\prediction) 
    \;\;\mbox{and}\;\;
    \miscali(\bayesprediction, \prediction) \triangleq \sum\nolimits_{i\in[\numData]} \prior_i \predictor_i(\prediction) \cdot \indicator{\truePosterior(\prediction) = \bayesprediction}
\end{align}
where $\truePosterior(\prediction)$ is the true expected outcome defined in Eqn.~\eqref{eq:true expected outcome}.
\end{proposition}

\begin{corollary}
    \label{cor:wlog ei space}
    For every persuasive calibration instance in the event-independent setting, there exists an optimal predictor $\optpredictor\in \predicSpaceEI_{(\caliBudget,\tnorm)}$, i.e., $\Payoff{\optpredictor} \geq \Payoff{\predictor}$ for every $\predictor\in \predicSpace_{(\caliBudget,\tnorm)}$.
\end{corollary}

\begin{proof}[Proof of \Cref{prop:two-step equivalence}]
Fix any $\caliBudget$-calibrated predictor $\predictor\in \predicSpace_\caliBudget$. Consider predictor $\bayespredictor$ and post-processing plan $\miscali$ constructed as in Eqn.~\eqref{eq:bayes predic construction}.

We first verify that the constructed predictor $\bayespredictor$ has zero ECE and thus is perfect calibrated. Let $\truePosterior$ and $\truePosterior\primed$ be the true expected outcome function induced by original predictor $\predictor$ and constructed predictor $\bayespredictor$, respectively.
Note that for every prediction $\bayesprediction\in\supp(\bayespredictor)$, its true expected outcome $\truePosterior\primed$ can be computed as
\begin{align*}
    \truePosterior\primed(\bayesprediction) 
    = \frac{ \sum\nolimits_{i\in[\numData]} \prior_i \bayespredictor_i(\bayesprediction) \cdot \truemean_i }{ 
    \sum\nolimits_{i\in[\numData]} \prior_i \bayespredictor_i(\bayesprediction)}
    & = 
    \frac{ \sum\nolimits_{i\in[\numData]} \prior_i  
    \sum_{\prediction\in\supp(\predictor_i):\truePosterior(\prediction)=\bayesprediction}\predictor_i(\prediction) \cdot \truemean_i }{ 
    \sum\nolimits_{i\in[\numData]} \prior_i\sum_{\prediction\in\supp(\predictor_i):\truePosterior(\prediction)=\bayesprediction}\predictor_i(\prediction) } = \bayesprediction
\end{align*}
where the first and third equalities hold due to the definition of true expected outcome function $\truePosterior\primed$ and $\truePosterior$, while the second equality holds due to the construction of predictor $\bayespredictor$.

We next verify that the constructed post-processing plan $\miscali$ is feasible. Note that for each prediction $\bayesprediction\in[0, 1]$, 
\begin{align*}
    \int_0^1 \miscali(\bayesprediction, \prediction)~\d  \prediction
    & = 
    \int_0^1 \sum\nolimits_{i\in[\numData]} \prior_i \predictor_i(\prediction) \cdot\indicator{\truePosterior(\prediction) = \bayesprediction}
    ~\d  \prediction
    \\
    &
    = 
    \sum\nolimits_{i\in[\numData]} \prior_i 
    \sum\nolimits_{\prediction\in\supp(\predictor_i):\truePosterior(\prediction)=\bayesprediction} \predictor_i(\prediction) 
    \overset{}{=}
    \sum\nolimits_{i\in[\numData]} \prior_i \cdot 
    \bayespredictor_i(\bayesprediction)~,
\end{align*}
where the first and last equalities follow the construction in Eqn.~\eqref{eq:bayes predic construction}. Thus, the \ref{eq:feasible miscal supp} (in \Cref{defn:miscal plan}) of post-processing plan $\miscali$ is satisfied. In addition, 
\begin{align*}
    \int_0^1 \int_{0}^1 \miscali(\bayesprediction, \prediction) \cdot |\bayesprediction - \prediction|^{\normexponent}~\d  \prediction\d  \bayesprediction
    & = 
    \int_0^1\int_0^1
    \sum\nolimits_{i\in[\numData]} \prior_i \predictor_i(\prediction) \cdot\indicator{\truePosterior(\prediction) = \bayesprediction} \cdot \abs{\bayesprediction- \prediction}^{\normexponent} ~\d  \prediction \d  \bayesprediction \\
    & = 
    \int_0^1
    \sum\nolimits_{i\in[\numData]} \prior_i \predictor_i(\prediction)  \cdot \abs{\truePosterior(\prediction)-\prediction}^{\normexponent} ~\d  \prediction = \left(\ECE[\normexponent]{\predictor}\right)^{\normexponent} \le \caliBudget^\normexponent
\end{align*}
which implies the \ref{eq:feasible miscal cal budget} of post-processing plan $\miscali$. Thus, $\miscali$ is feasible.

Invoking \Cref{prop:two-step equivalence}, we know predictor $\predictor\primed$ generated by $\bayespredictor$ and $\miscali$ is well-defined and $(\caliBudget,\tnorm)$-calibrated. Finally, we compute the principal's expected utility given predictor $\predictor\primed$,
\begin{align*}
    \Payoff{\newpredictor}
    & = \sum\nolimits_{i\in[\numData]}\int_0^1\prior_i \newpredictor_i(\prediction)\cdot  \indirectsenderU(\prediction) \, \d \prediction
     \overset{(a)}{=} 
    \int_0^1 \int_0^1 \sum\nolimits_{i\in[\numData]} \prior_i \frac{\bayespredictor_i(\bayesprediction)}{\bayesmarginalpredictor(\bayesprediction)}\cdot \miscali(\bayesprediction,\prediction)\cdot \indirectsenderU(\prediction)\,\d\bayesprediction  \, \d \prediction\\
    & = 
    \int_0^1  \int_0^1 \miscali(\bayesprediction, \prediction) \cdot \indirectsenderU(\prediction)\,\d\bayesprediction  \, \d \prediction
    \overset{(b)}{=}  
    \int_0^1 \int_0^1
    \sum\nolimits_{i\in[\numData]} \prior_i \predictor_i(\prediction) \cdot\indicator{\truePosterior(\prediction) = \bayesprediction} \cdot \indirectsenderU\left(\prediction\right) ~\d  \prediction \d  \bayesprediction\\
    &= 
    \int_0^1
    \sum\nolimits_{i\in[\numData]} \prior_i \predictor_i(\prediction)\cdot   \indirectsenderU\left(\prediction\right) ~\d  \prediction = \Payoff{\predictor}
\end{align*}
where equality (a) follows by definition of $\newpredictor$ (as it is constructed according to Eqn.~\eqref{eq:construct f}), 
equality (b) follows by the construction of $\miscali$ defined in Eqn.~\eqref{eq:bayes predic construction}.
This finishes the proof of \Cref{prop:two-step equivalence}.
\end{proof}

\xhdr{A two-step framework for the principal}
\Cref{prop:two-step equivalence} (along with \Cref{cor:wlog ei space}) establishes an equivalence between $\caliBudget$-calibrated predictors $\predictor$ and pairs consisting of perfectly calibrated predictor $\bayespredictor$ and feasible miscalibration plan $\miscali$. This equivalence allows us to adopt a conceptually \emph{two-step approach}, where the principal first designs a perfectly calibrated predictor and then selects a feasible miscalibration plan. The following proposition formalizes this idea and shows that the optimal $\caliBudget$-calibrated predictor can be characterized by \ref{eq:decoupled opt}, an infinite-dimensional linear program, which plays a key role for proving all our structure results in \Cref{subsec:characterization:K1 ECE}.

\begin{proposition}
    \label{prop:prog equivalence}
    For every persuasive calibration instance in the event-independent setting,  the principal's expected utility under the optimal $(\caliBudget,\tnorm)$-calibrated predictor $\optpredictor$ is equal to the optimal objective value of the following linear program with variables $\{\miscali(\bayesprediction,\prediction), \bayesmarginalpredictor(\bayesprediction)\}_{\bayesprediction,\prediction\in[0,1]}$:
    \begin{align}
        \label{eq:decoupled opt}
        \arraycolsep=5.4pt\def\arraystretch{1}
        \tag{$\textsc{LP-TwoStep}$}
        &\begin{array}{rlll}
        \max
        \limits_{{\miscali} \ge 0, {\bayesmarginalpredictor}} 
        ~ &
        \displaystyle 
        \int_0^1\int_0^1
        \miscali(\bayesprediction, \prediction) \cdot \indirectsenderU\left(\prediction\right)~\d  \prediction \d  \bayesprediction
        \quad & \text{s.t.} &
        \vspace{1mm}
        \\
        (\textsc{Calibration-Feasibility}) & 
        \displaystyle 
        \int_0^1\int_0^1
        \miscali(\bayesprediction, \prediction) \cdot \abs{\prediction -\bayesprediction}^\normexponent ~\d  \prediction \d  \bayesprediction
        \le \caliBudget^\normexponent,
        &  
        \vspace{1mm}
        \\
        (\textsc{Supply-Feasibility}) & 
        \displaystyle 
        \int_0^1\miscali(\bayesprediction, \prediction)~\d  \prediction = \bayesmarginalpredictor(\bayesprediction),
        &  \bayesprediction\in [0, 1]
        \vspace{2.5mm}
        \\
        (\textsc{MPC-Feasibility}) &
        \bayesmarginalpredictor \in \MPC(\prior)~,
        &
        \end{array}
    \end{align}  
    where $\bayesmarginalpredictor\in\MPC(\prior)$ is the mean-preserving constraint (see \Cref{defn:mpc}).
    
    Moreover, every optimal solution of \ref{eq:decoupled opt} can be converted into an optimal $(\caliBudget,\tnorm)$-calibrated predictor $\optpredictor$ (see \Cref{lem:bayes predictor is in MPC} and \Cref{prop:miscal procedure error guarantee}).
\end{proposition}

\begin{definition}[Mean-preserving contraction]
\label{defn:mpc}
    Fix any distribution $\prior$ with support $\supp(\prior)\subseteq[0, 1]$. For any distribution $\bayesmarginalpredictor$ with support $\supp(\bayesmarginalpredictor)\subseteq[0, 1]$., we say distribution $\bayesmarginalpredictor$ is a \emph{mean-preserving contraction} of distribution $\prior$ if for all $s\in[0, 1]$, 
    \begin{align*}
        \int_0^s \bayespredictorCDF(\bayesprediction)\,\d\bayesprediction
        \leq \int_0^s \priorCDF(\bayesprediction)\,\d\bayesprediction
    \end{align*}
    with the equality at $s = 1$, where $\bayespredictorCDF$ and $\priorCDF$ are the cumulative density function of distributions $\bayesmarginalpredictor$ and $\prior$, respectively. We define $\MPC(\prior)$ as the set of all distributions that are mean-preserving contractions of the distribution $\prior$.
\end{definition}

\begin{restatable}{lemma}{bayesMPC}
\label{lem:bayes predictor is in MPC}
A distribution $\bayesmarginalpredictor \in \Delta([0, 1])$ is a marginalized perfectly calibrated predictor if and only if 
$\bayesmarginalpredictor\in \MPC(\prior)$.
\end{restatable}
\begin{proof}
Consider any perfectly calibrated predictor $\bayespredictor$ and its corresponding marginalized predictor $\bayesmarginalpredictor$, 
we know that for every $\prediction\in\supp(\bayespredictor)$, we have $\truePosterior(\prediction) = \prediction$, which implies that
\begin{align*}
    \frac{
    \sum_{i\in[\numData]}
    \prior_i \cdot 
    \bayespredictor_i(\prediction)
    \cdot 
    \truemean_i
    }{
    \sum_{i\in[\numData]}
    \prior_i \cdot 
    \bayespredictor_i(\prediction)}
    = \prediction~.
\end{align*}
By viewing the predictor $\bayespredictor$ as a Blackwell experiment, we can see that every prediction $\prediction$ in the support $\supp(\bayesmarginalpredictor)$ represents the mean of the Bayesian posterior updated from receiving a signal $\prediction$ (with the prior $\prior$).
Thus, the marginalized predictor $\bayesmarginalpredictor$ equivalently represents a distribution of (Bayesian) posterior means. 
By \cite{B-53,RS-78,GK-16}, we know that a distribution $\bayesmarginalpredictor$ is a feasible distribution of posterior means if and only if it is an MPC of its prior probability $\prior$.
\end{proof}

We conclude this subsection by proving \Cref{prop:prog equivalence}.

\begin{proof}[Proof of \Cref{prop:prog equivalence}]
We first prove the direction that
for any $\caliBudget$-calibrated predictor $\predictor$, there exists a pair of perfectly calibrated predictor $\bayespredictor$ (let $\bayesmarginalpredictor$ be its corresponding marginalized perfectly calibrated predictor) and a miscalibration plan $\miscali$ such that $\bayesmarginalpredictor, \miscali$ is a feasible solution to \ref{eq:decoupled opt}, and meanwhile its objective in \ref{eq:decoupled opt} exactly equals to $\Payoff{\predictor}$.

To see this, given the predictor $\predictor$, we consider the predictor $\bayespredictor$, and the miscalibration plan constructed as in Eqn.~\eqref{eq:bayes predic construction}.
\Cref{prop:two-step equivalence} shows that the constructed predictor $\bayespredictor$ is perfectly calibrated and post-processing plan $\miscali$ is feasible. Let $\bayesmarginalpredictor$ be the marginalized predictor of $\bayespredictor$.
Invoking \Cref{lem:bayes predictor is in MPC}, we know that $\bayesmarginalpredictor\in \MPC(\prior)$.
Thus, $(\bayesmarginalpredictor, \miscali)$ forms a feasible solution to \ref{eq:decoupled opt}.
Moreover, the objective value of $\bayesmarginalpredictor, \miscali$ in \ref{eq:decoupled opt} is  
\begin{align*}
    \int_0^1  \int_0^1 \miscali(\bayesprediction, \prediction) \cdot \indirectsenderU(\prediction) ~\d  \prediction\d \bayesprediction
    & =  
    \int_0^1 \int_0^1
    \sum\nolimits_{i\in[\numData]} \prior_i \predictor_i(\prediction) \cdot\indicator{\truePosterior(\prediction) = \bayesprediction} \cdot \indirectsenderU\left(\prediction\right) ~\d  \prediction\d \bayesprediction\\
    & = 
    \int_0^1
    \sum\nolimits_{i\in[\numData]} \prior_i \predictor_i(\prediction)  \cdot \indirectsenderU\left(\prediction\right) ~\d  \prediction = \Payoff{\predictor}~.
\end{align*}
We next prove the reverse direction: given any feasible solution $\bayesmarginalpredictor, \miscali$ to \ref{eq:decoupled opt}, 
there exists an $\caliBudget$-calibrated predictor $\predictor$ such that its payoff exactly equals to the objective value of the solution $\bayesmarginalpredictor, \miscali$ in \ref{eq:decoupled opt}.

To see this, we note that from \Cref{lem:bayes predictor is in MPC}, we know for any $\bayesmarginalpredictor\in\MPC(\prior)$, there must exist a perfectly calibrated predictor $\bayespredictor$ whose marginalized perfectly calibrated predictor exactly equals to $\bayesmarginalpredictor$. 
Now given such a perfectly calibrated predictor $\bayespredictor$, and together with $\miscali$, we consider a predictor $\predictor$ generated according to Eqn.~\eqref{eq:construct f}. Invoking~\Cref{prop:miscal procedure error guarantee}, predictor $\predictor$ is well-defined and $\caliBudget$-calibrated. Finally, the principal's expected utility under predictor $\predictor$ is
\begin{align*}
    \Payoff{\predictor} 
    & = 
    \sum\nolimits_{i\in[\numData]}\prior_i \int_0^1 \predictor_i(\prediction) \cdot  \indirectsenderU(\prediction) ~\d  \prediction  \\
    & 
     \overset{}{=} 
    \int_0^1 \int_0^1 \sum\nolimits_{i\in[\numData]} \prior_i \frac{\bayespredictor_i(\bayesprediction)}{\bayesmarginalpredictor(\bayesprediction)}\cdot \miscali(\bayesprediction,\prediction)\cdot \indirectsenderU(\prediction)\,\d\bayesprediction  \, \d \prediction
     = 
    \int_0^1  \int_0^1 \miscali(\bayesprediction, \prediction) \cdot \indirectsenderU(\prediction)\,\d\bayesprediction  \, \d \prediction
\end{align*}
which exactly equals to the objective value of the solution $\bayesmarginalpredictor, \miscali$ in \ref{eq:decoupled opt}.
This finishes the proof of \Cref{prop:prog equivalence}.
\end{proof}

A direct application of our two-step framework is the following verification tool that can be used to verify if a given predictor $\predictor$ is indeed an optimal $(\caliBudget, \ell_1)$-calibrated predictor for $K_1$ ECE.
\begin{proposition}[Optimality verification]
\label{prop:opt verify}
Fix any $\caliBudget > 0$ and consider $\normexponent = 1$.
Given any predictor $\tilde{h}$, consider the predictor $\predictor$ generated (according to \eqref{eq:construct f}) by the perfectly calibrated predictor $\bayespredictor$ and the miscalibration plan $\miscali$, which are constructed according to \eqref{eq:bayes predic construction} with input $\tilde{h}$.
If there exists a symmetric linear-tailed convex function $\convexEnv$ (defined in \Cref{defn:symmetric linear-tailed}) with $\convexEnv(x) \ge \senderU(x), \forall x$ satisfying
\begin{enumerate}
    \item[(i)]
    $\caliBudgetDual\cdot\left(\ECE{\predictor} - \caliBudget\right) = 0$ where $\caliBudgetDual\triangleq \max_{\prediction} |(\convexEnv(\prediction))'|$;
    \item[(ii)] (ii-a): 
    $\supp(\predictor) \subseteq \{\prediction\in[0, 1]: \convexEnv(\prediction) = \senderU(\prediction)\}$, and 
    (ii-b): all $\prediction$ with $\prediction\neq \truePosterior(\prediction)$ 
    and their true expected outcome $\truePosterior(\prediction)$ 
    must lie in linear tails of function $\convexEnv$;
    \item[(iii)]
    consider a marginalized predictor $\bayesmarginalpredictor$ induced from the perfectly calibrated predictor $\bayespredictor$ (defined in \eqref{eq:bayes predic construction}),
    then (iii-a): 
    $\expect[q\sim \bayesmarginalpredictor]{\convexEnv(\bayesprediction)} = \expect[q\sim \lambda]{\convexEnv(\bayesprediction)}$, and (iii-b): $\bayesmarginalpredictor\in\MPC(\lambda)$;
\end{enumerate}
then the predictor $\tilde{h}$ is the optimal $(\caliBudget, \ell_1)$-calibrated predictor.
\end{proposition}
The proof of above \Cref{prop:opt verify} is provided at the end of subsequent \Cref{subsec:proofs event-inde}.

\subsection{Proofs for \texorpdfstring{\Cref{thm:event-inde}, \Cref{prop:convex env,prop:num predictions,prop:opt verify}}{}}
\label{subsec:proofs event-inde}

In this section, we provide the proofs of \Cref{thm:event-inde}, \Cref{prop:convex env,prop:num predictions}.
Utilizing the two-step framework, we know that the optimal $\caliBudget$-calibrated predictor can be transformed into an optimal solution of \ref{eq:decoupled opt} and vice versa. Therefore, proving the structural results for the optimal predictor is equivalent to proving their analogs for the optimal solution in \ref{eq:decoupled opt}.
With a slight abuse of notation, we define $\supp(\miscali)\triangleq \{\prediction\in[0,1]:\exists\bayesprediction\in[0, 1], \miscali(\bayesprediction,\prediction) > 0\}$.

To prove \Cref{thm:event-inde}, we first show the following equivalent statement for the optimal solution to \ref{eq:decoupled opt}.
\begin{lemma}[Restatement of \Cref{thm:event-inde} for $(\bayesmarginalpredictor^*, \miscali^*)$]
\label{lem:event-inde decoupled reformuation}
In \ref{eq:decoupled opt} with $\normexponent = 1$, there exists $0 \leq \lowq \leq \highq\leq 1$
and an optimal solution $(\bayesmarginalpredictor^*,\miscali^*)$ that satisfies the following three properties:
\begin{enumerate}
    \item[(i)] For every $(\bayesprediction,\prediction)$ such that $\miscali^*(\bayesprediction,\prediction) > 0$, if $\prediction\in[0,\lowq]$ (resp.\ $\prediction\in[\lowq,\highq]$, $\prediction\in[\highq, 1]$), then $\prediction\leq \bayesprediction$ (resp.\ $\prediction = \bayesprediction$, $\prediction\geq \bayesprediction$).
    \item[(ii)] There exists a function $\sgselect:\supp(\miscali^*) \rightarrow \R$ that  selects a subgradient at each prediction $\prediction\in\supp(\miscali^*)$, i.e., $\sgselect(\prediction) \in \partial \indirectsenderU(\prediction)$, such that function $\sgselect$ is increasing, and it satisfies $\sgselect(\prediction) \equiv \caliBudgetDual$ for $\prediction\in[0,\lowq]$, and $\sgselect(\prediction) \equiv -\caliBudgetDual$ for all $\prediction\in[\highq,1]$ for some $\caliBudgetDual\ge 0$.

    \item[(iii)] All points $(\prediction, \indirectsenderU(\prediction))_{\prediction\in\supp(\miscali^*)}$ form a convex function.
    Moreover, points $(\prediction, \indirectsenderU(\prediction))$ for all $\prediction$ in $[\highq, 1]$ are collinear (with slope $\caliBudgetDual$), and points $(\prediction, \indirectsenderU(\prediction))$ for all $\prediction$ in $[0,\lowq]$ are collinear (with slope $-\caliBudgetDual$).
\end{enumerate}
\end{lemma}

We now prove \Cref{thm:event-inde} using our two-step framework and \Cref{lem:event-inde decoupled reformuation}.
\begin{proof}[Proof of \Cref{thm:event-inde}]
    Consider the optimal solution $(\bayesmarginalpredictor^*,\miscali^*)$ to \ref{eq:decoupled opt} that satisfies all three properties in \Cref{lem:event-inde decoupled reformuation}. Invoking \Cref{prop:prog equivalence}, it can be converted in to an optimal $(\caliBudget,\ell_1)$-calibrated predictor $\optpredictor$. Let $\truePosterior$ be the true expected outcome function induced by predictor $\optpredictor$. By construction, for every $(\bayesprediction,\prediction)$ such that $\miscali^*(\bayesprediction,\prediction) > 0$, we have $\truePosterior(\prediction) = \bayesprediction$. (This is guaranteed since $\bayesmarginalpredictor^*$ is a marginalized perfectly calibrated predictor and $\miscali^*$ is a feasible post-processing plan.) Thus, the ``Miscalibration Structure'' property in \Cref{thm:event-inde} is implied by the first property in \Cref{lem:event-inde decoupled reformuation}. Similarly, by the construction of predictor $\optpredictor$ in \Cref{defn:miscal procedure}, $\supp(\optpredictor) = \supp(\miscali^*)$. Therefore, the ``Payoff Structure-I'' and ``Payoff Structure-II'' properties in \Cref{thm:event-inde} are implied by the second and third properties in \Cref{lem:event-inde decoupled reformuation}. This finishes the proof of \Cref{thm:event-inde}.
\end{proof}

\xhdr{The proof of \Cref{lem:event-inde decoupled reformuation}}
We below provide the proof of \Cref{lem:event-inde decoupled reformuation}.
The proof of \Cref{lem:event-inde decoupled reformuation} will use the following property of marginalized perfectly calibrated predictor. 

\begin{restatable}[Contracting predictions]{fact}{factcontracting}
\label{fact:contracting}
For any marginalized predictor $\bayesmarginalpredictor\in\MPC(\prior)$, consider two predictions $\prediction_1, \prediction_2 \in \supp(\bayesmarginalpredictor)$ and let $\prediction\primed \triangleq \frac{\bayesmarginalpredictor(\prediction_1)\prediction_1 + \bayesmarginalpredictor(\prediction_2)\prediction_2}{\bayesmarginalpredictor(\prediction_1) + \bayesmarginalpredictor(\prediction_2)}$.
Consider another distribution $\newbayesmarginalpredictor\in\Delta([0, 1])$ satisfying
$\newbayesmarginalpredictor(\prediction) = \bayesmarginalpredictor(\prediction)$ for all $\prediction\in[0, 1]\setminus \{\prediction_1, \prediction_2, \prediction\primed\}$, and 
$\newbayesmarginalpredictor(\prediction\primed) = \bayesmarginalpredictor(\prediction\primed) + \bayesmarginalpredictor(\prediction_1) + \bayesmarginalpredictor(\prediction_2)$.
Then $\newbayesmarginalpredictor$ is also marginalized perfectly calibrated predictor, namely, $\newbayesmarginalpredictor\in\MPC(\prior)$.
\end{restatable}
\begin{proof}
This fact immediately follows from the Blackwell informativeness \citep{B-53} as the marginalized  predictor $\bayesmarginalpredictor$ is more informative than the predictor $\bayesmarginalpredictor\primed$.
\end{proof}
The above property of marginalized predictor $\bayesmarginalpredictor$ essentially says that the distribution $\newbayesmarginalpredictor$ obtained by contracting any predictions supported in $\bayesmarginalpredictor$ is also a feasible marginalized prediction -- it is in the set $\MPC(\prior)$.

\begin{proof}[Proof of Statement (i) in \Cref{lem:event-inde decoupled reformuation}]
We prove the statement by contradiction.
Suppose that there exists a solution $\bayesmarginalpredictor, \miscali$ to \ref{eq:decoupled opt} that satisfies: it exists $\bayesprediction_1, \bayesprediction_2$ with $\bayesprediction_1 < \bayesprediction_2$  such that there exist $\prediction_1, \prediction_2$ with $\miscali(\bayesprediction_1, \prediction_1) > 0, \miscali(\bayesprediction_2, \prediction_2) > 0$ satisfying $\prediction_1 > \bayesprediction_1$ but $\prediction_2 < \bayesprediction_2$.
For notation simplicity, we let $m_1 \triangleq \miscali(\bayesprediction_1, \prediction_1)$ and $m_2 \triangleq \miscali(\bayesprediction_2, \prediction_2)$. 
We define
\begin{align}
    \label{eq:pooling}
    \newq \triangleq  \frac{m_1\bayesprediction_1+m_2\bayesprediction_2}{m_1+m_2}~.
\end{align}
Clearly we have $\bayesprediction_1 < \newq < \bayesprediction_2$.

We first consider the scenario where $\prediction_1 < \prediction_2$, in which we consider following two possible cases:
\begin{itemize}
    \item When $\newq \le \prediction_1$ or $\newq \ge \prediction_2$.
    In this case, we construct a miscalibartion plan $\newmiscali$ as follows:
    \begin{equation}
    \begin{alignedat}{3}
        \label{eq:scenario one case one}
        \newmiscali(\bayesprediction, \cdot) & = \miscali(\bayesprediction, \cdot), \quad
        && \bayesprediction\in \supp(\bayesmarginalpredictor)\setminus\{\bayesprediction_1, \bayesprediction_2, \newq\}~;\\
        \newmiscali(\bayesprediction_i, \prediction) & = \miscali(\bayesprediction_i, \prediction), \quad
        && \prediction\in [0, 1]\setminus\{\prediction_1, \prediction_2\}, i\in[2]~ ;\\
        \newmiscali(\newq, \prediction) & = 
        \miscali(\newq, \prediction), \quad
        && \prediction\notin \{\prediction_1, \prediction_2\}~ ;\\
        \newmiscali(\newq, \prediction_1) & = m_1 + \miscali(\newq, \prediction_1), \quad
        && 
        \newmiscali(\newq, \prediction_2) = m_2 + \miscali(\newq, \prediction_2)~.
    \end{alignedat}
    \end{equation}
    Let $\newbayesmarginalpredictor$ be a marginalized predictor constructed from $\newmiscali$ as defined in \ref{eq:feasible miscal supp}.
    Essentially, $\newbayesmarginalpredictor$ pools all the mass of $\bayesmarginalpredictor$ on the points $\bayesprediction_1, \bayesprediction_2$ at a single point $\newq$.
    By \Cref{fact:contracting}, we know such $\newbayesmarginalpredictor \in \MPC(\lambda)$.
    
    We next show that $\newbayesmarginalpredictor, \newmiscali$ is a feasible solution to \ref{eq:decoupled opt}.
    To see this, it suffices to show that the calibration budget is still feasible under $\newmiscali$, i.e., satisfying \ref{eq:feasible miscal cal budget}.
    Thus, it suffices to show
    \begin{align}
        m_1\left|\newq - \prediction_1\right|
        +  m_2\left|\newq - \prediction_2\right| 
        \le 
        m_1\cdot (\prediction_1- \bayesprediction_1) +  m_2\cdot (\bayesprediction_2-\prediction_2)
        \label{ineq:calibra budget verifi new}
    \end{align}
    If $\newq \le \prediction_1$, we then have
    \begin{align*}
        \text{LHS of } \eqref{ineq:calibra budget verifi new}
        = 
        m_1 \left(\prediction_1 - \newq\right)
        +  m_2\left(\prediction_2 - \newq\right) 
        & \overset{(a)}{=} 
        m_1 \prediction_1+m_2\prediction_2 - (m_1 \bayesprediction_1 + m_2 \bayesprediction_2) \\
        & =
        m_1(\prediction_1-\bayesprediction_1) + m_2(\prediction_2-\bayesprediction_2)
        \overset{(b)}{<}  \text{RHS of } \eqref{ineq:calibra budget verifi new}~,
    \end{align*}
    where equality (a) is by definition of $\newq$, 
    and inequality (b) is by the fact that $\prediction_2 < \bayesprediction_2$.
    
    Similarly, if $\newq \ge \prediction_2$, we then have
    \begin{align*}
        \text{LHS of } \eqref{ineq:calibra budget verifi new}
        = 
        m_1 \left(\newq - \prediction_1\right)
        +  m_2\left(\newq - \prediction_2\right) 
        & = m_1(\bayesprediction_1 - \prediction_1)+m_2(\bayesprediction_2 - \prediction_2)
        \overset{(a)}{<}  \text{RHS of } \eqref{ineq:calibra budget verifi new}~,
    \end{align*}
    where inequality (a) is by the fact that $\prediction_1 > \bayesprediction_1$.
    Thus, we have shown that $\ECE{\newmiscali} < \ECE{\miscali}$.

    We further notice that, by construction, the solution $\newbayesmarginalpredictor, \newmiscali$ yields the same objective value as the solution $\bayesmarginalpredictor, \miscali$ in \ref{eq:decoupled opt}.

    \item When $\prediction_1 < \newq  <  \prediction_2$.
    For this case, we consider following two possible sub-cases:
    \begin{itemize}
        \item 
        When $\indirectsenderU(\prediction_1) \ge \indirectsenderU(\prediction_2)$.
        In this sub-case, 
        we construct $\newmiscali$ as follows:
        \begin{alignat*}{3}
            \newmiscali(\bayesprediction, \cdot) & = \miscali(\bayesprediction, \cdot), \quad
            && \bayesprediction\in \supp(\bayesmarginalpredictor)\setminus\{\bayesprediction_1, \bayesprediction_2,\prediction_1\}~;\\
            \newmiscali(\bayesprediction_i, \prediction) & = \miscali(\bayesprediction_i, \prediction), \quad
            && \prediction\in [0, 1]\setminus\{\prediction_1, \prediction_2\}, i\in[2]~ ;\\
            \newmiscali(\prediction_1, \prediction) & = \miscali(\prediction_1, \prediction), \quad
            && \prediction \neq \prediction_1~ ;\\
            \newmiscali(\prediction_1, \prediction_1) & = m_1 + m' + \miscali(\prediction_1, \prediction_1), \quad
            && 
            \newmiscali(\bayesprediction_2, \prediction_2) = m_2 - m'~,
        \end{alignat*}
        where $m'$ satisfies that 
        $\newq = \frac{(m_1+m')\cdot \prediction_1 + (m_2-m')\cdot \bayesprediction_2}{m_1+m_2}$.
        Notice that since $\newq \in (\prediction_1, \prediction_2)$, 
        such $m'$ must exist.
        Thus, by \Cref{fact:contracting}, we must have the corresponding $\newbayesmarginalpredictor $ of $\newmiscali$ satisfying $\newbayesmarginalpredictor\in \MPC(\lambda)$.
        Moreover, it is easy to see that $\newmiscali$ has strictly smaller $\ECE{\newmiscali}$ compared to $\ECE{\miscali}$. 

        By construction, it is easy to see that $\newbayesmarginalpredictor, \newmiscali$ yields a weakly higher objective value compared to $\bayesmarginalpredictor, \miscali$ as their objective value difference is
        \begin{align*}
            \indirectsenderU(\prediction_1) \cdot\left(m_1+m'
            -
            m_1\right) +  
            \indirectsenderU(\prediction_2) \cdot\left(m_2-m'-m_2\right)  \ge 0~.
            \tag{due to $\indirectsenderU(\prediction_1)\ge \indirectsenderU(\prediction_2)$}
        \end{align*}

        \item 
        When $\indirectsenderU(\prediction_1) < \indirectsenderU(\prediction_2)$.
        In this sub-case, 
        we construct $\newmiscali$ as follows:
        \begin{alignat*}{3}
            \newmiscali(\bayesprediction, \cdot) & = \miscali(\bayesprediction, \cdot), \quad
            && \bayesprediction\in \supp(\bayesmarginalpredictor)\setminus\{\bayesprediction_1, \bayesprediction_2,\prediction_2\}~;\\
            \newmiscali(\bayesprediction_i, \prediction) & = \miscali(\bayesprediction_i, \prediction), \quad
            && \prediction\in [0, 1]\setminus\{\prediction_1, \prediction_2\}, i\in[2]~ ;\\
            \newmiscali(\prediction_2, \prediction) & = \miscali(\prediction_2, \prediction), \quad
            && \prediction\neq \prediction_2~ ;\\
            \newmiscali(\prediction_2, \prediction_2) & = m_2 + m' + \miscali(\prediction_2, \prediction_2), \quad
            && 
            \newmiscali(\bayesprediction_1, \prediction_1) = m_1 - m'~,
        \end{alignat*}
        where $m'$ satisfies that 
        \begin{align*}
            \newq = \frac{(m_1-m')\cdot \prediction_1 + (m_2+m')\cdot \bayesprediction_2}{m_1+m_2}~.
        \end{align*}
        Notice that since $\newq \in (\prediction_1, \prediction_2)$, 
        such $m'$ must exist. 
        Thus, by \Cref{fact:contracting}, we must have $\newbayesmarginalpredictor $ of $\newmiscali$ satisfying $\newbayesmarginalpredictor\in \MPC(\lambda)$.
        Moreover, we can also see $\ECE{\newmiscali}<\ECE{\miscali}$. 

        By construction, it is easy to see that $\newbayesmarginalpredictor, \newmiscali$ yields a weakly higher objective value compared to $\bayesmarginalpredictor, \miscali$ as their objective value difference is
        \begin{align*}
            \indirectsenderU(\prediction_1) \cdot\left(m_1-m'
            -
            m_1\right) +  
            \indirectsenderU(\prediction_2) \cdot\left(m_2+m'-m_2\right)  \ge 0~.
            \tag{due to $\indirectsenderU(\prediction_1) < \indirectsenderU(\prediction_2)$}
        \end{align*}
    \end{itemize}
\end{itemize}
We now consider the scenario where $\prediction_1 \ge \prediction_2$.
\begin{itemize}

    \item 
    When either $\newq \ge \prediction_1$ or $\newq \le \prediction_2$,
    similar to the scenario where $\prediction_1 < \prediction_2$, we can construct $\newmiscali$ (and its $\newbayesmarginalpredictor$) as we defined in \eqref{eq:scenario one case one}. In doing so, we can still ensure that $\newbayesmarginalpredictor, \newmiscali$ is a feasible solution to \ref{eq:decoupled opt} and it yields the same objective value with $\bayesmarginalpredictor, \miscali$.

    \item 
    When $\prediction_2 < \newq < \prediction_1$, for this case, we also construct $\newmiscali$ (and its $\newbayesmarginalpredictor$) as we defined in \eqref{eq:scenario one case one}.
    It is easy to see that $\newmiscali$ has strictly smaller $\ECE{\newmiscali}$ compared to $\ECE{\miscali}$, namely:
    \begin{align*}
        \ECE{\newmiscali} - \ECE{\miscali}
        & =  m_1\cdot (\prediction_1-\bayesprediction_1) + m_2\cdot(\bayesprediction_2-\prediction_2) - m_1 \cdot (\prediction_1 - \newq) - m_2\cdot (\newq - \prediction_2) \\
        & =
        m_1\cdot (\newq-\bayesprediction_1) + m_2\cdot (\bayesprediction_2-\newq)> 0.
    \end{align*}
    Meanwhile, the solution $\newbayesmarginalpredictor, \newmiscali$
    yields the same objective value with $\bayesmarginalpredictor, \miscali$.
\end{itemize}
Putting all pieces together, we can finish the proof.
\end{proof}

\begin{proof}[Proof of Statements (ii) and (iii) in \Cref{lem:event-inde decoupled reformuation}]
The proof follows by the complementary slackness. 
We 
consider the following Lagrange for \ref{eq:decoupled opt} with multipliers 
$\caliBudgetDual\ge 0$, 
$\dualvarmiscal(\bayesprediction, \prediction) \ge 0$, $\dualvarsupply(\bayesprediction)\in\R, \dualvarbayes(\bayesprediction)\ge 0, \dualvarbayesup \in \R, \dualvarMPC(t) \ge 0$
\begin{align*}
    \Lagrange = 
    & \int_0^1\int_0^1 
    \miscali(\bayesprediction, \prediction ) \cdot \indirectsenderU\left( \prediction \right) 
    \, \d \prediction \d \bayesprediction
    + 
    \caliBudgetDual\cdot \left(\caliBudget - \int_0^1\int_0^1 \miscali(\bayesprediction, \prediction ) \cdot |\bayesprediction-\prediction|
    \, \d \prediction \d \bayesprediction
    \right) \\
    & + \int_0^1 \dualvarsupply(\bayesprediction)\cdot \left(\int_0^1\miscali(\bayesprediction, \prediction)~\d  \prediction - \bayesmarginalpredictor(\bayesprediction)\right)~\d \bayesprediction + 
    \int_0^1\int_0^1 
    \dualvarmiscal(\bayesprediction, \prediction)
    \miscali(\bayesprediction, \prediction )\, \d \prediction \d \bayesprediction \\
    & 
    + \dualvarbayesup \left(\int_0^1 \bayesmarginalpredictor(\bayesprediction) \d \bayesprediction - 1\right)
    +  \int_0^1  \dualvarbayes(\bayesprediction) \bayesmarginalpredictor(\bayesprediction) \, \d \bayesprediction
    + \int_0^1 \dualvarMPC(t) \left(\int_0^t \int_0^x \prior(\bayesprediction)\, \d \bayesprediction \d x - \int_0^t \int_0^x g(\bayesprediction)\, \d \bayesprediction \d x\right)\, \d t \\
    & + \dualvarbayesmean\cdot \left(\int_0^1\bayesprediction \bayesmarginalpredictor(\bayesprediction)\,\d\bayesprediction - \priormean\right)
\end{align*}
By the first-order condition of optimal miscalibration plan $\miscali$, we must have 
\begin{alignat*}{3}
    \indirectsenderU(\prediction) - \caliBudgetDual\cdot |\bayesprediction-\prediction|  +  \dualvarsupply(\bayesprediction) + \dualvarmiscal(\bayesprediction, \prediction)
    & = 0, \quad 
    && \forall \bayesprediction, \prediction~.
\end{alignat*}
Since $\dualvarmiscal(\bayesprediction, \prediction) \ge 0$ for all $\bayesprediction, \prediction$, we have
\begin{align*}
    \indirectsenderU(\prediction) - \caliBudgetDual\cdot |\bayesprediction-\prediction|  
    & =   - \dualvarsupply(\bayesprediction),
    \quad 
    \forall \bayesprediction, \prediction \text{ with }\miscali(\bayesprediction, \prediction) > 0\\
    \indirectsenderU(\prediction) - \caliBudgetDual\cdot |\bayesprediction-\prediction|  
    & \le   - \dualvarsupply(\bayesprediction),
    \quad 
    \forall \bayesprediction, \prediction 
\end{align*}
The first-order condition of $\bayesmarginalpredictor$ also implies:
\begin{align*}
    \dualvarbayes(\bayesprediction) - \dualvarsupply(\bayesprediction) -  \int_\bayesprediction^1 \int_x^1 \dualvarMPC(t)\, \d t \d x + \dualvarbayesup + \bayesprediction\dualvarbayesmean
    & = 0, \quad 
    \forall \bayesprediction~.
\end{align*}
Note that $\miscali(\bayesprediction, \prediction) > 0$ must also imply that $\bayesmarginalpredictor(\bayesprediction) > 0$, which implies that $\dualvarmiscal(\bayesprediction, \prediction) = \dualvarbayes(\bayesprediction) = 0$.
We next define the function $\convexEnv(\cdot):[0, 1]\rightarrow\R$:
\begin{align*}
    \convexEnv(\bayesprediction) \triangleq 
    \int_\bayesprediction^1 \int_x^1 \dualvarMPC(t)\, \d t \d x - \dualvarbayesup- \bayesprediction \dualvarbayesmean~.
\end{align*}
It is easy to see that function $\convexEnv$ is convex.
\begin{claim}
\label{claim:convex gamma}
Function $\convexEnv$ is a convex function.
\end{claim}
With the above definition, we have the following condition for the optimal solution $\miscali, \bayesmarginalpredictor$:
\begin{align}
    \label{ineq:convex ub}
    \indirectsenderU(\prediction) - \caliBudgetDual\cdot |\bayesprediction-\prediction| & 
    \le \convexEnv(\bayesprediction)~,
    \quad 
    \forall \bayesprediction, \prediction~. \\
    \label{eq:complementary slackness}
    \indirectsenderU(\prediction) - \caliBudgetDual\cdot |\bayesprediction-\prediction| 
    & = \convexEnv(\bayesprediction)~,
    \quad 
    \text{if }\miscali(\bayesprediction, \prediction) > 0~. 
\end{align}
From the above two conditions, we can deduce the following observation:
\begin{claim}
\label{claim:pointwise large}
Function $\convexEnv(\prediction) \ge \indirectsenderU(\prediction)$ for all $\prediction\in[0, 1]$, and $\convexEnv(\prediction) = \indirectsenderU(\prediction)$ if $\miscali(\prediction, \prediction) > 0$.
\end{claim}
Function $\convexEnv(\prediction) \ge \indirectsenderU(\prediction)$ for all $\prediction\in[0, 1]$ is due to Eqn.~\eqref{ineq:convex ub} w.r.t.\ $(\prediction, \prediction)$, and $\convexEnv(\prediction) = \indirectsenderU(\prediction)$ if $\miscali(\prediction, \prediction) > 0$ is due to Eqn.~\eqref{eq:complementary slackness} w.r.t.\ $\miscali(\prediction, \prediction) > 0$.

Now let the predictions $\highq, \lowq$ and the sets $\setShigh, \setSlow$ be defined as follows:
\begin{align*}
    \highq 
    & \triangleq \inf_{\prediction\in[0, 1]}\{\prediction: \miscali(\bayesprediction, \prediction) > 0\wedge \prediction> \bayesprediction\}, \quad
    \setShigh \triangleq \{\prediction\in[\highq, 1]: \miscali(\bayesprediction, \prediction) > 0, \bayesprediction\in[0,1]\}~;\\
    \lowq 
    & \triangleq \sup_{\prediction\in[0, 1]}\{\prediction: \miscali(\bayesprediction, \prediction) > 0\wedge \prediction < \bayesprediction\}, \quad
    \setSlow \triangleq \{\prediction\in[0, \lowq]: \miscali(\bayesprediction, \prediction) > 0, \bayesprediction\in[0,1]\}~.
\end{align*}

\noindent\textbf{Collinearity.}
We next argue that all points $(\prediction, \indirectsenderU(\prediction))_{\prediction\in\setShigh}$ are all collinear.
For the above defined $\highq$, let $\bayesprediction_{\cc{H}}$ be the point that $\miscali(\bayesprediction, \highq) > 0$.
From \eqref{ineq:convex ub} and \eqref{eq:complementary slackness}, we know
\begin{align*}
    \indirectsenderU(\prediction) - \caliBudgetDual\cdot (\prediction - \bayesprediction_{\cc{H}}) 
    & \le \convexEnv(\bayesprediction_{\cc{H}}), \quad \forall \prediction \ge \bayesprediction_{\cc{H}}\\
    \indirectsenderU(\highq) - \caliBudgetDual\cdot (\highq - \bayesprediction_{\cc{H}}) 
    & = \convexEnv(\bayesprediction_{\cc{H}}), 
    \tag{$\miscali(\bayesprediction_{\cc{H}}, \highq) > 0$}
\end{align*}
Thus, for any $\prediction \ge \highq$, we have that 
\begin{align}
    \label{ineq:condi 3}
    \indirectsenderU(\prediction) 
    \le \indirectsenderU(\highq) + \caliBudgetDual\cdot (\prediction - \highq)~.
    % . \\
    % \wtr{\indirectsenderU(\prediction) 
    % \le \indirectsenderU(\highq) - \caliBudgetDual(\highq - \bayesprediction_{\cc{H}})^2 + \caliBudgetDual\cdot (\prediction - \bayesprediction_{\cc{H}})^2} 
\end{align}
On the other hand, for any $\prediction\in\setShigh$ with $\miscali(\bayesprediction, \prediction) > 0$, we know
\begin{align*}
    \indirectsenderU(\prediction) - \caliBudgetDual\cdot(\prediction-\bayesprediction) 
    & = \convexEnv(\bayesprediction)
    \tag{by \eqref{eq:complementary slackness} with $\miscali(\bayesprediction, \prediction) > 0$}\\
    & \ge \convexEnv(\bayesprediction_{\cc{H}}) + \caliBudgetDual\cdot (\bayesprediction - \bayesprediction_{\cc{H}})
    \tag{by \Cref{claim:convex gamma} and \Cref{claim:pointwise large}}
    \\
    \Rightarrow~ 
    \indirectsenderU(\prediction) & \ge 
    \convexEnv(\bayesprediction_{\cc{H}}) + \caliBudgetDual\cdot (\prediction - \bayesprediction_{\cc{H}})\\
    & = 
    \indirectsenderU(\highq) - \caliBudgetDual\cdot (\highq - \bayesprediction_{\cc{H}})
    + 
    \caliBudgetDual\cdot (\prediction - \bayesprediction_{\cc{H}})\\
    & = 
    \indirectsenderU(\highq) + \caliBudgetDual\cdot (\prediction - \highq)~.
\end{align*}
Thus, for any $\prediction\in\setShigh$ we must have 
\begin{align*}
    \indirectsenderU(\prediction) = \indirectsenderU(\highq) + \caliBudgetDual\cdot (\prediction - \highq)~,
\end{align*}
which shows that all points $(\prediction, \indirectsenderU(\prediction))_{\prediction\in\setShigh}$ must be collinear with a slope of $\caliBudgetDual$.
Similar arguments can be applied for predictions in the set $\setSlow$.

\noindent\textbf{Convexity of points $(\prediction, \indirectsenderU(\prediction))_{\prediction\in\supp(\miscali)}$.}
We next argue that all points $(\prediction, \indirectsenderU(\prediction))_{\prediction\in\supp(\miscali)}$
form a convex function. 
We define $\highq\primed \triangleq \sup\{\prediction\in\supp(\miscali): \prediction < \highq\}$. 
We next argue that under optimal miscalibration plan $\miscali$, we must have
\begin{align}
    \label{ineq:boundary slope}
    \frac{\indirectsenderU(\highq) - \indirectsenderU(\highq\primed)}{\highq-\highq\primed} \le \caliBudgetDual~.
\end{align}
To see the above inequality,  we have
\begin{align*}
    \indirectsenderU(\highq) - \caliBudgetDual \cdot (\highq - \highq\primed) 
    & \le \convexEnv(\highq\primed) 
    \tag{by~\eqref{ineq:convex ub} w.r.t.\ $(\highq, \highq\primed)$}\\
    & = \indirectsenderU(\highq\primed)
    \tag{by \Cref{claim:pointwise large}}
\end{align*}
Rearranging the above inequality can give us Eqn.~\eqref{ineq:boundary slope}. 
For all predictions $\prediction\in[\lowq, \highq]\cap\supp(\miscali)$, we know that $\convexEnv(\prediction) = \indirectsenderU(\prediction)$.
Since we know $\convexEnv(\cdot)$ is a convex function, thus all points $(\prediction, \indirectsenderU(\prediction))_{\prediction\in[\lowq, \highq]\cap\supp(\miscali)}$ must also form a convex function.
Since we also establish Eqn.~\eqref{ineq:boundary slope} and we know that all points $(\prediction, \indirectsenderU(\prediction))_{\prediction\in\setShigh}$ are collinear, we know points $(\prediction, \indirectsenderU(\prediction))_{\prediction\in[\lowq, 1]\cap\supp(\miscali)}$ must be convex.
Similar arguments can be established for the predictions in $[0, \lowq]$. We thus prove the convexity of points $(\prediction, \indirectsenderU(\prediction))_{\prediction\in\supp(\miscali)}$ for optimal miscalibration plan $\miscali$.
The proof then finishes.
\end{proof}

To prove \Cref{prop:convex env}, we first show the following equivalent statement for the optimal solution $(\bayesmarginalpredictor^*, \miscali^*)$ to \ref{eq:decoupled opt}.
\begin{lemma}[Restatement of \Cref{prop:convex env} for $(\bayesmarginalpredictor^*, \miscali^*)$]
\label{lem:convex env decoupled formluation}
In \ref{eq:decoupled opt} with $\normexponent = 1$, let $(\bayesmarginalpredictor^*,\miscali^*)$ be the optimal solution to \ref{eq:decoupled opt} that satisfies the properties in \Cref{lem:event-inde decoupled reformuation}, 
then there exists a symmetric linear-tailed convex function $\convexEnv:[0,1]\rightarrow\R$ 
such that 
(i) $\convexEnv(\prediction) \ge \indirectsenderU(\prediction)$ for all $\prediction\in[0, 1]$; 
(ii) $\supp(\miscali^*) \subseteq \{\prediction\in[0, 1]: \convexEnv(\prediction) = \indirectsenderU(\prediction)\}$; 
(iii) it is linear over $[0, \lowq]$ and $[\highq, 1]$ where $\lowq, \highq$ are correspondingly defined in \Cref{lem:event-inde decoupled reformuation}
\end{lemma}
\begin{proof}[Proof of \Cref{prop:convex env}]
    Consider the optimal solution $(\bayesmarginalpredictor^*,\miscali^*)$ to \ref{eq:decoupled opt} that satisfies all properties in \Cref{lem:num predictions decoupled formluation}. 
    Invoking \Cref{prop:prog equivalence}, it can be converted in to an optimal $(\caliBudget,\ell_1)$-calibrated predictor $\optpredictor$. 
    Let $\truePosterior$ be the true expected outcome function induced by predictor $\optpredictor$. 
    By construction, for every $(\bayesprediction,\prediction)$ such that $\miscali^*(\bayesprediction,\prediction) > 0$, we have $\truePosterior(\prediction) = \bayesprediction$. (This is guaranteed since $\bayesmarginalpredictor^*$ is a marginalized perfectly calibrated predictor and $\miscali^*$ is a feasible post-processing plan.) 
    
    Thus, the Statement (i) in \Cref{prop:convex env} is implied by the Statement (i) in \Cref{lem:event-inde decoupled reformuation}, Statement (ii) and (iii) in \Cref{prop:convex env} is also implied accordingly as we have $\supp(\optpredictor) = \supp(\miscali^*)$.
\end{proof}
\xhdr{The proof of \Cref{lem:convex env decoupled formluation}}
We next provide proof for \Cref{lem:convex env decoupled formluation}. Its analysis is similar to the duality argument used in the proof of Statements (ii) and (iii) in \Cref{lem:event-inde decoupled reformuation}.
\begin{proof}[Proof of \Cref{lem:convex env decoupled formluation}]
To simplify the presentation, we shorthand the optimal solution $(\bayesmarginalpredictor^*,\miscali^*)$ as $(\bayesmarginalpredictor,\miscali)$.
Recall the convex function $\convexEnv$ and the Lagrange multiplier $\caliBudgetDual$ introduced in the proof of \Cref{lem:event-inde decoupled reformuation}. 
Conditions~\eqref{ineq:convex ub} and \eqref{eq:complementary slackness} are 
\begin{align*}
    \indirectsenderU(\prediction) - \caliBudgetDual\cdot |\bayesprediction-\prediction| & 
    \le \convexEnv(\bayesprediction)~,
    \quad 
    \forall \bayesprediction, \prediction~. \\
    \indirectsenderU(\prediction) - \caliBudgetDual\cdot |\bayesprediction-\prediction| 
    & = \convexEnv(\bayesprediction)~,
    \quad 
    \text{if }\miscali(\bayesprediction, \prediction) > 0~. 
\end{align*}
Define $\lowx \triangleq \bayesprediction_{\cc{L}}$ and $\highx\triangleq \bayesprediction_{\cc{H}}$ where $\bayesprediction_\cc{L}$ (namely it satisfies $\miscali(\bayesprediction_\cc{L}, \lowq) > 0$) and $\bayesprediction_\cc{H}$ (namely it satisfies $\miscali(\bayesprediction_\cc{H}, \highq) > 0$) are defined in the proof of \Cref{lem:event-inde decoupled reformuation}. 
Now we construct a symmetric linear-tailed convex function $\convexEnv\primed$ satisfying the all properties in \Cref{lem:convex env decoupled formluation} as follows:
\begin{align*}
    \convexEnv\primed(\prediction) \triangleq\left\{
    \begin{array}{ll}
     \convexEnv(\lowx) + \caliBudgetDual(\lowx-\prediction)    & \quad \text{if $\prediction\in[0,\lowq]$}
     \\
     \convexEnv(\prediction)    & \quad \text{if $\prediction\in[\lowq,\highq]$}
     \\
     \convexEnv(\highx) + \caliBudgetDual(\prediction-\highx)    & \quad \text{if $\prediction\in[\highq,1]$}
    \end{array}
    \right.
\end{align*}
We first verify that the constructed function $\convexEnv\primed$ is a symmetric linear-tailed convex function. By construction, function $\convexEnv\primed$ has symmetric linear tails and is convex for every $\prediction\in[\lowq,\highq]$. In addition, due to the definition of $\lowx,\highx$, the linear slope $\caliBudgetDual \geq -\convexEnv'(\lowx)$ and $\caliBudgetDual\geq \convexEnv'(\highx)$ and thus function $\convexEnv\primed$ is convex for all $\prediction\in[0, 1]$. 

We next verify Statement (i). By construction, for every prediction $\prediction\in[\lowq,\highq]$, $\convexEnv\primed(\prediction) = \convexEnv(\prediction) \geq \indirectsenderU(\prediction)$ due to condition~\eqref{ineq:convex ub}. For every prediction $\prediction\in[0, \lowq]$, $\convexEnv\primed(\prediction) = \convexEnv(\lowx) + \caliBudgetDual(\lowx-\prediction) \geq \indirectsenderU(\prediction)$ due to condition~\eqref{ineq:convex ub} at $\lowx$ and the collinearity in the third property in \Cref{lem:event-inde decoupled reformuation}. For every prediction $\prediction\in[\highq, 1]$, a symmetric argument applies.

We now verify Statement (ii). By construction, for every prediction $\prediction\in[\lowq,\highq]\cap\supp(\miscali)$, $\convexEnv\primed(\prediction) = \convexEnv(\prediction) = \indirectsenderU(\prediction)$ due to the first property in \Cref{lem:event-inde decoupled reformuation} and condition~\eqref{eq:complementary slackness}. Similarly, for every prediction $\prediction\in[0, \lowq]\cap\supp(\miscali)$,
$\convexEnv\primed(\prediction) = \convexEnv(\lowx) + \caliBudgetDual(\lowx-\prediction) = \indirectsenderU(\prediction)$, due to the collinearity in the third property in \Cref{lem:event-inde decoupled reformuation} and condition~\eqref{eq:complementary slackness}. For every prediction $\prediction\in[\highq, 1]\cap\supp(\miscali)$, a symmetric argument applies.

Finally, Statement~(iii) holds due to the definition of $\lowx,\highx$ and the first property in \Cref{lem:event-inde decoupled reformuation}, which implies $\lowx \geq \lowq$ and $\highx \leq \highq$. This finishes the proof of \Cref{lem:convex env decoupled formluation}.
\end{proof}

To prove \Cref{prop:num predictions}, we first show the following equivalent statement for the optimal solution $(\bayesmarginalpredictor^*, \miscali^*)$ to \ref{eq:decoupled opt}.
\begin{lemma}[Restatement of \Cref{prop:num predictions} for $(\bayesmarginalpredictor^*, \miscali^*)$]
\label{lem:num predictions decoupled formluation}
In \ref{eq:decoupled opt} with $\normexponent = 1$, let $(\bayesmarginalpredictor^*,\miscali^*)$ be the optimal solution to \ref{eq:decoupled opt} that satisfies the properties in \Cref{lem:event-inde decoupled reformuation}, and in addition,
\begin{enumerate}
    \item[(i)] 
    we have $|\supp(\miscali^*)| \le \numData + 2$.
    
    \item[(ii)]
    for each event $i\in[\numData]$,  we have $|\{\prediction\in [0, 1]: \bayespredictor^*_i(\bayesprediction) > 0 \text{ and } \miscali^*(\bayesprediction, \prediction) > 0, \bayesprediction\in[0, 1]\}| \le 4$.
    
    \item[(iii)]
    for each $q\in\supp(\bayesmarginalpredictor^*)$, we have $\left|\{\prediction \in[0, 1]: \miscali^*(\bayesprediction, \prediction ) > 0\}\right| \le 2$.
    Moreover, if the equality holds (i.e., $\left|\{\prediction \in[0, 1]: \miscali^*(\bayesprediction, \prediction ) > 0\}\right| = 2$), at least one prediction $\prediction$ belongs to the under-confidence $[0,\lowq]$ or over-confidence interval $[\highq,1]$, i.e., $\prediction\in [0,\lowq]\cup[\highq,1]$.
\end{enumerate}
\end{lemma}
\begin{proof}[Proof of \Cref{prop:num predictions}]
    Consider the optimal solution $(\bayesmarginalpredictor^*,\miscali^*)$ to \ref{eq:decoupled opt} that satisfies all properties in \Cref{lem:num predictions decoupled formluation}. 
    Invoking \Cref{prop:prog equivalence}, it can be converted in to an optimal $(\caliBudget,\ell_1)$-calibrated predictor $\optpredictor$. 
    Let $\truePosterior$ be the true expected outcome function induced by predictor $\optpredictor$. 
    By construction, for every $(\bayesprediction,\prediction)$ such that $\miscali^*(\bayesprediction,\prediction) > 0$, we have $\truePosterior(\prediction) = \bayesprediction$. (This is guaranteed since $\bayesmarginalpredictor^*$ is a marginalized perfectly calibrated predictor and $\miscali^*$ is a feasible post-processing plan.) 
    
    Thus, the Statement (i) in \Cref{prop:num predictions} is implied by the first property in \Cref{lem:event-inde decoupled reformuation}, as we have $\supp(\optpredictor) = \supp(\miscali^*)$. 
    Similarly, by the construction of predictor $\optpredictor$ in \Cref{defn:miscal procedure}, Statements (ii) and (iii) properties in \Cref{prop:num predictions} are implied by the second and third statements in \Cref{lem:event-inde decoupled reformuation}. This finishes the proof of \Cref{prop:num predictions}.
\end{proof}

\xhdr{The proof of \Cref{lem:num predictions decoupled formluation}}
We below provide the proof for \Cref{lem:num predictions decoupled formluation}.
Before proceeding to the proof, we first prove the following monotonicity property of the prediction $\prediction$ w.r.t.\ its corresponding true expected outcome $\bayesprediction$.
\begin{lemma}[Monotonicity of the prediction $\prediction$ w.r.t.\ its $\bayesprediction$ with $\miscali(\bayesprediction, \prediction) > 0$]
\label{lem:monotone devia}
It is without loss to consider a solution $\bayesmarginalpredictor, \miscali$ to \ref{eq:decoupled opt} such that for any $\bayesprediction_1 < \bayesprediction_2$, we have $\max\{\prediction\in[0, 1]: \miscali(\bayesprediction_1, \prediction)> 0\}  \le \min\{\prediction\in[0, 1]: \miscali(\bayesprediction_2, \prediction)> 0\}$. 
\end{lemma}
\begin{proof}[Proof of \Cref{lem:monotone devia}]
Fix a feasible solution $\bayesmarginalpredictor, \miscali$ to \ref{eq:decoupled opt}.
Fix any  $\bayesprediction_1 < \bayesprediction_2$.
Let $\prediction_1 \triangleq \max\{\prediction\in[0, 1]: \miscali(\bayesprediction_1, \prediction)> 0\}$ and $\prediction_2 \triangleq \min\{\prediction\in[0, 1]: \miscali(\bayesprediction_2, \prediction)> 0\}$.
For notation simplicity, we let $m_1 \triangleq \miscali(\bayesprediction_1, \prediction_1)$ and $m_2 \triangleq \miscali(\bayesprediction_2, \prediction_2)$. 

Let us suppose $\prediction_1 > \prediction_2$. We construct the following $\newmiscali$ (and its $\newbayesmarginalpredictor$):
\begin{alignat*}{3}
    \newmiscali(\bayesprediction, \cdot) & = \miscali(\bayesprediction, \cdot), \quad
    && \bayesprediction\in \supp(\miscali)\setminus\{\bayesprediction_1, \bayesprediction_2\}~;\\
    \newmiscali(\bayesprediction_i, \prediction) & = \miscali(\bayesprediction_i, \prediction), \quad
    && \prediction\in [0, 1]\setminus\{\prediction_1, \prediction_2\}, i\in[2]~ ;\\
    \newmiscali(\bayesprediction_1, \prediction_1) 
    & = (m_1 - m_2)\indicator{m_1 \ge m_2}, \quad
    && \newmiscali(\bayesprediction_1, \prediction_2) 
    = m_1 + \miscali(\bayesprediction_1, \prediction_2) - (m_1 - m_2)\indicator{m_1 \ge m_2}; \\
    \newmiscali(\bayesprediction_2, \prediction_2) 
    & = (m_2-m_1)\indicator{m_1 < m_2}, \quad
    && \newmiscali(\bayesprediction_2, \prediction_1) 
    = m_2 + \miscali(\bayesprediction_2, \prediction_1) - (m_2-m_1)\indicator{m_1 < m_2}.
\end{alignat*}
By construction, it is easy to see that $\newbayesmarginalpredictor\in \MPC(\lambda)$, and $\newmiscali$ has the same expected calibration error with $\miscali$:
$\ECE{\newmiscali} = \ECE{\miscali}$.
Thus, $\newbayesmarginalpredictor, \newmiscali$ is also a feasible solution to \ref{eq:decoupled opt}.
In the meantime, under $\newmiscali$, we have $\max\{\prediction\in[0, 1]: \newmiscali(\bayesprediction_1, \prediction)> 0\} \le \min\{\prediction\in[0, 1]: \newmiscali(\bayesprediction_2, \prediction)> 0\}$.
Moreover, it is easy see that 
\begin{align*}
    \newmiscali(\bayesprediction_1, \prediction_1) + \newmiscali(\bayesprediction_2, \prediction_1) 
    & = 
    \miscali(\bayesprediction_1, \prediction_1) + \miscali(\bayesprediction_2, \prediction_1) \\
    \newmiscali(\bayesprediction_1, \prediction_2) + \newmiscali(\bayesprediction_2, \prediction_2) 
    & = 
    \miscali(\bayesprediction_1, \prediction_2) + \miscali(\bayesprediction_2, \prediction_2)~.
\end{align*}
Thus, solution $\newbayesmarginalpredictor, \newmiscali$ yields the same objective value with $\bayesmarginalpredictor, \miscali$,
which finishes the proof.
\end{proof}
\begin{proof}[Proof of Statement (i) in \Cref{lem:num predictions decoupled formluation}]
Suppose that there exists an optimal solution $\bayesmarginalpredictor, \miscali$ to the program \ref{eq:decoupled opt} such that there exist predictions $\bayesprediction_1\le \prediction_1 <  \prediction_2 < \prediction_3$ with $\miscali(\bayesprediction_1, \prediction_1) > 0, \miscali(\bayesprediction_1, \prediction_2) > 0, \miscali(\bayesprediction_2, \prediction_3)> 0$
for some $\bayesprediction_1 < \bayesprediction_2$.
From \Cref{lem:monotone devia}, we can without loss consider $\prediction_3 > \prediction_2$.
Let $t_i \triangleq \prediction_i-\bayesprediction_1$ for $i\in[3]$.
From the Statement (iii) in \Cref{lem:event-inde decoupled reformuation}, we have
\begin{align}
    \label{eq:collinearity}
    \frac{\indirectsenderU(\prediction_2) - \indirectsenderU(\prediction_1)}{t_2 - t_1} 
    =
    \frac{\indirectsenderU(\prediction_3) - \indirectsenderU(\prediction_2)}{t_3-t_2}~.
\end{align}
We next construct $\newbayesmarginalpredictor, \newmiscali$ as follows:
\begin{alignat*}{3}
    \newmiscali(\bayesprediction, \cdot) & = \miscali(\bayesprediction, \cdot), \quad
    && \bayesprediction\in \supp(\bayesmarginalpredictor)\setminus\{\bayesprediction_1, \bayesprediction_3\}~;\\
    \newmiscali(\bayesprediction_1, \prediction) & = \miscali(\bayesprediction_1, \prediction), \quad
    && \prediction\in [0, 1]\setminus\{\prediction_1, \prediction_2\} ~ ;\\
    \newmiscali(\bayesprediction_1, \prediction_1) & = \miscali(\bayesprediction_1, \prediction_1)
    + 
    \miscali(\bayesprediction_1, \prediction_2) \cdot \frac{t_3-t_2}{t_3-t_1}
    &&\\
    \newmiscali(\bayesprediction_1, \prediction_3) & = 
    \miscali(\bayesprediction_1, \prediction_2) \cdot \frac{t_2-t_1}{t_3-t_1}
    &&
\end{alignat*}
The above construction essentially reallocates the mass $\miscali(\bayesprediction_1, \prediction_2)$ on prediction $\prediction_2$ to the predictions $\prediction_1$ and $\prediction_3$. 
Such reallocation ensures that $\ECE{\newmiscali} = \ECE{\miscali}$, which can be seen as follows:
\begin{align*}
    \ECE{\newmiscali} - \ECE{\miscali} 
    & = 
    \left(\miscali(\bayesprediction_1, \prediction_1) + \miscali(\bayesprediction_1, \prediction_2)\cdot \frac{t_3 - t_2}{t_3 - t_1}\right)\cdot 
    (\prediction_1 - \bayesprediction_1) + 
    \miscali(\bayesprediction_1, \prediction_2)\cdot \frac{\prediction_2 - \prediction_1}{t_3 - t_1}\cdot t_3 \\
    & \quad - \left(
    \miscali(\bayesprediction_1, \prediction_1)\cdot t_1 + \miscali(\bayesprediction_1, \prediction_2)\cdot t_2
    \right) \\
    & = 
    \miscali(\bayesprediction, \prediction_2)\cdot \left(\frac{t_3 - t_2}{t_3 - t_1} \cdot t_1 + \frac{t_2-t_1}{t_3 - t_1} \cdot t_3 - t_2\right)
    = 0~.
\end{align*}
Note that the above construction ensures that $\newbayesmarginalpredictor = \bayesmarginalpredictor$, and moreover, the objective value difference between the two solutions $\newbayesmarginalpredictor, \newmiscali$ and $\bayesmarginalpredictor, \miscali$ can be computed as
\begin{align*}
    & \left(\miscali(\bayesprediction_1, \prediction_1) + \miscali(\bayesprediction_1, \prediction_2)\cdot \frac{t_3 - t_2}{t_3 - t_1}\right)\cdot  \indirectsenderU(\prediction_1) + 
    \miscali(\bayesprediction_1, \prediction_2)\cdot \frac{t_2-t_1}{t_3 - t_1}\cdot \indirectsenderU(\prediction_3) \\
    & - \left(
    \miscali(\bayesprediction_1, \prediction_1)\cdot \indirectsenderU(\prediction_1) + \miscali(\bayesprediction_1, \prediction_2)\cdot \indirectsenderU(\prediction_2)
    \right) \\
    = ~ & 
    \miscali(\bayesprediction, \prediction_2)
    \left(\frac{t_3 - t_2}{t_3 - t_1} \cdot  \indirectsenderU(\prediction_1) + \frac{t_2-t_1}{t_3 - t_1}\cdot \indirectsenderU(\prediction_3)
    - \indirectsenderU(\prediction_2)\right)
    = 0~,
\end{align*}
where the last inequality is due to \eqref{eq:collinearity}.
Thus, solution $\newbayesmarginalpredictor, \newmiscali$ is also optimal.

Similarly, the analysis for the case can be carried over when there exist predictions $\prediction_3 <  \prediction_2 < \prediction_1\le\bayesprediction_1$ with $\miscali(\bayesprediction_1, \prediction_1) > 0, \miscali(\bayesprediction_1, \prediction_2) > 0, \miscali(\bayesprediction_2, \prediction_3)> 0$.

If we have $\prediction_1\primed < \bayesprediction_1 <\prediction_1\doubleprimed$ with $\miscali(\bayesprediction_1, \prediction_1\primed) > 0, \miscali(\bayesprediction_1, \prediction_1\doubleprimed)>0$, and suppose there exists $\bayesprediction_2, \bayesprediction_3$ with $\bayesprediction_2<\bayesprediction_1<\bayesprediction_3$ such that $\miscali(\bayesprediction_2, \prediction_2) > 0, \miscali(\bayesprediction_3, \prediction_3) > 0$.
Then we argue that there always exists $\newmiscali$ such that it can either reallocates the mass $\miscali(\bayesprediction_1, \prediction_1\primed)$ to the predictions $\prediction_1\doubleprimed, \prediction_3$ or reallocates the mass $\miscali(\bayesprediction_1, \prediction_1\doubleprimed)$ to the predictions $\prediction_1\primed, \prediction_2$, without changing the objective value.

To see this, we first note that from the Statement (iii) in \Cref{lem:event-inde decoupled reformuation}, we still have
\begin{align*}
    \frac{\indirectsenderU(\prediction_1\primed) - \indirectsenderU(\prediction_1\doubleprimed)}{t_1\primed - t_1\doubleprimed} 
    = 
    \frac{\indirectsenderU(\prediction_3) - \indirectsenderU(\prediction_1\doubleprimed)}{t_3 - t_1\doubleprimed} 
    = \frac{\indirectsenderU(\prediction_2) - \indirectsenderU(\prediction_1\primed)}{t_2 - t_1\primed}~.
\end{align*}
where here $t_1\primed \triangleq \bayesprediction_1 - \prediction_1\primed$, $t_1\doubleprimed \triangleq \prediction_1\doubleprimed-\bayesprediction_1$, $t_3 = \prediction_3 - \bayesprediction_1$, and $t_2 = \bayesprediction_1 - \prediction_2$.

From \Cref{lem:monotone devia}, we can without loss consider $\prediction_2 \le\prediction_1\primed$ and $\prediction_1\doubleprimed\le \prediction_3$.
When $t_1\doubleprimed \le t_1\primed$, one can use the same argument above to construct a new feasible $\newmiscali$ such that it reallocates the mass $\miscali(\bayesprediction_1, \prediction_1\primed)$ to the predictions $\prediction_1\doubleprimed$ and $\prediction_3$ without changing the objective value.
Similarly, when $t_1\doubleprimed > t_1\primed$, one can use the same argument above to construct a new feasible $\newmiscali$ such that it reallocates the mass $\miscali(\bayesprediction_1, \prediction_1\doubleprimed)$ to the predictions $\prediction_1\primed$ and $\prediction_2$ without changing the objective value.

Putting all pieces together, we show that there exists at most $\bayesprediction_1, \bayesprediction_2 \in \supp(\bayesmarginalpredictor)$, such that each of them may generate two different predictions. 
By \Cref{lem:bipooling}, we further know that for each event $i$, we have $|\supp(\bayesprediction_i)| \le 2$. Thus, in total, we have at most $\numData+2$ predictions.
\end{proof}

We then first prove the third statement in \Cref{lem:num predictions decoupled formluation}.
\begin{proof}[Proof of Statement (iii) in \Cref{lem:num predictions decoupled formluation}]
We prove the result by construction.
Let $\bayesmarginalpredictor, \miscali$ be an optimal solution to the program \ref{eq:decoupled opt}, and suppose there exists three different predictions $\prediction_1, \prediction_2, \prediction_3$ that have the same true expected outcome, namely,
$\truePosterior(\prediction_1) = \truePosterior(\prediction_2) = 
\truePosterior(\prediction_3) = \bayesprediction$ for some $\bayesprediction\in\supp(\bayesmarginalpredictor)$.
We show that we can construct another feasible solution $\newbayesmarginalpredictor, \newmiscali$ to the program \ref{eq:decoupled opt} that has the same objective value with the solution  $\bayesmarginalpredictor, \miscali$.

Let $t_i = |\prediction_i -\bayesprediction|$ for $i\in[3]$.
Without loss of generality, let us assume that $t_1\le t_2\le t_3$. 
Note that it cannot be the case where $t_1 = t_3$, as this case must imply that either $\prediction_2 = \prediction_1$ or $\prediction_2 = \prediction_3$, which violates the assumption that $\prediction_1\neq\prediction_2\neq\prediction_3$.
We next consider the construction of $\newbayesmarginalpredictor, \newmiscali$ as follows:
\begin{alignat*}{3}
    \newmiscali(\bayesprediction', \cdot) & = \miscali(\bayesprediction', \cdot), \quad
    && \bayesprediction'\in \supp(\bayesmarginalpredictor)\setminus\{\bayesprediction\}~;\\
    \newmiscali(\bayesprediction, \prediction) & = \miscali(\bayesprediction, \prediction), \quad
    && \prediction\in [0, 1]\setminus\{\prediction_1, \prediction_3\} ~ ;\\
    \newmiscali(\bayesprediction, \prediction_1) 
    & = \miscali(\bayesprediction, \prediction_1) + \miscali(\bayesprediction, \prediction_2)\cdot \frac{t_3 - t_2}{t_3 - t_1}, \quad
    &&  \\
    \newmiscali(\bayesprediction, \prediction_3) 
    & = \miscali(\bayesprediction, \prediction_3) + \miscali(\bayesprediction, \prediction_2)\cdot \frac{t_2 - t_1}{t_3 - t_1}; &&
    \\
    \newmiscali(\bayesprediction, \prediction_2) 
    & = 0; &&
\end{alignat*}
The above construction essentially reallocates the mass $\miscali(\bayesprediction, \prediction_2)$ on prediction $\prediction_2$ to the predictions $\prediction_1$ and $\prediction_3$. 
Such reallocation ensures that $\ECE{\newmiscali} = \ECE{\miscali}$, which can be seen as follows:
\begin{align*}
    & \ECE{\newmiscali} - \ECE{\miscali} \\
    = ~ & 
    \left(\miscali(\bayesprediction, \prediction_1) + \miscali(\bayesprediction, \prediction_2)\cdot \frac{t_3 - t_2}{t_3 - t_1}\right)\cdot t_1 + 
    \left(\miscali(\bayesprediction, \prediction_3) + \miscali(\bayesprediction, \prediction_2)\cdot \frac{t_2 - t_1}{t_3 - t_1}\right)\cdot t_3 \\
    & - \left(
    \miscali(\bayesprediction, \prediction_1)\cdot t_1 + \miscali(\bayesprediction, \prediction_2)\cdot t_2
    + 
    \miscali(\bayesprediction, \prediction_3)\cdot t_3
    \right) \\
    = ~ & 
    \miscali(\bayesprediction, \prediction_2)\cdot \frac{t_3 - t_2}{t_3 - t_1} \cdot t_1 + 
    \miscali(\bayesprediction, \prediction_2)\cdot \frac{t_2 - t_1}{t_3 - t_1}\cdot t_3 - 
    \miscali(\bayesprediction, \prediction_2) \cdot t_2
    = 0~.
\end{align*}
Note that the above construction ensures that $\newbayesmarginalpredictor = \bayesmarginalpredictor$, and moreover, the objective value difference between the two solutions $\newbayesmarginalpredictor, \newmiscali$ and $\bayesmarginalpredictor, \miscali$ can be computed as
\begin{align*}
    & \left(\miscali(\bayesprediction, \prediction_1) + \miscali(\bayesprediction, \prediction_2)\cdot \frac{t_3 - t_2}{t_3 - t_1}\right)\cdot \indirectsenderU(\prediction_1) + 
    \left(\miscali(\bayesprediction, \prediction_3) + \miscali(\bayesprediction, \prediction_2)\cdot \frac{t_2 - t_1}{t_3 - t_1}\right)\cdot \indirectsenderU(\prediction_3) \\
    & - \left(
    \miscali(\bayesprediction, \prediction_1)\cdot \indirectsenderU(\prediction_1) + \miscali(\bayesprediction, \prediction_2)\cdot \indirectsenderU(\prediction_2)
    + 
    \miscali(\bayesprediction, \prediction_3)\cdot \indirectsenderU(\prediction_3)
    \right) \\
    = ~ & 
    \miscali(\bayesprediction, \prediction_2)\cdot \frac{t_3 - t_2}{t_3 - t_1} \cdot \indirectsenderU(\prediction_1) + 
    \miscali(\bayesprediction, \prediction_2)\cdot \frac{t_2 - t_1}{t_3 - t_1}\cdot \indirectsenderU(\prediction_3) - 
    \miscali(\bayesprediction, \prediction_2) \cdot \indirectsenderU(\prediction_2)
    = 0~,
\end{align*}
where the last inequality is due to the Statement (iii) in \Cref{lem:event-inde decoupled reformuation}, which ensures that for some $\caliBudgetDual \ge 0$, 
\begin{align*}
    \indirectsenderU(\prediction_1) - \caliBudgetDual\cdot t_1 
    = 
    \indirectsenderU(\prediction_2) - \caliBudgetDual\cdot t_2
    = \indirectsenderU(\prediction_3) - \caliBudgetDual\cdot t_3
\end{align*}
Thus, solution $\newbayesmarginalpredictor, \newmiscali$ is also optimal.
We thus finish the proof.
\end{proof}

We now prove the second statement in \Cref{lem:num predictions decoupled formluation}.

\begin{lemma}[Bi-pooling structure of optimal $\bayesmarginalpredictor$, adopted from \citealp{ABSY-23}]
\label{lem:bipooling}
There exists an optimal marginalized predictor $\bayesmarginalpredictor$ such that every event induces at most two true expected outcomes, namely, for every $i\in[\numData]$, we have $\{\bayesprediction: \bayesmarginalpredictor_i(\bayesprediction) > 0\} \le 2$.
\end{lemma}

\begin{proof}[Proof of Statement (ii) in \Cref{lem:num predictions decoupled formluation}]
Combining Statement (iii) in \Cref{prop:num predictions} and \Cref{lem:bipooling} would finish the proof.
\end{proof}

The proof \Cref{prop:opt verify} is based on a primal argument.
It utilizes the equivalence that we established in \Cref{prop:prog equivalence} and the properties of marginalized perfectly calibrated predictor (see \Cref{defn:mpc} and \Cref{lem:bayes predictor is in MPC}).
\begin{proof}[Proof of \Cref{prop:opt verify}]
Let us fix any predictor $\tilde{h}$, and consider the predictor $\predictor$ that is generated, according to \eqref{eq:construct f}, by the perfectly calibrated predictor $\bayespredictor$ and the miscalibration plan $\miscali$, which are constructed according to \eqref{eq:bayes predic construction} with input $\tilde{h}$.
By \Cref{prop:two-step equivalence}, we know that $\Payoff{\tilde{h}} = \Payoff{\predictor}$.
To show the optimality of $\tilde{h}$, it suffices to show the optimality of predictor $\predictor$.
By \Cref{cor:wlog ei space}, it suffices to show that for any $\caliBudget$-calibrated predictor $\newpredictor \in \predicSpaceEI_\caliBudget$, we have
\begin{align}
    \label{ineq:goal}
    \Payoff{\predictor} - \Payoff{\newpredictor}
    =
    \int_0^1 
    \sum\nolimits_{i\in[\numData]}\prior_i
    \predictor_i(\prediction)\indirectsenderU(\prediction) ~ \d \prediction -
    \int_0^1 
    \sum\nolimits_{i\in[\numData]}\prior_i
    \newpredictor_i(\prediction)
    \indirectsenderU(\prediction) ~ \d \prediction \ge 0
\end{align}
Since we know that $\convexEnv(\bayesprediction) \ge \indirectsenderU(\bayesprediction)$ for all $\bayesprediction\in[0, 1]$, and condition $\supp(\predictor) \subseteq \{\prediction\in[0, 1]: \convexEnv(\prediction) = \indirectsenderU(\prediction)\}$, we have
\begin{align*}
    0 = \int_0^1 \sum\nolimits_{i\in[\numData]}\prior_i\predictor_i(\prediction) \cdot \left(\indirectsenderU(\prediction) - \convexEnv(\prediction) \right) ~ \d \prediction 
    \ge 
    \int_0^1 \sum\nolimits_{i\in[\numData]}\prior_i\newpredictor_i(\prediction)\cdot \left(\indirectsenderU(\prediction) - \convexEnv(\prediction) \right) ~ \d \prediction~.
\end{align*}
Rearranging the above inequality, we then have
\begin{align}
    \text{LHS of } \eqref{ineq:goal}
    & \ge 
    \int_0^1 \sum\nolimits_{i\in[\numData]}\prior_i\predictor_i(\prediction)\convexEnv(\prediction) ~ \d \prediction -
    \int_0^1 \sum\nolimits_{i\in[\numData]}\prior_i\newpredictor_i(\prediction)\convexEnv(\prediction) ~ \d \prediction~.
    \label{ineq:goal 2}
\end{align}
Now it suffices to show that RHS of \eqref{ineq:goal 2} is always non-negative.
To see this, recall that by \Cref{prop:two-step equivalence}, for any $\predictor\in\predicSpaceEI_\caliBudget$, there must exist a pair of perfect calibrated predictor $\bayespredictor$ and a feasible miscalibration plan $\miscali$ such that 
\begin{align*}
    \sum\nolimits_{i\in[\numData]}\prior_i
    \predictor_i(\prediction)
     \overset{(a)}{=} 
     \sum\nolimits_{i\in[\numData]}\prior_i 
    \int_0^1 \bayespredictor_i(\bayesprediction)\cdot \frac{\miscali(\bayesprediction,\prediction)}{\sum_{j\in[n]} \prior_j\cdot \bayespredictor_j(\bayesprediction)}\,\d\bayesprediction 
    = \int_0^1 \bayesmarginalpredictor(\bayesprediction) \miscali(\prediction\mid \bayesprediction) ~\d  \bayesprediction~,
\end{align*}
where equality (a) is by Eqn.~\eqref{eq:construct f}, and 
$\miscali(\prediction \mid \bayesprediction)
\triangleq 
\frac{\miscali(\bayesprediction,\prediction)}{\sum_{j\in[n]} \prior_j\cdot \bayespredictor_j(\bayesprediction)}$ for $\bayesprediction, \prediction\in[0, 1]$, 
and $\bayesmarginalpredictor$ is the marginalized predictor from $\bayespredictor$.
Thus, we can write RHS of \eqref{ineq:goal 2} as follows:
\begin{align}
    \text{RHS of } \eqref{ineq:goal 2}
    & = 
    \int_0^1 \bayesmarginalpredictor(\bayesprediction) \int_0^1 \miscali(\prediction\mid \bayesprediction) \convexEnv(\prediction) ~ \d \prediction\d q -
    \int_0^1 \sum\nolimits_{i\in[\numData]}\prior_i\newpredictor_i(\prediction)\convexEnv(\prediction) ~ \d \prediction \nonumber \\
    & \overset{(a)}{=} 
    \int_0^1 \bayesmarginalpredictor(\bayesprediction) \int_0^1 \miscali(\prediction\mid \bayesprediction) \cdot\left(\convexEnv(\bayesprediction) + \caliBudgetDual \cdot |\prediction - q|\right) ~ \d \prediction\d \bayesprediction -
    \int_0^1 \sum\nolimits_{i\in[\numData]}\prior_i\newpredictor_i(\prediction)\convexEnv(\prediction) ~ \d \prediction~,
    \label{ineq:goal 3}
\end{align}
where $\caliBudgetDual\triangleq \max_{\prediction} |(\convexEnv(\prediction))'|$, and equality (a) holds by condition (ii-b) where for any $\miscali(\bayesprediction, \prediction) > 0$ (i.e., $\miscali(\prediction\mid \bayesprediction) > 0$) with $\prediction\neq \bayesprediction$, we know $\bayesprediction, \prediction$ must lie in linear tails of function $\convexEnv$, and thus we have $\convexEnv(\prediction) = \convexEnv(\bayesprediction) + \caliBudgetDual\cdot|\prediction-\bayesprediction|$.

When $\caliBudgetDual = 0$, we know that function $\convexEnv(\prediction) \equiv C$ for some constant $C \ge 0$ (here $C\ge 0$ is due to the fact that $\convexEnv(x) \ge \indirectsenderU(x) \ge 0$ for all $x\in[0, 1]$).
Thus, $\text{RHS of } \eqref{ineq:goal 3} = C - C = 0$.

When $\caliBudgetDual > 0$, then by condition (i), we must have 
\begin{align}
    \label{eq:decompose ece}
    \ECE{\predictor} = \int_0^1\bayesmarginalpredictor(\bayesprediction)\int_0^1\miscali(\prediction\mid \bayesprediction)\cdot|\prediction-\bayesprediction| ~\d \prediction\d \bayesprediction = \caliBudget~.
\end{align}
Thus,
\begin{align*}
    \text{RHS of } \eqref{ineq:goal 3} 
    & =  
    \int_0^1 \bayesmarginalpredictor(\bayesprediction) \convexEnv(\bayesprediction) ~\d  \bayesprediction + \caliBudgetDual\caliBudget -
    \int_0^1 \sum\nolimits_{i\in[\numData]}\prior_i\newpredictor_i(\prediction)\convexEnv(\prediction) ~ \d \prediction\\
    & \overset{(a)}{\ge}  
    \int_0^1 \bayesmarginalpredictor(\bayesprediction) \convexEnv(\bayesprediction) ~\d  \bayesprediction + \caliBudgetDual\cdot \ECE{\newmarginalpredictor} -
    \int_0^1 \newmarginalpredictor(\prediction)\convexEnv(\prediction) ~ \d \prediction\\
    & =
    \int_0^1 \bayesmarginalpredictor(\bayesprediction) \convexEnv(\bayesprediction) ~\d  \bayesprediction + \caliBudgetDual\cdot \int_0^1 \newbayesmarginalpredictor(\bayesprediction) \int_0^1 \newmiscali(\prediction\mid \bayesprediction) |\prediction-\bayesprediction|~\d \prediction\d q -
    \int_0^1 \newmarginalpredictor(\prediction)\convexEnv(\prediction) ~ \d \prediction
    \tag{holds similarly due to \eqref{eq:decompose ece}}\\
    & =
    \int_0^1 \bayesmarginalpredictor(\bayesprediction) \convexEnv(\bayesprediction) ~\d  \bayesprediction 
    - 
    \int_0^1 \newbayesmarginalpredictor(\bayesprediction) \int_0^1 \newmiscali(\prediction\mid \bayesprediction) \left(\convexEnv(\prediction) - \caliBudgetDual|\prediction-\bayesprediction|\right)~\d \prediction\d q \\
    & \ge
    \int_0^1 \bayesmarginalpredictor(\bayesprediction) \convexEnv(\bayesprediction) ~\d  \bayesprediction 
    - 
    \int_0^1 \newbayesmarginalpredictor(\bayesprediction) \int_0^1 \newmiscali(\prediction\mid \bayesprediction) \convexEnv(\bayesprediction)~\d \prediction\d \bayesprediction 
    \tag{due to $\convexEnv(\prediction) - \caliBudgetDual \cdot |\prediction-\bayesprediction|\le \convexEnv(\bayesprediction),  \forall q, \prediction\in[0, 1]$}\\
    & = 
    \int_0^1 \bayesmarginalpredictor(\bayesprediction) \convexEnv(\bayesprediction) ~\d  \bayesprediction 
    - 
    \int_0^1 \newbayesmarginalpredictor(\bayesprediction) \convexEnv(\bayesprediction)~\d  \bayesprediction \\
    & =
    \int_0^1 \lambda(\bayesprediction) \convexEnv(\bayesprediction) ~\d  \bayesprediction 
    - 
    \int_0^1 \newbayesmarginalpredictor(\bayesprediction) \convexEnv(\bayesprediction)~\d  \bayesprediction \tag{by condition (iii-a)}\\
    & \ge 0~, \tag{by \Cref{defn:mpc} and convexity of $\convexEnv$}
\end{align*}
where inequality (a) holds as $\ECE{\newpredictor}\le\caliBudget$, and for the predictor $\newpredictor\in\predicSpaceEI_\caliBudget$, we also have 
$\sum\nolimits_{i\in[\numData]}\prior_i
\newpredictor_i(\prediction)
= \int_0^1 \newbayesmarginalpredictor(\bayesprediction) \newmiscali(\prediction\mid \bayesprediction) ~\d  \bayesprediction$.
\end{proof}

\section{FPTAS for General Setting}
\label{sec:fptas}

In this section, we study the time complexity of computing an optimal $(\caliBudget, \tnorm)$-calibrated predictor. As the main result of this section, we present an FPTAS (\Cref{algo:fptas}) that works for the general setting (e.g., the sender's indirect utility can be event-dependent) and for all $\tnorm$-norm ECEs.
\begin{theorem}[FPTAS for general setting]
\label{thm:fptas}
For every persuasive calibration instance, given any $\discrepara\in(0, 1)$, there exists a linear programming (see \ref{eq:decoupled opt general discre}) based algorithm (\Cref{algo:fptas}) that computes a $(\caliBudget,\tnorm)$-calibrated predictor $\predictor$ achieving a $(1- \discrepara)$-approximation to the optimal $(\caliBudget,\tnorm)$-calibrated predictor~$\optpredictor$, i.e., $\Payoff{\predictor} \ge (1 - \discrepara)\cdot \Payoff{\optpredictor}$. The running time of the algorithm is $\poly(\sfrac{1}{\discrepara}, \numData, \numAction)$, where $\numData = |\stateSpace|, \numAction = |\actionSpace|$ are the number of events and the number of agent's actions, respectively.
\end{theorem}
We remark that the running time of \Cref{algo:fptas} in the above theorem does not depend on the ECE budget $\caliBudget$ or the exponent parameter $\normexponent$ in the ECE metric. In \Cref{sec:polytime}, we also present a polynomial-time algorithm to compute the optimal predictor for $\ell_1$-norm ECE and $\ell_\infty$-norm ECE. Determining whether a polynomial-time algorithm for computing the optimal predictor exists for $\tnorm$-norm ECE with other exponent parameters $\normexponent \in (1, \infty)$, or proving a computational hardness result, is left as an interesting future direction.

In \Cref{subsec:fptas algo discription}, we introduce the FPTAS (\Cref{algo:fptas}) and provide an overview of its two main technical ingredients and their analysis. In \Cref{subsec:generalizeed two-step framework,subsec:discretization and rounding}, we present the full details of these two technical ingredients, respectively.
Finally, we prove \Cref{thm:fptas} in \Cref{subsec:fptas proof}.

\subsection{Algorithm Description and Analysis Overview}
\label{subsec:fptas algo discription}
In this section, we first sketch the main algorithmic ingredients of \Cref{algo:fptas}, along with the intuitions and challenges behind them. We then provide details of each ingredient and their corresponding analysis.

To obtain a time-efficient algorithm for computing an approximately optimal predictor, a natural approach is to formulate the principal's problem as a computationally tractable optimization program (e.g., a linear or convex program). However, due to the calibration error constraint defined on the predictor $\predictor$, this approach is not straightforward. In \Cref{apx:Failure}, we explain the failure of some natural optimization program formulations that directly use the predictor as the decision variable and optimize over the entire feasible predictor space $\predicSpace_{(\caliBudget,\tnorm)}$.

\begin{algorithm}[ht]
    \caption{The FPTAS for the persuasive calibration with general $\normexponent \ge 1$}
    \begin{algorithmic}[1]
    \State \textbf{Input:} persuasive calibration instance $(\stateSpace, \prior, \senderU, \receiverU, \actionSpace, (\caliBudget, \tnorm))$, approximation precision $\discrepara \in (0, 1)$. 
    \State 
    Solve 
    {\text{\hyperref[eq:decoupled opt general discre]{$\textsc{$\textsc{LP-DiscTwoStep}^+_{(\sfrac{\discrepara}{3})}$}$}}}
    and let $\{\miscali_{i, j}^*\}_{i, j:i\le j}$ be its optimal solution.
    \State 
    Use the optimal solution $\{\miscali_{i, j}^*\}_{i, j:i\le j}$ to
    construct the predictor $\predictor$ as follows: 
    \begin{align*}
        \forall i\in[\numData],~\prediction\in[0,1]:
        \quad
        \predictor_i(\prediction) \gets 
        \displaystyle\frac{1}{\prior_i}\cdot\int_0^1\left(\sum\limits_{k \in [i:\numData]} \miscali^*_{i, k}(\bayesprediction, \prediction) 
        \cdot 
        \frac{\truemean_k - \bayesprediction}{\truemean_k - \truemean_i}
         + 
        \sum\limits_{k \in [i-1]} \miscali^*_{k, i}(\bayesprediction, \prediction) \cdot 
        \frac{\bayesprediction-\truemean_k}{\truemean_i - \truemean_k}
        \right)
        \, \d \bayesprediction
    \end{align*}
    \State \textbf{Output:} predictor $\predictor$.
    \end{algorithmic}
    \label{algo:fptas}
\end{algorithm}

We overcome the challenge and develop \Cref{algo:fptas}, which contains two main technical ingredients: it combines a generalized two-step framework, which extends the two-step approach in \Cref{subsec:characterization:two-step perspective} from the event-independent setting to the general setting, with a carefully designed discretization scheme.

\xhdr{Ingredient 1: generalized two-step framework}
For the general setting where the principal's utility may simultaneously depend on the realized event, the agent's action, and the realized outcome, the two-step framework with the \emph{event-independent} post-processing plan established in \Cref{subsec:characterization:two-step perspective} may not be sufficient. Motivated by this, we introduce a generalized two-step framework in \Cref{subsec:generalizeed two-step framework}. However, if we consider an event-dependent post-processing plan that \emph{fully decouples} the miscalibration across all events, additional non-linear constraints---seemingly intractable---are required. To bypass this issue, we introduce a generalized two-step framework with an \emph{(event-dependent) bi-event post-processing plan} (\Cref{defn:bed miscal plan,defn:bed miscal procedure}), which allows us to formulate an infinite-dimensional (but tractable) linear program~\ref{eq:decoupled opt general}. Our bi-event post-processing plan is rich enough to ensure that there always exists an optimal predictor that can be generated within our generalized framework (\Cref{prop:LP general}).

\xhdr{Ingredient 2: Instance-dependent two-layer discretization \& rounding-based analysis} 
Our FPTAS (\Cref{algo:fptas}) solves a discretized version of \ref{eq:decoupled opt general}. Common discretization schemes, such as uniform discretization, are not directly applicable, as they may render the discretized LP infeasible or significantly degrade the objective value due to the ECE budget constraint. To overcome this challenge, in \Cref{subsec:discretization and rounding}, we introduce an instance-dependent discretization scheme (\Cref{defn:discre}) with a two-layer structure.
First, our scheme constructs a non-uniform $\Theta(\discrepara)$-net that includes the expected outcomes $\stateSpace=\{\truemean_1,\dots,\truemean_{\numData}\}$ and the discontinuity points $\discontinuityset$ of the principal's indirect utility function. Second, for each point in this $\Theta(\discrepara)$-net, it introduces a finer $\Theta(\discreparaBudget)$-net within a small neighborhood, where $\discreparaBudget$ depends not only on the discretization precision $\discrepara$ but also on the ECE budget $\caliBudget$ and the norm exponent parameter $\normexponent$. 
This two-layer structure of our discretization scheme ensures that the discretized LP~\ref{eq:decoupled opt general discre} remains feasible, closely approximates the original objective value, and maintains a polynomial size. Finally, we develop a rounding scheme (\Cref{algo:rounding}) to analyze both the feasibility and the objective value of \ref{eq:decoupled opt general discre}. We believe that both our discretization scheme and our rounding argument are of independent interest in the algorithmic information design literature.

\subsection{Generalized Two-step Framework: Bi-Event Post-Processing Plan}
\label{subsec:generalizeed two-step framework}
In this section, we generalize the two-step framework developed in \Cref{subsec:characterization:two-step perspective} from the event-independent setting to the general setting.

In the event-independent setting, where the principal has an event-independent indirect utility, \Cref{prop:two-step equivalence} shows that it suffices to consider predictors from $\predicSpaceEI_{(\caliBudget,\tnorm)}$, which can be implemented via an event-independent post-processing plan (\Cref{defn:miscal plan}).
However, in the general setting where the principal's utility depends on the realized event, 
this sufficiency is no longer guaranteed.

Perhaps the most natural generalization is to allow event-dependent post-processing, by extending $\miscali(\bayesprediction,\prediction)$ to $\miscali_i(\bayesprediction, \prediction)$, which depends on the realized event $i$. 
However, \emph{fully decoupling} the post-processing plan across events introduces a significant challenge: it requires us to impose a \emph{non-linear} feasibility constraint that ensures
$\truePosterior(\prediction) = \bayesprediction$ after post-processing according to $\{\miscali_i(\bayesprediction,\prediction)\}_{i\in[n]}$. See \Cref{apx:Failure} for further discussion.

To address this challenge, rather than fully decoupling the post-processing plans, we propose a more tractable alternative:
an \emph{(event-dependent) bi-event post-processing plan}, which captures pairwise miscalibration between events.
This formulation offers greater flexibility than the event-independent approach, while avoiding the complexity of full decoupling. 
The formal definition is as follows.

\begin{definition}[(Event-dependent) Bi-event post-processing plan]
\label{defn:bed miscal plan}
An \emph{(event-dependent) bi-event post-processing plan} $\miscali$ is a joint distribution over $[\numData]\times[\numData]\times [0, 1]\times[0, 1]$. We denote by $\miscali_{i,j}(\bayesprediction, \prediction)$ the probability density (or probability mass for discrete distribution) of $(\bayesprediction, \prediction)\in[0, 1]^2$ under a pair of events $(i, j)$ with $i\leq j$. 

Fix any $\normexponent\geq 1,\caliBudget\geq 0$. A {\bed} post-processing plan $\miscali$ is \emph{feasible} under $\tnorm$-norm {\ECEbudget} $\caliBudget$ if it satisfies 
\begin{align}
    \label{eq:bed feasible miscal cal budget}
    \tag{\textsc{{BiE}-Calibration-Feasibility}}
    & 
    \left(\int_0^1\int_0^1 \sum\nolimits_{(i, j):i\le j} \miscali_{i, j}(\bayesprediction, \prediction) \cdot |\bayesprediction - \prediction|^\normexponent\, \d\prediction\d\bayesprediction\right)^{\frac{1}{\normexponent}} \le\caliBudget,
    \\
    \intertext{and for every event $i\in[\numData]$,}
    \label{eq:bed feasible miscal supp}
    \tag{\textsc{BiE-Supply-Feasibility}}
    & 
    \int_0^1\int_0^1 \left(
    \sum\nolimits_{k\in[i: \numData]} \miscali_{i, k}(\bayesprediction, \prediction)\cdot 
    \mixratio_{i, k}(\bayesprediction)
    + \sum\nolimits_{k\in[i-1]} \miscali_{k,i}(\bayesprediction, \prediction)
    \cdot 
    (1-\mixratio_{k, i}(\bayesprediction))
    \right)
    \, \d\prediction\d\bayesprediction = \prior_i
\end{align}
where auxiliary {\contributionratio} $\mixratio_{i, j}(\bayesprediction)
\triangleq \frac{\truemean_j -\bayesprediction}{\truemean_j - \truemean_i}$ for $\truemean_i< \truemean_j$, and $\mixratio_{i, j}(\bayesprediction) \triangleq 1$ for $\truemean_i = \truemean_j$.
When $(\normexponent,\caliBudget)$ are clear from the context, we simply refer to $\miscali$ as a feasible bi-event post-processing plan.
\end{definition}
The two feasibility constraints for the {\bed} post-processing plan in \Cref{defn:bed miscal plan} can be viewed as a generalization of the analogous constraints for the event-independent post-processing plan in \Cref{defn:miscal plan}. Additionally, it bypasses the challenge that arises when fully decoupling the post-processing plan across events (see \Cref{apx:Failure}) by introducing auxiliary {\contributionratios} $\{\mixratio_{i,j}(\bayesprediction)\}$. 
The intuition is as follows: consider any pair of events $(i, j)$ and prediction $\bayesprediction$. Events $i$ and $j$ are combined to form a true expected outcome $\bayesprediction$ \emph{if and only if} the probability mass contributed by event $i$ is proportional to $\mixratio_{i, j}(\bayesprediction)$, while the contribution from event $j$ is proportional to $1 - \mixratio_{i, j}(\bayesprediction)$. We formalize our intuition in \Cref{defn:bed miscal procedure} and \Cref{prop:bed miscal procedure error guarantee} below. We also illustrate it using \Cref{example:strict inclusion revisit}.\footnote{Unlike our (event-independent) post-processing plan, in \Cref{defn:bed miscal plan}, we do not explicitly consider the perfectly calibrated predictor, as our current goal is to develop an LP for computation rather than characterizing its structure. See also \Cref{footnote:necessity of g in two-step framework}.}

\begin{definition}[Bi-event post-processing procedure]
\label{defn:bed miscal procedure}
Fix any $\normexponent \geq 1, \caliBudget\geq 0$. Given any feasible {\bed} post-processing plan $\miscali$, a new (possibly imperfectly calibrated) predictor $\predictor$ is \emph{generated} as follows: for every {\event} $i\in[\numData]$, and prediction $\prediction\in[0, 1]$,
\begin{align}
    \label{eq:construct f general}
        \predictor_i(\prediction) \gets 
        \displaystyle\frac{1}{\prior_i}\cdot\int_0^1\left(\sum\nolimits_{k \in [i:\numData]} \miscali_{i, k}(\bayesprediction, \prediction) 
        \cdot 
        \mixratio_{i,k}(\bayesprediction)
         + 
        \sum\nolimits_{k \in [i-1]} \miscali_{k, i}(\bayesprediction, \prediction) \cdot 
        (1-
        \mixratio_{k,i}(\bayesprediction))
        \right)
        \, \d \bayesprediction\end{align}
To simplify the presentation, we also say predictor $\predictor$ is \emph{generated by} {\bed} post-processing plan $\miscali$.
\end{definition}

\begin{example}[Predictors generated from {\bed} post-processing procedure]
\label{example:strict inclusion revisit}
    Consider the same two-event setting and predictor $\predictor$ as in \Cref{example:structural opt:strict inclusion}. Predictor $\predictor$ can be generated by the following {\bed} post-processing plan $\miscali$: $\miscali_{1,1}(0.5,0.125) = 0.5$ and $\miscali_{2,2}(0.5, 0.875) = 0.5$. It can be verified that for every $\normexponent \geq 1$ and $\caliBudget \geq \ECEnorm{\predictor}$, {\bed} post-processing plan $\miscali$ is feasible under $\tnorm$-norm ECE budget $\caliBudget$.
\end{example}

\begin{proposition}
\label{prop:bed miscal procedure error guarantee}
Fix any $\normexponent\geq 1,\caliBudget\geq 0$, and any feasible {\bed} post-processing plan $\miscali$. The predictor $\predictor$ generated by $\miscali$ (in \Cref{defn:bed miscal procedure}) is well-defined and $(\caliBudget,\tnorm)$-calibrated, i.e., $\ECEnorm{\predictor} \leq \caliBudget$.
\end{proposition}
\begin{proof}[Proof of \Cref{prop:bed miscal procedure error guarantee}]
    Fix any feasible {\bed} post-processing plan $\miscali$. Consider predictor $\predictor$ generated by $\miscali$  (according to Eqn.~\eqref{eq:construct f general} in \Cref{defn:bed miscal procedure}). Note that predictor $\prediction$ is well-defined since $\int_0^1 \predictor_i(\prediction) \d\prediction = 1$ for every event $i\in[\numData]$, which is implied by \ref{eq:bed feasible miscal supp} (in \Cref{defn:bed miscal plan}) of {\bed} post-processing plan $\miscali$. 

    We next bound the $\tnorm$-ECE of predictor $\predictor$. Note that 
    \begin{align*}
        \predictor(\prediction) = \sum\nolimits_{i\in[\numData]} \prior_i\predictor_i(\prediction)
        = \int_0^1 \sum\nolimits_{(i, j):i\le j} \miscali_{i, j}(\bayesprediction, \prediction) \, \d\bayesprediction~.
    \end{align*}
    We next argue that for any prediction $\prediction\in[0, 1]$, if $\predictor(\prediction) > 0$, we have
    \begin{align}
        \label{ineq:cali error per p}
        \left|\prediction - \truePosterior(\prediction)\right|^{\normexponent}
        \le 
        \int_0^1 \frac{1}{\predictor(\prediction)} \sum\nolimits_{(i, j):i\le j} \miscali_{i, j}(\bayesprediction, \prediction) \cdot |\prediction - \bayesprediction|^\normexponent\, \d\bayesprediction
    \end{align}
    To see the above inequality, we notice that the true expected outcome $\truePosterior(\prediction)$ of prediction $\prediction$ under predictor $\predictor$ is 
    \begin{align*}
        \truePosterior(\prediction)
        & = 
        \frac{1}{\predictor(\prediction)}\sum\nolimits_{i\in[\numData]} \truemean_i\cdot \left(
        \int_0^1 \sum\nolimits_{k \in [i:\numData]} \miscali_{i, k}(\bayesprediction, \prediction) \mixratio_{i,k}(\bayesprediction) \, \d \bayesprediction + 
        \int_0^1 \sum\nolimits_{k \in [i-1]} \miscali_{k, i}(\bayesprediction, \prediction) (1-\mixratio_{k,i}(\bayesprediction)) \, \d \bayesprediction
        \right) \\
        & =
        \frac{1}{\predictor(\prediction)}\sum\nolimits_{(i,j):i\leq j}\int_0^1 
        \miscali_{i,j}(\bayesprediction, \prediction)\cdot (
        \mixratio_{i,j}(\bayesprediction)\cdot \truemean_i
        +
        (1-\mixratio_{i,j}(\bayesprediction))\cdot \truemean_j
        )
        \, \d \bayesprediction
        = 
        \frac{1}{\predictor(\prediction)}\sum\nolimits_{(i,j):i\leq j}\int_0^1 
        \miscali_{i,j}(\bayesprediction, \prediction)\cdot \bayesprediction
        \, \d \bayesprediction
\end{align*}
Thus, as function $x \mapsto|\prediction-x|^\normexponent$ is convex for $\normexponent \ge 1$, the inequality \eqref{ineq:cali error per p} holds by Jensen's inequality. 
Consequently, we have
\begin{align*}
    \left(\ECEnorm{\predictor}\right)^\normexponent 
    = 
    \int_0^1 \sum\nolimits_{i\in[\numData]} \prior_i\predictor_i(\prediction) \cdot \left|\prediction - \truePosterior(\prediction)\right|^{\normexponent}\,\d\prediction
    \le 
    \int_0^1 \int_0^1 \sum\nolimits_{(i, j):i\le j} \miscali_{i, j}(\bayesprediction, \prediction) \cdot |\prediction - \bayesprediction|^\normexponent\, \d\bayesprediction \d\prediction \le \caliBudget^\normexponent
\end{align*}
where the last inequality holds due to \ref{eq:bed feasible miscal cal budget} (in \Cref{defn:bed miscal plan}) of {\bed} post-processing plan $\miscali$.
This finishes the proof of \Cref{prop:bed miscal procedure error guarantee}.
\end{proof}

So far, we have extended the two-step framework from the event-independent setting, which uses the event-independent post-processing plan, to the general setting, which uses the {\bed} post-processing plan. While the event-independent post-processing plan may not be sufficient to generate the optimal predictor in the general setting, the key question now is whether the more general {\bed} post-processing plan is sufficient. The following proposition answers this question affirmatively---considering the {\bed} post-processing plan is indeed sufficient to generate the optimal predictor.

\begin{proposition}
    \label{prop:LP general}
    For every persuasive calibration instance,  the principal's expected utility under the optimal $(\caliBudget,\tnorm)$-calibrated predictor $\optpredictor$ is equal to the optimal objective value of the following linear program with variables $\{\miscali_{i,j}(\bayesprediction,\prediction)\}_{i,j\in[n]:i\leq j,\bayesprediction,\prediction\in[0,1]}$:
    \begin{align}
    \label{eq:decoupled opt general}
    \arraycolsep=5.4pt\def\arraystretch{1}
        \tag{$\textsc{LP-TwoStep}^+$}
        &\begin{array}{rlll}
        \max
    \limits_{\miscali\ge 0} 
    ~ &
    \displaystyle 
    \sum\limits_{(i, j): i\le j}\int_0^1\int_0^1
    \miscali_{i, j}(\bayesprediction, \prediction) \cdot \indirectsenderU_{i, j}(\bayesprediction, \prediction)~\d  \prediction \d  \bayesprediction
    \quad & \text{s.t.} &
    \vspace{1mm}
    \\
    (\ref{eq:bed feasible miscal cal budget})
    &
    \displaystyle
    \left(\int_0^1\int_0^1 \sum\nolimits_{(i, j):i\le j} \miscali_{i, j}(\bayesprediction, \prediction) \cdot |\bayesprediction - \prediction|^\normexponent\, \d\prediction\d\bayesprediction\right)^{\frac{1}{\normexponent}} \le\caliBudget,
    & 
    \vspace{1mm}
    \\
    (\ref{eq:bed feasible miscal supp}) 
    & &
    \\
    \multicolumn{2}{c}{
    \qquad
    \displaystyle
    \int_0^1\int_0^1 \left(
    \sum\limits_{k\in[i: \numData]} \miscali_{i, k}(\bayesprediction, \prediction)\cdot 
    \mixratio_{i, k}(\bayesprediction)
    + \sum\limits_{k\in[i-1]} \miscali_{k,i}(\bayesprediction, \prediction)
    \cdot 
    (1-\mixratio_{k, i}(\bayesprediction))
    \right)
    \, \d\prediction\d\bayesprediction = \prior_i}
    &
    i\in[\numData]
    \end{array}
    \end{align}  
    where, with a slight abuse of notation, we expand the definition of indirect utility function as: for every pair of events $(i, j)$ with $i \leq j$ and predictions $\bayesprediction,\prediction\in[0, 1]$,
    \begin{align*}
        \indirectsenderU_{i, j}(\bayesprediction, \prediction)
        & \triangleq 
        \mixratio_{i, j}(\bayesprediction)\cdot \indirectsenderU_i(\prediction)
        + 
        (1 - \mixratio_{i, j}(\bayesprediction))\cdot \indirectsenderU_j(\prediction)
    \end{align*}
    Moreover, every optimal solution of \ref{eq:decoupled opt general} can be converted into an optimal $\caliBudget$-calibrated predictor $\optpredictor$ using Eqn.~\eqref{eq:construct f general}.
\end{proposition}
We remark that \ref{eq:decoupled opt general} is an infinite-dimensional linear program. In the next subsection, we introduce an instance-dependent discretization scheme for this program, providing a provable guarantee on the discretization error. We conclude this section by proving \Cref{prop:LP general}.
\begin{proof}[Proof of \Cref{prop:LP general}]
We first prove the direction that
given any feasible solution $ \miscali$ to \ref{eq:decoupled opt general}, 
there exists an $\caliBudget$-calibrated predictor $\predictor$ such that the induced expected utility of the principal exactly equals to the objective value of the solution $\bayesmarginalpredictor, \miscali$ in \ref{eq:decoupled opt general}.

Due to the constraints in \ref{eq:decoupled opt general}, the feasible solution $\miscali$ is also a feasible {\bed} post-processing plan. Invoking \Cref{prop:bed miscal procedure error guarantee}, the predictor $\predictor$ generated by {\bed} post-processing plan $\miscali$ (according to Eqn.~\eqref{eq:construct f general}) is well-defined ans $(\caliBudget,\tnorm)$-calibrated. It suffices to compute the expected utility of the principal from predictor $\predictor$:
\begin{align*}
     \Payoff{\predictor} 
    = ~ & \int_0^1\sum\limits_{i\in[\numData]}\prior_i\predictor_i(\prediction)\indirectsenderU_i(\prediction) \,\d\prediction  \\
    = ~ & 
    \int_0^1\int_0^1 
    \sum\limits_{i\in[\numData]} \left(\sum\limits_{k \in [i:\numData]} \miscali_{i, k}(\bayesprediction, \prediction) 
    \cdot 
    \mixratio_{i, k}(\bayesprediction)
     + 
    \sum\limits_{k \in [i-1]} \miscali_{k, i}(\bayesprediction, \prediction) \cdot 
    \left(1 - \mixratio_{k, i}(\bayesprediction)\right)
    \right)
    \cdot \indirectsenderU_i(\prediction) \,\d\prediction\d \bayesprediction \\
    = ~ & 
    \sum\limits_{(i, j): i\le j}\int_0^1\int_0^1
    \miscali_{i, j}(\bayesprediction, \prediction) \cdot \indirectsenderU_{i, j}(\bayesprediction, \prediction)~\d  \prediction \d  \bayesprediction
\end{align*}
where the last term is exactly the objective value of solution $\miscali$ in \ref{eq:decoupled opt general}.
We thus complete the first direction.

We next prove the reverse direction: 
for any $\caliBudget$-calibrated predictor $\predictor$, there exists $\miscali$ that is a feasible solution to \ref{eq:decoupled opt general}, and its objective value in \ref{eq:decoupled opt general} exactly equals to the payoff (i.e., the expected utility of the principal) $\Payoff{\predictor}$ under predictor $\predictor$.
For this direction, our analysis uses the following technical lemma developed in \citet{FTX-22}:

\begin{lemma}[\citealp{FTX-22}]
\label{lem:wlog bayespredictor}
Let $X$ be a random variable 
with discrete support 
$\supp(X)$.
There exists a 
positive integer $K$,
a finite set of $K$
random variables $\{X_k\}_{k\in[K]}$,
and convex combination coefficients
$\convexcombinbf\in [0, 1]^K$ 
with $\sum_{k\in[K]} \convexcombin_k = 1$ such that:
\begin{enumerate}
    \item[(i)] \underline{Bayesian-plausibility}: for each $k\in [K]$, $\expect{X_k} = \expect{X}$;
    \item[(ii)] \underline{Binary-support}:
    for each $k\in[K]$, the size of $X_k$'s support is at most 2,
    i.e., $|\supp(X_k)| \leq 2$
    \item[(iii)] \underline{Consistency}:
    for each $x\in \supp(X)$, 
$\prob{X=x} = \sum_{k\in[K]}
\convexcombin_k\cdot 
\prob{X_k = x}$
\end{enumerate}
\end{lemma}

Given any $(\caliBudget, \tnorm)$-calibrated predictor $\predictor$, we first construct a (hypothetical) predictor $\predictor\primed$ as follows:
Initialize $\supp(\predictor\primed)\gets \emptyset$.
For each $\prediction\in\supp(\predictor)$, we consider a random variable $X$ with the following distribution
\begin{align*}
    \prob{X = \truemean_i} = \frac{\prior_i \predictor_i(\prediction)}{\sum\nolimits_{i\in[\numData]} \prior_i \predictor_i(\prediction)}~, \quad i\in[\numData]~.
\end{align*}
Let integer $K$, random variables $\{X_k\}_{k\in[K]}$,
and convex combination coefficients
$\convexcombinbf\in [0, 1]^K$  be the elements in \Cref{lem:wlog bayespredictor} for the aforementioned random variable $X$. 
We consider a new predictor $\predictor\primed$ defined as follows:
Add $K$ copies $\{\prediction^{(1)}, \dots, \prediction^{(K)}\}$ of prediction $\prediction$ into the support space $\supp(\predictor\primed)$,
i.e., $\supp(\predictor\primed) \gets 
\supp(\predictor\primed) \cup \{\prediction^{(1)}, \dots, \prediction^{(K)}\}$.
For each $k\in[K]$, set $\predictor\primed_i(\prediction^{(k)})
\gets 
\frac{1}{\prior_i}
{\convexcombin_k\cdot \prob{X_k = \truemean_i}\cdot 
{(\sum_{j\in[\numData]} \prior_j\predictor_j(\prediction))}}
$.
Note that this construction ensures that 
\begin{align*}
    \sum_{k\in[K]}
    \predictor\primed_i(\prediction^{(k)})
    &=
    \sum_{k\in[K]}
    \frac{1}{\prior_i}
    {\convexcombin_k\cdot \prob{X_k = \truemean_i}\cdot 
    {\left(\sum_{j\in[\numData]} \prior_j\predictor_j(\prediction)\right)}}
    \overset{(a)}{=}
    \frac{1}{\prior_i}
    \prob{X = \truemean_i}\cdot 
    {\left(\sum_{j\in[\numData]} \prior_j\predictor_j(\prediction)\right)}
    = \predictor_i(\prediction)
\end{align*}
where equality~(a) holds due to the ``Consistency'' property 
in \Cref{lem:wlog bayespredictor}.
Hence, the constructed predictor $\predictor\primed$
is a valid distribution.
Moreover, it can be verified that $\ECEnorm{\predictor\primed} = \ECEnorm{\predictor}$ as
\begin{align*}
    \left(\ECEnorm{\predictor\primed}\right)^\normexponent 
    & =  
    \int_{p\in\supp(\predictor)} \sum\nolimits_{k\in[K]} \sum\nolimits_{i\in[\numData]} \prior_i\predictor_i(\prediction^{(k)}) \cdot \left|\prediction^{(k)} - \truePosterior(\prediction^{(k)})\right|^{\normexponent}\,\d\prediction \\
    & \overset{(a)}{=} 
    \int_{p\in\supp(\predictor)} \sum\nolimits_{i\in[\numData]} \prior_i\predictor_i(\prediction) \cdot \left|\prediction - \truePosterior(\prediction)\right|^{\normexponent}\,\d\prediction = \left(\ECEnorm{\predictor}\right)^\normexponent~,
\end{align*}
where equality~(a) holds due to the ``Bayesian-plausibility'' property in \Cref{lem:wlog bayespredictor}.
Finally, we also have $\Payoff{\predictor\primed} = \Payoff{\predictor}$ as 
\begin{align*}
    \Payoff{\predictor\primed} = 
    \int_{p\in\supp(\predictor)} \sum\nolimits_{k\in[K]} \sum\nolimits_{i\in[\numData]} \prior_i\predictor_i(\prediction^{(k)}) \cdot \indirectsenderU_i(\prediction^{(k)})\,\d\prediction 
    =\Payoff{\predictor}~,
\end{align*}
where we again use the ``Bayesian-plausibility'' property in \Cref{lem:wlog bayespredictor}.

By construction, the predictor $\predictor\primed$ has the following property: $|\{i: \predictor\primed_i(\prediction^{(k)}) > 0\}|\le 2$
due to the ``Binary-support'' property in \Cref{lem:wlog bayespredictor}.
Now based on the constructed predictor $\predictor\primed$, we construct solution $\miscali$ as follows:
for all $\prediction^{(k)}\in\supp(\predictor\primed)$,
\begin{align*}
    \miscali_{i, j}(\bayesprediction, \prediction^{(k)}) \triangleq 
    \prior_i \predictor\primed_i(\prediction^{(k)}) 
    + 
    \prior_j \predictor\primed_j(\prediction^{(k)})~,
    \quad
    \text{if } \{i: \predictor\primed_i(\prediction^{(k)}) > 0\} = \{i, j\}
    \text{ and } 
    \truePosterior(\prediction^{(k)}) =\bayesprediction
\end{align*}
Moreover, for each $\prediction^{(k)}$, we have 
\begin{equation}
    \begin{aligned}
    \label{eq:property miscali}
    \miscali_{i, j}(\bayesprediction, \prediction^{(k)}) \cdot \mixratio_{i, j}(\bayesprediction)
    & = 
    \left(\prior_i \predictor\primed_i(\prediction^{(k)}) 
    + 
    \prior_j \predictor\primed_j(\prediction^{(k)})\right)
    \cdot \mixratio_{i, j}(\bayesprediction)\\
    & = 
    \left(\prior_i \predictor\primed_i(\prediction^{(k)}) 
    + 
    \prior_i \predictor\primed_i(\prediction^{(k)}) 
    \cdot \frac{1 - \mixratio_{i, j}(\bayesprediction)}{\mixratio_{i, j}(\bayesprediction)}\right)
    \cdot \mixratio_{i, j}(\bayesprediction) = \prior_i \predictor\primed_i(\prediction^{(k)})\\
    \miscali_{i, j}(\bayesprediction, \prediction^{(k)}) \cdot (1-\mixratio_{i, j}(\bayesprediction)) & = 
    \left(\prior_j \predictor\primed_j(\prediction^{(k)})\frac{ \mixratio_{i, j}(\bayesprediction)}{1 - \mixratio_{i, j}(\bayesprediction)}
    + 
    \prior_j \predictor\primed_j(\prediction^{(k)}) \right)
    \cdot (1-\mixratio_{i, j}(\bayesprediction)) = \prior_j \predictor\primed_j(\prediction^{(k)})
    \end{aligned}
\end{equation}
Thus, with above observations, we have, for any $k\in[\numData]$,
\begin{align*}
    & \int_0^1\int_0^1 \left(
    \sum\nolimits_{j\in[k, \numData]} \miscali_{k, j}(\bayesprediction, \prediction)\cdot 
    \mixratio_{k, j}(\bayesprediction)
    + \sum\nolimits_{i\in[k-1]} \miscali_{i, k}(\bayesprediction, \prediction)
    \cdot 
    (1-\mixratio_{i, k}(\bayesprediction))
    \right)
    \, \d\bayesprediction \d\prediction \\
    = ~ & 
    \int_{p\in\supp(\predictor)}\sum_{k=1}^K
    \prior_k \predictor\primed_k(\prediction^{(k)})
    \,\d\prediction
    = \prior_k \int_{p\in\supp(\predictor)}\predictor_k(\prediction)
    \,\d\prediction = \prior_k~.
\end{align*}
Moreover, the first constraint is also satisfied as we have 
\begin{align*}
    & \int_0^1\int_0^1 \sum\nolimits_{(i, j):i\le j} \miscali_{i, j}(\bayesprediction, \prediction) \cdot |\bayesprediction - \prediction|^\normexponent\, \d\bayesprediction \d\prediction \\
    = ~ & 
    \int_{\prediction\in\supp(\predictor)}\int_0^1 \sum\nolimits_{i\in[\numData]}\prior_i\predictor_i(\prediction)\cdot |\bayesprediction - \prediction|^\normexponent\, \d\bayesprediction \d\prediction = \left(\ECEnorm{\predictor}\right)^\normexponent
    \le \caliBudget^\normexponent~.
\end{align*}
Finally, we compute the objective value of the constructed solution $\miscali$,
\begin{align*}
    \sum\nolimits_{(i, j): i\le j}\int_0^1\int_0^1
    \miscali_{i, j}(\bayesprediction, \prediction) \cdot \indirectsenderU_{i, j}(\bayesprediction, \prediction)~\d  \prediction \d  \bayesprediction
    \overset{(a)}{=} \Payoff{\predictor\primed} = \Payoff{\predictor}~,
\end{align*}
where equality (a) holds by definition of $\indirectsenderU_{i, j}(\bayesprediction, \prediction)$, and the property of $\miscali$ established in \eqref{eq:property miscali}.
We thus finish the second direction and conclude the proof of \Cref{prop:LP general}.
\end{proof}

\subsection{Instance-Dependent Discretization and Rounding Scheme}
\label{subsec:discretization and rounding}

In this section, we introduce the second key technical ingredient of \Cref{algo:fptas}: a discretization scheme for the infinite-dimensional linear program \ref{eq:decoupled opt general}. Our goal is to develop this scheme with a provable guarantee on the discretization error. To achieve this, an ideal discretization scheme and its induced discretized linear program (LP) should to satisfy the following two properties:
\begin{enumerate}
    \item The discretized LP has a feasible solution.
    \item Given an optimal solution to original LP \ref{eq:decoupled opt general}, we can identify a feasible solution in the discretized LP whose objective value is close and the ECE is weakly smaller.
\end{enumerate}
Both properties introduce challenges in applying common discretization schemes, such as uniform discretization, to our problem. (See \Cref{ex:uniform failure} where the discretized LP becomes infeasible under uniform discretization.) While similar challenges arising from the first property has been observed in the algorithmic information design literature \citep[e.g.,][]{AFT-23}, the second challenge is, to the best of our knowledge, unique due to the ECE budget constraint.
In particular, suppose the discretized LP under uniform discretization with precision $\discrepara$ happens to be feasible. A natural rounding approach would then convert a feasible solution into a discretized feasible solution by rounding each $\miscali_{i,j}(\bayesprediction,\prediction)$ to $\miscali_{i,j}(\bayesprediction',\prediction')$, where $(\bayesprediction',\prediction')$ is the closest discretized point to $(\bayesprediction,\prediction)$. However, after this naive rounding, the ECE may increase by $\Theta(\discrepara)$, which could be significantly larger than the ECE budget $\caliBudget$.
To ensure that the ECE increment remains controlled under this rounding procedure, we would need to set the discretization precision as $\discrepara \gets \Theta(\caliBudget)$, i.e., on the same order as the ECE budget $\caliBudget$. However, with such a discretization precision, the size of the discretized set and the corresponding discretized LP would have a polynomial dependence on $\caliBudget$, which prevents us from obtaining the FPTAS stated in \Cref{thm:fptas}.

\begin{example}[The failure of uniform discretization]
\label{ex:uniform failure}
Consider a three-event setting with $\{\truemean_i\}_{i\in[3]}$ and a zero ECE budget $\caliBudget = 0$, i.e., predictor is feasible if and only if it is perfectly calibrated.

Now, consider a uniform-grid-only discretization scheme $\discresupp = \{\discrepara, 2\discrepara, \ldots\} \cap [0, 1]$. Suppose the event distribution $\prior$ assigns negligible probabilities to the first and third events, and the true expected outcome for the second event, $\truemean_2$, is not in the the discretized set $\discresupp$ (i.e., $\truemean_2 \notin \discresupp$), then it can be verified with \Cref{lem:bayes predictor is in MPC} that it is impossible to construct any perfectly calibrated predictor whose support lies entirely on the grid points in $\discresupp$.  
Therefore, no feasible predictor exists if we restrict its support to discretized set $\discresupp$.
\end{example}
We overcome the two aforementioned challenges by introducing an instance-dependent two-layer discretization scheme (\Cref{defn:discre}). Loosely speaking, to ensure that the discretized LP remains feasible, our scheme first constructs a non-uniform $\Theta(\discrepara)$-net that includes discretized points corresponding to the expected outcomes $\stateSpace=\{\truemean_1,\dots,\truemean_{\numData}\}$ and the discontinuity points $\discontinuityset$ of the principal's indirect utility function. 
Additionally, for every point in this $\Theta(\discrepara)$-net, we introduce a finer ``multiplicative $\Theta(\discrepara\caliBudget^\normexponent)$-net'' within a tiny neighborhood of length $\Theta(\caliBudget^\normexponent)$ (see formal construction in \Cref{defn:discre}). By analyzing the optimal solution of the original program \ref{eq:decoupled opt general}, we argue that locally refining the grid within these small neighborhoods is sufficient to approximately preserve the optimal objective value while keeping ECE variations under control (\Cref{lem:discre error,lem:p in S}).
With this approach, the size of the discretized set remains polynomial in the problem input and the discretization parameter $\discrepara$ (\Cref{claim:disc support size}), enabling the construction of an FPTAS.

\begin{definition}[Instance-dependent two-layer discretization scheme]
\label{defn:discre}
Given any discretization parameter $\discrepara \in (0, \sfrac{1}{3})$, define \emph{discretized set $\discresupp$} of the space $[0, 1]$ for both true expected outcomes and the predictions as follows: 
\begin{align*}
    \discresupp \triangleq 
    \underbrace{\left(\{0, \discrepara, 2\discrepara, \ldots\} 
    \cap[0, 1]
    \right)
    \cup \stateSpace
    \cup 
    \discontinuityset
    }_{\text{``first global layer''}}
    \cup 
    \underbrace{
    \left\{z \pm (\discreparaBudget \cdot (1+\discrepara)^s)^{\sfrac{1}{\normexponent}}\right\}_{z\in\discontinuityset\cup\stateSpace,
    s\in [0:\numdeltas]}
    }_{\text{``second local layer''}}
\end{align*}
where $\discontinuityset$ is the set of all discontinuity points of the sender's indirect utility function $\{\indirectsenderU_i\}_{i\in[n]}$, 
$\discreparaBudget \triangleq \caliBudget^\normexponent\cdot \discrepara$,
and $\numdeltas\triangleq\lceil  \frac{2}{\ln(1+\discrepara)} \ln \frac{1}{\discrepara} \rceil$.
\end{definition}
Although our discretization scheme above depends on the ECE budget $\caliBudget$ and the norm exponent parameter $\normexponent$, the number of discretized points in the set only depends on the problem instance and discretization parameter $\discrepara$. 
\begin{claim}
\label{claim:disc support size}
Given any discretization parameter $\discrepara \in (0, \sfrac{1}{3})$,
the size of the discretized set $\discresupp$ defined in \Cref{defn:discre} satisfies
$|\discresupp| 
= O(\frac{1}{\discrepara} 
+\frac{\numAction+\numData}{\ln(1+\discrepara)} \ln \frac{1}{\discrepara} )$,
where $\numData = |\stateSpace|, \numAction = |\actionSpace|$ are the number of events and the number of agent's actions, respectively.
\end{claim}
\begin{proof}
    Since the agent has $\numAction$, the discontinuity points of the sender's indirect utility function $\{\indirectsenderU_i\}_{i\in[n]}$ is at most $\numAction + 1$. 
    Thus, the size of the discretized set $\discresupp$ in the statement holds by construction.
\end{proof}
Given the discretization scheme, we are ready to present the discretized LP \ref{eq:decoupled opt general discre} and its discretization error guarantee.
\begin{lemma}[Discretized LP and discretization error guarantee]
\label{lem:discre error}
Given any discretization parameter $\discrepara \in (0, \sfrac{1}{3})$, consider the following discretized version of \ref{eq:decoupled opt general}:
\begin{align}
    \label{eq:decoupled opt general discre}
    \arraycolsep=5.4pt\def\arraystretch{1}    
    \tag{$\textsc{$\textsc{LP-DiscTwoStep}^+_{\discrepara}$}$}
    &\begin{array}{rlll}
    \max
    \limits_{\miscali\ge 0} 
    ~ &
    \displaystyle 
    \sum\limits_{(i, j): i\le j}\sum\limits_{\prediction, \bayesprediction\in\discresupp}
    \miscali_{i, j}(\bayesprediction, \prediction) \cdot \indirectsenderU_{i, j}(\bayesprediction, \prediction)
    \quad & \text{s.t.} &
    \vspace{1mm}
    \\
    & 
    \displaystyle
    \sum\limits_{\prediction, \bayesprediction\in\discresupp} \sum\limits_{(i, j):i\le j} \miscali_{i, j}(\bayesprediction, \prediction) \cdot |\bayesprediction - \prediction|^\normexponent \le\caliBudget^\normexponent,
    & 
    \vspace{1mm}
    \\
    & 
    \displaystyle
    \sum\limits_{\prediction, \bayesprediction\in\discresupp} \left(
    \sum\limits_{k\in[i: \numData]} \miscali_{i, k}(\bayesprediction, \prediction)\cdot 
    \mixratio_{i, k}(\bayesprediction)
    + \sum\limits_{k\in[i-1]} \miscali_{k, i}(\bayesprediction, \prediction)
    \cdot 
    (1-\mixratio_{k, i}(\bayesprediction))
    \right)
    \, \d\prediction\d\bayesprediction = \prior_i,
    & i\in[\numData]
    \vspace{1mm}
    \end{array}
\end{align}
where discretized set $\discresupp$ is defined in \Cref{defn:discre}. Then, the following three properties hold: 
\begin{enumerate}
    \item [(i)] \ref{eq:decoupled opt general discre} is a feasible linear program with $\poly(\numData, \numAction, \sfrac{1}{\discrepara})$ size, where $\numData = |\stateSpace|, \numAction = |\actionSpace|$ are the number of events and the number of agent's actions, respectively.
    \item [(ii)] Every feasible solution of \ref{eq:decoupled opt general discre} is also a feasible solution of \ref{eq:decoupled opt general}.
    \item [(iii)] The optimal objective value of \ref{eq:decoupled opt general discre} is a $(1-3\discrepara)$-approximation to the optimal objective value of \ref{eq:decoupled opt general}, i.e., 
    \begin{align*}
    \OBJ{\text{\ref{eq:decoupled opt general discre}}} \ge (1-3\discrepara)\cdot \OBJ{\text{\ref{eq:decoupled opt general}}}~.
    \end{align*}
\end{enumerate} 
\end{lemma}
In \Cref{lem:discre error}, the second property follows directly from the construction of \ref{eq:decoupled opt general discre}.
In the remainder of this section, we develop a rounding argument to establish feasibility and bound the discretization error of \ref{eq:decoupled opt general discre}.

\xhdr{A rounding scheme for analyzing \ref{eq:decoupled opt general discre}}
To show \Cref{lem:discre error}, we argue that (1) an optimal solution to \ref{eq:decoupled opt general} can, via a rounding scheme (see \Cref{algo:rounding}), be converted to a feasible solution to \ref{eq:decoupled opt general discre}; (2) the converted solution has a close objective value to the original solution.

Before introducing the rounding scheme, we first observe that it is without loss to consider an optimal solution $\miscali$ to \ref{eq:decoupled opt general} such that any generated prediction lies in the set $\discresupp$.
\begin{lemma}
\label{lem:p in S}
In \ref{eq:decoupled opt general}, there exists an optimal solution $\miscali$ such that for any $\miscali_{i, j}(\bayesprediction, \prediction) > 0$, it satisfies that $\prediction\in \discontinuityset\cup\stateSpace$.
\end{lemma}
The intuition behind this result is that the principal's indirect utility function, $\indirectsenderU_i$, is piecewise constant. Consequently, if the true expected outcome $\bayesprediction$ and the prediction $\prediction$ fall into different constant segments, it is always possible---without loss---to modify the predictor so that it ``miscalibrates'' true expected outcome $\bayesprediction$ to a discontinuity point. Conversely, if they lie within the same constant segment, we can distribute true expected outcome $\bayesprediction$ across a pair of events or discontinuity points and then generate perfectly calibrated predictions for these points.
\begin{proof}[Proof of \Cref{lem:p in S}]
Recall that $\discontinuityset \triangleq \{z_k\}_{k\in[\numAction]}$ denotes the set of all discontinuity points in the principal's indirect utility function.
Fix any feasible solution $\miscali$ to \ref{eq:decoupled opt general},
for any pair of $(\bayesprediction, \prediction)$ with $\miscali_{i, j}(\bayesprediction, \prediction) > 0$ and $\prediction\notin \discontinuityset$, we consider two possible cases:

Suppose $\prediction\in(z_k, z_{k+1})$ and $\bayesprediction\in(z_l, z_{l+1})$ where $k<l$ (or\ $k>l$). Consider a new solution that,
    instead of miscalibrating $\bayesprediction$ to $\prediction$, it miscalibrates $\bayesprediction$ to $z_{k+1}$ (or $z_{k}$). 
    This adjustment weakly decreases the calibration error while weakly improving the expected utility of the principal, due to the upper semi-continuity of the principal's indirect utility function (see Footnote~\ref{fn:tie-breaking}).

    Suppose both $\prediction, \bayesprediction \in(z_k, z_{k+1})$.
    In this case, it is without loss to assume $\bayesprediction = \prediction$, since otherwise the solution incurs additional calibration error without any improvement in expected utility of the principal.
    Suppose further that $\bayesprediction\in (\truemean_i, \truemean_j)$.
    Define $\bayesprediction_L \triangleq \truemean_i\vee z_k$ and let $\bayesprediction_R \triangleq \truemean_j \wedge z_{k+1}$.
    Now consider a modified solution that, instead of contracting the true expected outcomes $\truemean_i, \truemean_j$ to $\bayesprediction$, it contracts them to $\bayesprediction_L, \bayesprediction_R$. 
    Notice that this is always feasible as $\truemean_i \le \bayesprediction_L \le \bayesprediction_R\le \truemean_j$.
    Now this new solution generates perfectly calibrated predictions from these true expected outcomes $\bayesprediction_L, \bayesprediction_R$.
    Since the calibration error does not increase and the expected utility of the principal weakly improves (again due to upper semi-continuity -- see Footnote~\ref{fn:tie-breaking}), the modified solution is at least as good as the original.

\smallskip
\noindent
Combing the two cases, we complete the proof of \Cref{lem:p in S}.
\end{proof}

By \Cref{lem:p in S}, we only need to round the true expected outcome $\bayesprediction$ so that it lies in the set $\discresupp$.
At a high-level, our rounding scheme proceeds in two steps:

\xhdr{Step 1 (Lines \ref{preprocessing step one}-\ref{preprocessing step two})} In this step, for each $\miscali_{i, j}(\bayesprediction, \prediction) > 0$, we use $\scaledownfrac$ fraction of it to generate perfectly calibrated predictions. By doing this, we save some calibration error so that we can have enough buffer for the potential increase of calibration error in the later rounding steps.  

\xhdr{Step 2 (Lines \ref{round q start}-\ref{round q end})}
    In this step, given a pair $(\bayesprediction, \prediction)$ with $\miscali_{i, j}(\bayesprediction, \prediction) > 0$, we round the true expected outcome $\bayesprediction$, using the remaining fraction $1-\scaledownfrac$ of $\miscali_{i, j}(\bayesprediction, \prediction)$, to two grid points $\bayesprediction_L, \bayesprediction_R\in\discresupp$.
    We then consider three possible cases, depending on the distance $|\bayesprediction-\prediction|$. 
    
    In \textbf{Case I and Case II} where the distance $|\bayesprediction-\prediction|$ is not too large, we round $\bayesprediction$ to $\bayesprediction_L, \bayesprediction_R$ and then miscalibrate these points to the prediction $\prediction$ as follows:
    \begin{align}
        \label{eq:rounding}
        \begin{split}
        \miscali\primed_{i, j}(\bayesprediction_L, \prediction) &\gets 
        \miscali\primed_{i, j}(\bayesprediction_L, \prediction) 
        +
        {(1-\scaledownfrac)}
        \miscali_{i, j}(\bayesprediction, \prediction) \frac{\bayesprediction_R-\bayesprediction}{\bayesprediction_R-\bayesprediction_L}
        \\
        \miscali\primed_{i, j}(\bayesprediction_R, \prediction) &\gets 
        \miscali\primed_{i, j}(\bayesprediction_R, \prediction)+
        {(1-\scaledownfrac)}
        \miscali_{i, j}(\bayesprediction, \prediction)\left(1 - \frac{\bayesprediction_R-\bayesprediction}{\bayesprediction_R-\bayesprediction_L}\right)
        \end{split}
    \end{align}
    where the ratio $\frac{\bayesprediction_R-\bayesprediction}{\bayesprediction_R-\bayesprediction_L}$ ensures that, after rounding, the constructed $\miscali\primed$ satisfies the second constraint in \ref{eq:decoupled opt general discre}.
    We carefully choose the grid points $\bayesprediction_L, \bayesprediction_R \in \discresupp$ so that, after rounding, the calibration error does not increase significantly. Note that for these two cases, the expected utility of the principal remains unchanged after rounding.

    In \textbf{Case III} where the distance $|\bayesprediction-\prediction|$ is very large, we use the remaining fraction $1-\scaledownfrac$ of $\miscali_{i, j}(\bayesprediction, \prediction)$ to generate perfectly calibrated predictions. 
    In this case, there is no increase in calibration error. Furthermore, by the condition defining this case, we can also argue that any loss of the expected utility of the principal is upper-bounded by a factor of $\discrepara(1+\discrepara)$.

\begin{algorithm}[ht]
    \caption{Rounding scheme of $\miscali$ for the discretization error analysis}
    \begin{algorithmic}[1]
    \State \textbf{Input:} a feasible solution $\miscali$ to \ref{eq:decoupled opt general}; 
    \State \textbf{Input:}
    $\discrepara > 0, \discresupp$ and $\discreparaBudget = \discrepara \cdot \caliBudget^\normexponent$, and let $\scaledownfrac \gets 1 - \frac{1}{1+2\discrepara}$
    \State \textbf{Output:} a feasible solution $\miscali\primed$ to \ref{eq:decoupled opt general discre}.
    \State \textbf{Initialize:} $\miscali\primed_{i, j}(\bayesprediction, \prediction) \gets 0$ for all $i, j\in[\numData]$ and all $\bayesprediction, \prediction\in\discresupp$.
    \For{every $(\bayesprediction, \prediction) \in [0, 1]^2$ with $\miscali_{i, j}(\bayesprediction, \prediction) > 0$}
    \hfill \Comment{Without loss consider $\prediction \le \bayesprediction$}
    \State     \Comment{Without loss consider $\prediction\in \discontinuityset\cup\stateSpace$.}
    \State \Comment{Below we use $\scaledownfrac$ fraction of $\miscali_{i, j}(\bayesprediction, \prediction)$ to induce perfectly calibrated predictions.}
    \State $\miscali\primed_{i, j}(\truemean_i, \truemean_i) 
    \gets 
    \miscali\primed_{i, j}(\truemean_i, \truemean_i) 
    +
    {\scaledownfrac}\cdot \miscali_{i, j}(\bayesprediction, \prediction)\cdot \mixratio_{i, j}(\bayesprediction)$.
    \label{preprocessing step one}
    \State $\miscali\primed_{i, j}(\truemean_j, \truemean_j) 
    \gets  
    \miscali\primed_{i, j}(\truemean_j, \truemean_j)
    +
    {\scaledownfrac}
    \cdot \miscali_{i, j}(\bayesprediction, \prediction)\cdot (1- \mixratio_{i, j}(\bayesprediction))$.
    \label{preprocessing step two}
    \State \Comment{Below we round $\bayesprediction$ to be in $\discresupp$.}
    \If{$\bayesprediction - \prediction < \discreparaBudget^{\sfrac{1}{\normexponent}}$} 
    \label{round q start}
    \hfill\Comment{Case I}
        \State 
        Set $\bayesprediction_L \gets \truemean_i \vee \prediction$, $\bayesprediction_R \gets (\prediction + \discreparaBudget^{\sfrac{1}{\normexponent}}) \wedge \truemean_j$.
        \hfill \Comment{When $i = j$, we must have $\bayesprediction_L = \truemean_i = \bayesprediction_R$}
        \State 
        Round $\bayesprediction$ to be $\bayesprediction_L, \bayesprediction_R$ according to Eqn.~\eqref{eq:rounding}.
        \hfill \Comment{We set $\frac{0}{0} = 1$ when $i=j$ happens}
    \ElsIf{$\discreparaBudget^{\sfrac{1}{\normexponent}} \le \bayesprediction - \prediction \le (\discreparaBudget\cdot (1+\discrepara)^{\numdeltas-1})^{\sfrac{1}{\normexponent}} $}
    \hfill\Comment{Case II}
        \State 
        Set $\bayesprediction_L \gets \truemean_i\vee \prediction$, 
        and set $\bayesprediction_R$ as defined in \Cref{claim:s must exists}.
        \State
        Round $\bayesprediction$ to be $\bayesprediction_L, \bayesprediction_R$ according to Eqn.~\eqref{eq:rounding}.
    \Else
    \hfill\Comment{Case III}
        \State 
        $\miscali\primed_{i, j}(\truemean_i, \truemean_i) \gets
        \miscali\primed_{i, j}(\truemean_i, \truemean_i)+
        {(1-\scaledownfrac)}\cdot
        \miscali_{i, j}(\bayesprediction, \prediction)\cdot \mixratio_{i, j}(\bayesprediction)$;
        \State 
        $\miscali\primed_{i, j}(\truemean_j, \truemean_j) \gets 
        \miscali\primed_{i, j}(\truemean_j, \truemean_j)
        +
        {(1-\scaledownfrac)} \cdot
        \miscali_{i, j}(\bayesprediction, \prediction)\cdot (1-\mixratio_{i, j}(\bayesprediction))$.
    \EndIf
    \label{round q end}
    \EndFor
    \end{algorithmic}
    \label{algo:rounding}
\end{algorithm}
With the rounding scheme (\Cref{algo:rounding}), we are now ready to prove \Cref{lem:discre error}.
\begin{proof}[Proof of \Cref{lem:discre error}]
The property (ii) follows directly from the construction of \ref{eq:decoupled opt general discre}. 
We next prove the statements (i) and (iii).

For the simplicity of the presentation, we define $\barECEnorm{\miscali} \triangleq \int_0^1\int_0^1\sum_{i, j: i\le j}\miscali_{i, j}(\bayesprediction, \prediction)\cdot |\bayesprediction-\prediction|^\normexponent\,\d\bayesprediction\d\prediction$ and $\barECEnorm{\miscali\mid (\bayesprediction, \prediction)} \triangleq \sum_{i, j: i\le j}\miscali_{i, j}(\bayesprediction, \prediction)\cdot |\bayesprediction-\prediction|^\normexponent$.
By definition, we have $\barECEnorm{\miscali\primed} = \sum_{\bayesprediction, \prediction} \barECEnorm{\miscali\mid (\bayesprediction, \prediction)}$.

Let us fix a optimal solution $\miscali$ to \ref{eq:decoupled opt general} that satisfying \Cref{lem:p in S}.
Below we show that the rounding scheme detailed in \Cref{algo:rounding} which takes the solution $\miscali$ as input can output a feasible solution $\miscali\primed$ to \ref{eq:decoupled opt general discre}, namely, we have $\barECEnorm{\miscali\primed} \le \caliBudget^\normexponent$, and its support are restricted to be in the set $\discresupp$; and meanwhile, the output $\miscali\primed$ has $(1-3\discrepara)$-approximate expected utility of the principal to the solution $\miscali$. 

We analyze the incurred calibration error of the constructed $\miscali\primed$ step by step:
\begin{itemize}
    \item 
    In Step 1, the constructed $\miscali\primed$ has $0$ calibration error as it always generates perfectly calibrated predictions. 

    \item 
    In Step 2 with Case I, the calibration error incurred in $\miscali\primed$ can be bounded as follows:
    \begin{align*}
        &  
        \sum\nolimits_{(\bayesprediction_L, \bayesprediction_R, \prediction) \text{ in Case I}}
        \barECEnorm{\miscali\primed\mid (\bayesprediction_L, \prediction)}
        + 
        \barECEnorm{\miscali\primed\mid (\bayesprediction_R, \prediction)} \\
        = ~ &  
        \sum\nolimits_{(\bayesprediction_L, \bayesprediction_R, \prediction)\text{ in Case I}} 
        \sum\nolimits_{(i, j):i\le j} \left(
        \miscali\primed_{i, j}(\bayesprediction_L, \prediction) \cdot |\bayesprediction_L-\prediction|^\normexponent
        + 
        \miscali\primed_{i, j}(\bayesprediction_R, \prediction) \cdot |\bayesprediction_R-\prediction|^\normexponent
        \right) \\
        \le ~ & 
        {(1-\scaledownfrac)\cdot \discreparaBudget}
        \sum\nolimits_{(\bayesprediction, \prediction) \text{ in Case I}} 
        \sum\nolimits_{(i, j):i\le j} \left(
        \miscali_{i, j}(\bayesprediction, \prediction) \frac{\bayesprediction_R-\bayesprediction}{\bayesprediction_R-\bayesprediction_L} 
        + 
        \miscali_{i, j}(\bayesprediction, \prediction) \left(1 - \frac{\bayesprediction_R-\bayesprediction}{\bayesprediction_R-\bayesprediction_L}\right)
        \right) 
        \tag{as $\bayesprediction_L\le\bayesprediction_R \le \prediction+\discreparaBudget^{\sfrac{1}{\normexponent}}$}
        \\
        \le ~ & 
        {(1-\scaledownfrac)\cdot \discreparaBudget}~.
    \end{align*}
    
    \item 
    In Step 2 with Case II, we know that, by construction, we have
    \begin{align*}
        \barECEnorm{\miscali\primed\mid (\bayesprediction_L, \prediction)}
        & = \sum\nolimits_{(i, j):i\le j} \miscali\primed_{i, j}(\bayesprediction_L, \prediction) \cdot |\bayesprediction_L-\prediction|^\normexponent \\
        & \le  
        \sum\nolimits_{(i, j):i\le j}
        {(1-\scaledownfrac)}
        \miscali_{i, j}(\bayesprediction, \prediction) \frac{\bayesprediction_R-\bayesprediction}{\bayesprediction_R-\bayesprediction_L} \cdot |\bayesprediction-\prediction|^\normexponent 
        \tag{as $\prediction\le \bayesprediction_L \le\bayesprediction$}\\
        \barECEnorm{\miscali\primed\mid (\bayesprediction_R, \prediction)}
        & = \sum\nolimits_{(i, j):i\le j} \miscali\primed_{i, j}(\bayesprediction_R, \prediction) \cdot |\bayesprediction_R-\prediction|^\normexponent \\
        & = 
        {(1-\scaledownfrac)}
        \sum\nolimits_{(i, j):i\le j} \miscali_{i, j}(\bayesprediction, \prediction) \left(1 - \frac{\bayesprediction_R-\bayesprediction}{\bayesprediction_R-\bayesprediction_L}\right)\cdot |\bayesprediction_R-\prediction|^\normexponent\\
        & \le 
        {(1-\scaledownfrac)}
        \sum\nolimits_{(i, j):i\le j} \miscali_{i, j}(\bayesprediction, \prediction) \left(1 - \frac{\bayesprediction_R-\bayesprediction}{\bayesprediction_R-\bayesprediction_L}\right)\cdot |\bayesprediction_R'-\prediction|^\normexponent \tag{as $\prediction\le \bayesprediction_R\le \bayesprediction_R'$}\\
        & \le 
        {(1+\discrepara)(1-\scaledownfrac)}
        \sum\nolimits_{(i, j):i\le j} \miscali_{i, j}(\bayesprediction, \prediction) \left(1 - \frac{\bayesprediction_R-\bayesprediction}{\bayesprediction_R-\bayesprediction_L}\right)\cdot |\bayesprediction-\prediction|^\normexponent
    \end{align*}
    Thus, the total calibration error of $\miscali\primed$ contributed from this case can be upper bounded by
    \begin{align*}
        & \sum\nolimits_{(\bayesprediction_L, \bayesprediction_R, \prediction) \text{ in Case II}}
        \left(\barECEnorm{\miscali\primed\mid (\bayesprediction_L, \prediction)}
        + 
        \barECEnorm{\miscali\primed\mid (\bayesprediction_R, \prediction)}\right) \\
        \le  ~ & 
        (1+\discrepara)(1 - \scaledownfrac)\cdot 
        \sum\nolimits_{(\bayesprediction, \prediction) \text{ in Case II}}\sum\nolimits_{(i, j):i\le j}  \left(\miscali_{i, j}(\bayesprediction, \prediction) 
        \cdot |\bayesprediction-\prediction|^\normexponent
        + 
        \miscali_{i, j}(\bayesprediction, \prediction) 
        \cdot |\bayesprediction-\prediction|^\normexponent \right) \\
        =  ~ & 
        (1+\discrepara)(1 - \scaledownfrac)\cdot \sum\nolimits_{(\bayesprediction, \prediction) \text{ in Case II}} \cdot \barECEnorm{\miscali\mid (\bayesprediction, \prediction)}\\
        \le  ~ & 
        (1+\discrepara)(1 - \scaledownfrac)
        \barECEnorm{\miscali} 
        \le 
        (1+\discrepara)(1 - \scaledownfrac) \cdot \frac{\discreparaBudget}{\discrepara}~.
        \tag{by definition $\discreparaBudget = \discrepara\cdot \caliBudget^\normexponent$, $\barECEnorm{\miscali}\le \caliBudget^\normexponent$}
    \end{align*}

    \item 
    In Step 2 with Case III, notice that in this case, the constructed $\miscali\primed$ always generates  perfectly calibrated predictions. 
    Moreover, the total probability mass for this case cannot be large, as 
    \begin{align*}
        \barECEnorm{\miscali} 
        & \ge 
        \sum\nolimits_{(\bayesprediction, \prediction) \text{ in Case III }}\barECEnorm{\miscali\mid (\bayesprediction, \prediction)} \\
        & = \sum\nolimits_{(\bayesprediction, \prediction) \text{ in Case III}} \sum\nolimits_{(i, j):i\le j} 
        \miscali_{i, j}(\bayesprediction, \prediction)  |\bayesprediction-\prediction|^\normexponent \\
        & \ge 
        \sum\nolimits_{(\bayesprediction, \prediction) \text{ in Case III}} \sum\nolimits_{(i, j):i\le j} 
        \miscali_{i, j}(\bayesprediction, \prediction)  \discreparaBudget(1+\discrepara)^{\numdeltas-1} 
        \tag{by condition of Case III}\\
        & \ge 
        \sum\nolimits_{(\bayesprediction, \prediction) \text{ in Case III}}
        \sum\nolimits_{(i, j):i\le j} \miscali_{i, j}(\bayesprediction, \prediction) \cdot \frac{\barECEnorm{\miscali}}{\discrepara{(1+\discrepara)}}~.
        \tag{by definition of $\numdeltas$}
    \end{align*}
    Thus, in Case III, we must have $\sum\nolimits_{(\bayesprediction, \prediction) \text{ in Case III}} \sum\nolimits_{(i, j):i\le j} 
    \miscali_{i, j}(\bayesprediction, \prediction) \le \discrepara$.
    Consequently, the loss of the expected utility of the principal in constructed $\miscali\primed$ at this case is upper bounded by $\discrepara{(1+\discrepara)}\cdot \OBJ{\text{\ref{eq:decoupled opt general}} \condition \miscali}$.
\end{itemize}
Putting all pieces together, we have the following guarantee of the ECE of constructed $\miscali$:
\begin{align*}
    \barECEnorm{\miscali} - \barECEnorm{\miscali\primed}  
    & \ge  \barECEnorm{\miscali} 
    - \discreparaBudget(1-\scaledownfrac )
    - (1+\discrepara)(1 - \scaledownfrac) \cdot \frac{\discreparaBudget}{\discrepara} \\
    & = \discreparaBudget\left(\frac{1}{\discrepara}
    - (1-\scaledownfrac )
    - \frac{(1+\discrepara)(1-\scaledownfrac)}{\discrepara}
    \right) = 0
    \tag{by the choice of $\scaledownfrac$}
\end{align*}
In addition, we have the following guarantee of the expected utility of the principal of constructed $\miscali$:
\begin{align*}
    & \OBJ{\text{\ref{eq:decoupled opt general discre}}\condition \miscali\primed}  \\
    \ge  ~ & \OBJ{\text{\ref{eq:decoupled opt general}} \condition \miscali} 
    - \underbrace{\scaledownfrac \OBJ{\text{\ref{eq:decoupled opt general}} \condition \miscali}}_{\text{payoff loss in Step 1}}
    - \underbrace{\discrepara {(1+\discrepara)}\OBJ{\text{\ref{eq:decoupled opt general}} \condition \miscali}}_{\text{payoff loss in Step 2 with Case III}} \\
    \ge  ~ &  (1 - 3 \discrepara)  \OBJ{\text{\ref{eq:decoupled opt general}} \condition \miscali} 
    \tag{by the choice of $\scaledownfrac$}
\end{align*}
We thus finish the proof of \Cref{lem:discre error} as desired.
\end{proof}
Finally, we include \Cref{claim:s must exists} (and its proof) mentioned in our rounding scheme (\Cref{algo:rounding}).
\begin{claim}
\label{claim:s must exists}
In Step 2 with Case II of \Cref{algo:rounding} (line 15), we define $\bayesprediction_R' \gets \prediction+ (\bayesprediction-\prediction)(1+\discrepara)^{\sfrac{1}{\normexponent}}$.
% Then we always have $\bayesprediction_R' > \bayesprediction$.
When $\bayesprediction_R' \ge \truemean_j$, set $\bayesprediction_R \gets \truemean_j$. 
When $\bayesprediction_R' < \truemean_j$, 
there must exist $ s\in[\numdeltas] $ such that $\prediction+(\discreparaBudget(1+\discrepara)^{s})^{\sfrac{1}{\normexponent}}\in [\bayesprediction, \bayesprediction_R']$, set $\bayesprediction_R\gets \prediction+(\discreparaBudget(1+\discrepara)^{s})^{\sfrac{1}{\normexponent}}$.
\end{claim}
\begin{proof}[Proof of \Cref{claim:s must exists}]
To show that when $\bayesprediction_R' < \truemean_j$, 
there must exist $ s\in[\numdeltas] $ such that $\prediction+(\discreparaBudget(1+\discrepara)^{s})^{\sfrac{1}{\normexponent}}\in [\bayesprediction, \bayesprediction_R']$, it suffices to show that there exists $s\in[\numdeltas]$:
\begin{align*}
    \bayesprediction\le 
    \prediction+(\discreparaBudget(1+\discrepara)^{s})^{\sfrac{1}{\normexponent}}\le \bayesprediction_R' 
    ~ \Leftrightarrow ~
    \frac{\ln \frac{(\bayesprediction-\prediction)^\normexponent}{\discreparaBudget}}{\ln (1+\discrepara)}
    \le s 
    \le  \frac{\ln \frac{(\bayesprediction_R'-\prediction)^\normexponent}{\discreparaBudget}}{\ln (1+\discrepara)}
    = 1 + \frac{\ln \frac{(\bayesprediction-\prediction)^\normexponent}{\discreparaBudget}}{\ln (1+\discrepara)}~.
\end{align*}
By the condition $\bayesprediction-\prediction \ge \discreparaBudget^{\sfrac{1}{\normexponent}}$, we observe that 
we must have $\frac{\ln \frac{(\bayesprediction-\prediction)^\normexponent}{\discreparaBudget}}{\ln (1+\discrepara)} \ge 0$. 
We next argue that we must have $\bayesprediction_R' \le \prediction + (\discreparaBudget(1+\discrepara)^\numdeltas)^{\sfrac{1}{\normexponent}}$. To see this, we observe
\begin{align*}
    \prediction + (\discreparaBudget(1+\discrepara)^\numdeltas)^{\sfrac{1}{\normexponent}}- \bayesprediction_R' 
    & =(\discreparaBudget(1+\discrepara)^\numdeltas)^{\sfrac{1}{\normexponent}} - (\bayesprediction-\prediction)(1+\discrepara)^{\sfrac{1}{\normexponent}} \\
    & \ge (\discreparaBudget(1+\discrepara)^\numdeltas)^{\sfrac{1}{\normexponent}} - (\discreparaBudget(1+\discrepara)^{\numdeltas-1})^{\sfrac{1}{\normexponent}}(1+\discrepara)^{\sfrac{1}{\normexponent}}  = 0
\end{align*}
where the last equality holds due to the condition of Case II.
Thus, there must exist $ s\in[\numdeltas] $ such that $\prediction+(\discreparaBudget(1+\discrepara)^{s})^{\sfrac{1}{\normexponent}}\in [\bayesprediction, \bayesprediction_R']$.
This finishes the proof of \Cref{claim:s must exists}.
\end{proof}

\subsection{Analysis of \Cref{algo:fptas}}
\label{subsec:fptas proof}
We are ready to analyze \Cref{algo:fptas} and prove \Cref{thm:fptas}.
\begin{proof}[Proof of \Cref{thm:fptas}]
    In \Cref{algo:fptas}, it first solves \ref{eq:decoupled opt general discre}, which is a linear program with size $\poly(\sfrac{1}{\discrepara}, \numData, \numAction)$, and then constructs the predictor based on the optimal solution of \ref{eq:opt general actionRem}. Both steps require $\poly(\sfrac{1}{\discrepara}, \numData, \numAction)$ running time and thus the algorithm is polynomial-time. Finally, the approximation of the algorithm follows \Cref{prop:LP general} and \Cref{lem:discre error}. This finishes the proof of \Cref{thm:fptas}.
\end{proof}

\section{Polytime Algorithm for \texorpdfstring{$\ell_1$-Norm}{L1-Norm} ECE and \texorpdfstring{$\ell_\infty$-Norm}{L-Infinity-Norm} ECE}
\label{sec:polytime}

In this section, we show that for the most standard ECE metrics, the $K_1$ ECE, and also the $K_\infty$ ECE, there exists a polynomial-time algorithm that can compute the optimal $(\caliBudget, \tnorm)$-calibrated predictor with $\normexponent \in\{1, \infty\}$.
\begin{theorem}[Polynomial-time algorithm for $\ell_1$-norm and $\ell_\infty$-norm ECE]
\label{thm:opt general}
For every persuasive calibration instance with $\ell_1$-norm ECE (resp. $\ell_\infty$-norm) ECE, there exists a linear programming (see \ref{eq:opt general actionRem}) based algorithm (\Cref{algo:poly-time}) that computes an optimal $(\caliBudget,\ell_1)$-calibrated (resp.\ $(\caliBudget,\ell_\infty)$-calibrated) predictor. The running time of the algorithm is $\poly(\numData, \numAction)$, where $\numData = |\stateSpace|, \numAction = |\actionSpace|$ are the number of events and the number of agent's actions, respectively.
\end{theorem}

For $\ell_1$-norm or $\ell_\infty$-norm ECE, it is not difficult to express the $(\caliBudget, \tnorm)$-calibration as a linear constraint on $\predictor = \{\predictor_i(\prediction)\}_{i \in [\numData], \prediction \in [0, 1]}$. However, to obtain an optimal algorithm with polynomial running time, we also need to restrict the continuous space of predictions to a discrete polynomial-size set. To bypass this challenge, we build on an idea from the algorithmic information design literature. 

As an example, in the classic Bayesian persuasion problem \citep{KG-11}, the \emph{revelation principle} assures that it is sufficient to consider signaling schemes that recommend an action. Hence, instead of searching over a (possibly) continuous signal space, it suffices to construct a signaling scheme with a signal space equal to the action space, which is polynomial-sized. However, such an approach seems difficult to apply to our model, as the design space in our model is restricted to predictions rather than arbitrary signals, and the agent in our model follows a much simpler behavior---i.e., always trusting the prediction---rather than engaging in strategic reasoning. Therefore, it is unclear a priori whether a revelation principle can be established.

To prove \Cref{thm:opt general}, we show that the above idea from the algorithmic information design literature indeed applies. To achieve this, we first introduce a new variant of the Bayesian persuasion problem, in which the receiver (in our model): (1) has a utility function that depends linearly on the payoff-relevant state, and (2) instead of updating the belief in a fully Bayesian manner, exhibits a signal-dependent bias when updating the belief. 

In \Cref{subsec:bp w bias}, we formally define this variant of the Bayesian persuasion problem, establish the revelation principle, and provide a solution for this problem, which will be used to solve our persuasive calibration problem and is of independent interest. In \Cref{subsec:instance equivalence}, we show the instance equivalence between our problem and this Bayesian persuasion variant (see \Cref{thm:instance equivalence}), provide \Cref{algo:poly-time}, and prove \Cref{thm:opt general}.

\subsection{(Bayesian) Persuasion with Signal-Dependent Bias}
\label{subsec:bp w bias}

Bayesian persuasion \citep{KG-11} is a classic and important model that has been studied extensively in the literature. In this section, we introduce a new variant---\emph{(Bayesian) persuasion with signal-dependent bias}---and provide a characterization of its optimal solution, which serves as a key ingredient for developing \Cref{algo:poly-time} and proving \Cref{thm:opt general}. Additionally, given the significance of the Bayesian persuasion and information design literature, we believe that our new variant and its results may be of independent interest. Below, we first revisit the classic Bayesian persuasion model and then introduce the new variant with signal-dependent bias.

A Bayesian persuasion instance is a game for two players, a sender and a receiver, and it is specified by a tuple $(\stateSpaceBP, \priorBP, \senderUBP, \receiverUBP, \actionSpaceBP)$: $\stateSpaceBP \subseteq [0, 1]$ denotes the payoff-relevant state space; $\priorBP\in\Delta(\stateSpaceBP)$ represents the common prior for both players, and $\priorBP_{\state}$ denotes the prior probability of state $\state$; 
$\senderUBP(\cdot,\cdot):\actionSpaceBP\times\stateSpace\rightarrow \R$ (resp.\ $\receiverUBP(\cdot,\cdot):\actionSpaceBP\times\stateSpace\rightarrow \R$) is the sender's (resp.\ receiver's) utility function that depends on the receiver's action and the realized state. 
The sender observes the realized state, while the receiver cannot.

The sender's goal is to design a signaling scheme, given by conditional distributions $(\signalingscheme_{\state}(\cdot))_{\state\in\stateSpaceBP}$ where each $\signalingscheme_{\state}(\cdot)\in\Delta(\signalSpace)$  specifies a distribution over a measurable signal space $\signalSpace$, to maximize her expected payoff. 
In classic Bayesian persuasion, upon observing a realized signal, the receiver updates his posterior belief about the underlying state in a {\em Bayesian} manner 
and takes an action that maximizes his expected utility (subject to updated posterior belief).

\xhdr{(Bayesian) persuasion with signal-dependent bias}
In our considered variant, the receiver has a utility function $\receiverU(\action, \cdot)$ that linearly depends on the state.
Thus, only the mean of the belief matters for his optimal action.
In addition, when updating his belief, the receiver exhibits a signal-dependent bias defined as follows:
\begin{definition}[Signal-dependent bias in belief updates]
\label{defn:receiver behavior}
Upon observing a signal realization $\signal\sim\actionpredictor$, the receiver computes the following mean of the updated belief, denoted by $\prediction(\signal)$:
\begin{align}
    \label{eq:receiver bias}
    \prediction(\signal)
    \triangleq
    \frac{\bias(\signal)+
    \sum\nolimits_{\state\in\stateSpaceBP} \priorBP_{\state} \signalingscheme_{ \state}(\signal) \cdot \state}{\sum\nolimits_{\state\in\stateSpaceBP} \priorBP_{\state} \signalingscheme_{ \state}(\signal)} \in [0, 1]
\end{align}
where $\bias(\signal) \in \R$ represents receiver's bias when updating his belief (relative to an exact Bayesian update).
The receiver then takes an action maximizing his expected utility according to the mean $\prediction(\signal)$ of his belief,
\begin{align}
    \label{eq:receiver response}
    \bestr(\signal) \triangleq \argmax_{\action\in\actionSpaceBP} \receiverUBP\left(\action, \prediction(\signal) \right)~.
\end{align}
\end{definition}
Intuitively, $\bias(\signal)$ captures the receiver's (possibly) irrational bias when updating their belief.
As a sanity check, when $\bias(\signal) \equiv 0$ for all $\signal\in\signalSpace$, the receiver's belief update in Eqn.~\eqref{eq:receiver bias} recovers the standard Bayesian posterior mean. 
In our variant, such biases are chosen to be favorable to the sender but must satisfy the following \emph{bounded rationality} constraint:
\begin{align}
    \label{eq:aggregated IC}
    \left(\expect[\signal\sim \signalingscheme]{\left|\frac{\bias(\signal)}{\sum\nolimits_{\state\in\stateSpaceBP} \priorBP_{\state} \cdot \signalingscheme_{\state}(\signal)}\right|^{\normexponent}}\right)^{\sfrac{1}{\normexponent}}\le \caliBudget~.
\end{align}
The above constraint essentially regulates that the aggregated biases with being weighted by the subjective belief are upper bounded by a budget $\caliBudget \ge 0$.
As the biases are chosen to be favorable to the sender, the sender can also optimize over the bias assignments subject to the above constraint.
We refer to the pair of signaling scheme $\actionpredictor$ and a bias assignment $\bias$ as the sender's strategy.

Our formulation above, in particular the favorable bias assignments, shares similarity to the literature on $\caliBudget$-incentive compatible (IC) mechanism design (see, e.g., \citealp{BBHM-05, HL-10,HL-15,CDW-12,COVZ-21,GW-21,BBC-24}) where the ties are also usually chosen to be favorable to the mechanism designer.\footnote{There is also another line of research that studies robust mechanism design when agents select the $\caliBudget$-best response, which minimizes the designer's utility \citep[e.g.,][]{GHWX-23,CL-23,YZ-24}.} Notably, even when selling a single item to a single buyer, the optimal $\caliBudget$-IC mechanism admits no simple structural characterization \citep{BBC-24}. In contrast, by using the equivalence between the persuasive calibration and Bayesian persuasion with signal-dependent bias (\Cref{thm:instance equivalence}), the structural characterizations developed for the former calibration problem in \Cref{sec:optimal structure} can be translated into structural results for the latter persuasion problem as well.

We also remark that many previous works have explored relaxing the Bayesian rationality assumptions in the classic Bayesian persuasion problem, either from the action-taking perspective (see, e.g., \citealp{FHT-24,YZ-24}) or from the belief-updating perspective (see, e.g., \citealp{BC-16,P-19,DZ-22}).
Our new variant, Bayesian persuasion with signal-dependent bias, falls into the latter category. 
In general, the revelation principle -- a commonly-used principle to analyze the classic Bayesian persuasion problem --  need not hold 
when the receiver is subject to biases in probabilistic inference, where the main reason is that 
biased beliefs typically break linearity, it is no longer true that the convex combinations of implementable outcomes stay implementable in the same way \citep{DZ-22}.
However, as we show below, the revelation principle still holds in our considered persuasion problem. 
\begin{definition}[Direct and IC strategy]
We say a sender's strategy $(\actionpredictor, \bias)$, i.e., a pair of signaling scheme $\actionpredictor$ and bias assignment $\bias$, is {\em direct} and {\em incentive compatible (IC)}, if it satisfies (i) its signal space equals to the receiver's action space; (ii) upon receiving a recommended action, it is indeed receiver's best response (defined as in Eqn.~\eqref{eq:receiver response}) to follow this action recommendation. 
\end{definition}

\begin{lemma}[Revelation principle]
\label{lem:revelation}
For the persuasion problem with signal-dependent bias, the revelation principle holds: for every sender's strategy $(\signalingscheme, \bias)$, there exists a direct and IC strategy $(\signalingscheme\primed, \bias\primed)$ that achieves the same outcome, i.e., conditional on the same state realization, the distributions of actions selected by the receiver are the same under both the original strategy $(\signalingscheme, \bias)$ and the new direct and IC strategy $(\signalingscheme\primed, \bias\primed)$.  

Therefore, there always exists a direct and IC sender's strategy that maximizes the sender's expected utility.
\end{lemma} 
The intuition behind \Cref{lem:revelation} is as follows:
for any pair belief means $\prediction(\signal_1), \prediction(\signal_2)$ induced from two signals $\signal_1, \signal_2$, if they lead to a same receiver action, then this action remains optimal for the new belief mean obtained through convex combination of them.
Meanwhile, this convex combination does not increase the aggregated biases.
\begin{proof}[Proof of \Cref{lem:revelation}]
Let $\actionpredictor, \bias$ satisfy the conditions in Eqn.~\eqref{eq:aggregated IC}.
Suppose under two different signals $\signal_1, \signal_2 \sim \actionpredictor$ with mean-belief update biases $\bias(\signal_1), \bias(\signal_2)$, the receiver takes the same action, namely, $\bestr(\signal_1) = \bestr(\signal_2) \equiv a$.
Consider a new signaling $\signalingscheme\primed$ that contracts these two signals $\signal_1, \signal_2$ into a same signal $\signal\primed$, namely, $\actionpredictor_i\primed(\signal\primed) = \actionpredictor_i(\signal_1) + \actionpredictor_i(\signal_2)$ for all $i\in[\numData]$, and a new mean-belief update bias assignment $\bias\primed(\signal\primed) = \bias(\signal_1) + \bias(\signal_2)$. 
From Eqn.~\eqref{eq:receiver response}, we know that for any action $\action'\in\actionSpaceBP$ and signal $\signal\in\{\signal_1, \signal_2\}$, the receiver's utility satisfies
\begin{align*}
    \receiverUBP(\action, \prediction(\signal)) \ge \receiverUBP(\action', \prediction(\signal))~,
\end{align*}
where $\prediction(\signal)$ is defined in Eqn.~\eqref{eq:receiver bias}.
By algebra, we have 
\begin{align*}
    & \receiverUBP(\action, \prediction(\signal_1))\cdot \sum\nolimits_{\state\in\stateSpaceBP}\priorBP_\state\actionpredictor_\state(\signal_1) + 
    \receiverUBP(\action, \prediction(\signal_2))\cdot \sum\nolimits_{\state\in\stateSpaceBP}\priorBP_\state\actionpredictor_\state(\signal_2) \\
    \ge ~ & 
    \receiverUBP(\action', \prediction(\signal_1))\cdot \sum\nolimits_{\state\in\stateSpaceBP}\priorBP_\state\actionpredictor_\state(\signal_1) + 
    \receiverUBP(\action', \prediction(\signal_2))\cdot \sum\nolimits_{\state\in\stateSpaceBP}\priorBP_\state\actionpredictor_\state(\signal_2)
\end{align*}
Since function $\receiverUBP(\action, \state)$ is linear over the states $\state\in\stateSpaceBP$, the above inequality implies that for any action $\action'\in\actionSpaceBP$
\begin{align*}
    \receiverUBP(\action, \prediction(\signal\primed)) \ge \receiverUBP(\action', \prediction(\signal\primed))~,
\end{align*}
Thus, both strategies $(\actionpredictor, \bias)$ and $(\actionpredictor\primed, \bias\primed)$ induce same distribution for the receiver's action conditional on each state.

Meanwhile, by Jensen's inequality, the LHS of condition \eqref{eq:aggregated IC} under $(\actionpredictor\primed, \bias\primed)$ is weakly smaller than that of $(\actionpredictor, \bias)$. 
Thus, if multiple distinct signals lead to a same receiver's action, one can contract these signals to a direct signal that recommends the receiver to take this action that is always followed by the receiver without violating the condition Eqn.~\eqref{eq:aggregated IC}.
\end{proof}
With \Cref{lem:revelation}, we can formulate the sender's problem as finding an optimal direct and IC strategy, which can be further characterized by an maximization program with polynomial numbers of constraints and variables.
\begin{proposition}[Optimization program characterization]
\label{cor:actionRem BP}
For the persuasion problem with signal-dependent bias, the optimal sender's strategy $(\optsignaling, \optbias)$ and her optimal expected utility is the solution of the following maximization program:
\begin{align}
    \label{eq:opt general actionRem}
    \arraycolsep=5.4pt\def\arraystretch{1}
    \tag{$\textsc{P-ActRec}$}
    &\begin{array}{llll}
    \max\limits_{\actionpredictor \ge 0; \bias} &
    \displaystyle
    \sum\nolimits_{\state\in\stateSpaceBP} \sum\nolimits_{\action\in \actionSpace} \priorBP_\state  \actionpredictor_\state(\action) \cdot \senderUBP(\action, \state)
    ~ & \text{s.t.} &
    \vspace{1mm}
    \\
    & 
    \displaystyle 
    \left(\sum\nolimits_{\action\in\actionSpaceBP} \sum\nolimits_{\state\in\stateSpaceBP}\priorBP_\state \actionpredictor_\state(\action) \cdot \left|\frac{\bias(\action)}{\sum\nolimits_{\state\in\stateSpaceBP}\priorBP_\state \actionpredictor_\state(\action)}\right|^\normexponent\right)^{\frac{1}{\normexponent}} \le \caliBudget
    &  
    \vspace{2mm}
    \\
    & 
    \displaystyle
    \receiverUBP\left(\action, 
    \bias(\action)+
    \sum\nolimits_{\state\in\stateSpaceBP} \priorBP_{\state} \signalingscheme_{ \state}(\action) \cdot \state
    \right)
    \geq 
    \receiverUBP\left(\action',
    \bias(\action)+
    \sum\nolimits_{\state\in\stateSpaceBP} \priorBP_{\state} \signalingscheme_{ \state}(\action) \cdot \state
    \right)
    &  \action, \action'\in\actionSpaceBP
    \vspace{3mm}
    \\
    & 
    \displaystyle
    \sum\nolimits_{\action\in\actionSpaceBP} \actionpredictor_\state(\action)
    = 1
    &  
    \state\in\stateSpaceBP \quad 
    \vspace{2mm}
    \\
    & 
    \displaystyle 
    \bias(\action) + \sum\nolimits_{\state\in\stateSpaceBP} \priorBP_\state\actionpredictor_\state(\action) \cdot \state
    \geq 0 
    &  
    \action\in\actionSpaceBP \quad 
    \vspace{2mm}
    \\
    & 
    \displaystyle 
    \bias(\action) + \sum\nolimits_{\state\in\stateSpaceBP}
    \priorBP_\state\actionpredictor_\state(\action) \cdot \state
    \leq \sum\nolimits_{\state\in\stateSpaceBP} \priorBP_\state\actionpredictor_\state(\action)
    &  
    \action\in\actionSpaceBP \quad 
    \end{array}
\end{align}
which has $\poly(|\stateSpaceBP|, |\actionSpaceBP|)$ variables, $\{\signalingscheme_{\state}(\action), \bias(\action)\}_{\state\in\stateSpaceBP, \action\in\actionSpaceBP}$, and $\poly(|\stateSpaceBP|, |\actionSpaceBP|)$ constraints. Both the objective function and all constraints, except for the first constraint, are linear.\footnote{Recall that function $\receiverUBP(\action, \state)$ is linear over the states $\state\in\stateSpaceBP$.}

In the special case of $\normexponent = 1$ or $\normexponent = \infty$, the first constraint in \ref{eq:opt general actionRem} is equivalent to 
\begin{align*}
    \text{when } \normexponent = 1: 
    & \quad \sum\nolimits_{\action\in\actionSpace} \left|\bias(\action)\right|
    \le \caliBudget~; \\ 
    \text{when } \normexponent = \infty:
    & \quad \left|\bias(\action)\right|
    \le \caliBudget  ~, \quad \action\in\actionSpace~.
\end{align*}
which is essentially a linear constraint. Therefore, \ref{eq:opt general actionRem} becomes a linear program for $\normexponent = 1$ or $\normexponent = \infty$.
\end{proposition}

As mentioned in \Cref{cor:actionRem BP}, \ref{eq:opt general actionRem} is a polynomial-size linear program when $\normexponent \in \{1, \infty\}$ and can thus be solved in polynomial time. For other values of $\normexponent \in (1, \infty)$, the first constraint in \ref{eq:opt general actionRem} becomes non-linear. We leave it as an interesting future direction to explore whether \ref{eq:opt general actionRem} can be solved in polynomial time for general $\normexponent$, e.g., by constructing a time-efficient separation oracle.

\begin{proof}[Proof of \Cref{cor:actionRem BP}]
    Invoking \Cref{lem:revelation}, there exists an optimal strategy $(\signalingscheme^*, \bias^*)$ for the sender that is direct and IC. We first verify that it is a feasible solution of \ref{eq:opt general actionRem} with the same objective value. Since strategy $(\signalingscheme^*, \bias^*)$ is direct and IC, the expected utility of the sender is equal to the objective value of $(\signalingscheme^*, \bias^*)$ in \ref{eq:opt general actionRem}. The first constraint is satisfied since strategy $(\signalingscheme^*, \bias^*)$ satisfies the bounded rationality constraint defined in Eqn.~\eqref{eq:aggregated IC}. The second constraint is satisfied since strategy $(\signalingscheme^*, \bias^*)$ is IC and utility function $\receiverUBP(\action, \state)$ is linear over the states $\state\in\stateSpaceBP$. The third, fourth and fifth constraints are satisfied due to the feasibility of strategy $(\signalingscheme^*, \bias^*)$.

    Finally, following the same argument, every feasible solution of \ref{eq:opt general actionRem} represents a sender's strategy that is well-defined, feasible, direct and IC and has the expected utility of the sender equal to the objective of the solution. Therefore, we finish the proof of \Cref{cor:actionRem BP} as desired.
\end{proof}

\subsection{Instances Equivalence and Analysis of \Cref{algo:poly-time}}
\label{subsec:instance equivalence}

In this section, we formally establish the equivalence between the persuasive calibration problem and the persuasion problem with signal-dependent bias (\Cref{subsec:bp w bias}). We then use the equivalence to develop \Cref{algo:poly-time} and prove \Cref{thm:opt general}.

At first glance, the two problems might seem different, as the agent in the persuasive calibration model always trusts the prediction, while the agent in the persuasion model performs a biased Bayesian update to form their posterior belief. However, the ECE budget constraint in the former and the bounded rationality constraint (Eqn.~\eqref{eq:aggregated IC}) in the latter model enable us to establish the an equivalence between the two problems. The proof of \Cref{thm:instance equivalence} is at the end of this section.

\begin{theorem}[Instance equivalence]
\label{thm:instance equivalence}
Fix any instance of the persuasive calibration problem, specified by a tuple $(\stateSpace, \prior, \senderU, \receiverU, \actionSpace, (\caliBudget, \tnorm))$.
Consider the following instance $(\stateSpace, \prior, \senderUBP, \receiverUBP, \actionSpace, (\caliBudget, \tnorm))$ of the (Bayesian) persuasion problem with signal-dependent bias:
\begin{itemize}
    \item[(i)] the sender's/receiver's utility functions are given as:
    \begin{alignat*}{3}
        \senderUBP(\action, \state_i) 
        & \gets (1 - \state_i) \cdot \senderU_i(\action, \outcome = 0) + \state_i \cdot \senderU_i(\action, \outcome = 1), 
        && \quad \action\in\actionSpace, i\in[\numData]~; \\
        \receiverUBP(\action, \state_i) 
        & \gets (1 - \state_i) \cdot \receiverU(\action, \outcome = 0) + \state_i \cdot \receiverU(\action, \outcome = 1), 
        && \quad \action\in\actionSpace, i\in[\numData]~;
    \end{alignat*}
    \item[(ii)] the sender's strategy follows the bounded rationality constraint defined in Eqn.~\eqref{eq:aggregated IC}.
\end{itemize}
Then an optimal $(\caliBudget, \tnorm)$-calibrated predictor $\optpredictor$ in the persuasive calibration instance can be converted as an optimal sender's strategy $(\optsignaling, \optbias)$ 
in the above defined instance of the (Bayesian) persuasion problem with signal-dependent bias, and vice versa. Specifically, given an optimal sender's strategy $(\optsignaling, \optbias)$ that is direct and IC, an optimal $(\caliBudget, \tnorm)$-calibrated predictor $\optpredictor$ can be constructed as 
\begin{align}
    \label{eq:construct f opt}
    \optpredictor_i(\prediction) \gets 
    \sum\nolimits_{\action\in\actionSpace} \optsignaling_{\state_i}(\action)\cdot\indicator{\frac{\sum\nolimits_{i\in[\numData]} \prior_i\optsignaling_{\state_i}(\action) \cdot \truemean_i + \optbias(\action)}{\sum\nolimits_{i\in[\numData]} \prior_i\optsignaling_{\state_i}(\action)} = \prediction}, \quad
    i\in[\numData], ~~ \prediction\in[0,1]~.
\end{align}
Given an optimal $(\caliBudget, \tnorm)$-calibrated predictor $\optpredictor$, an optimal sender's strategy $(\optsignaling, \optbias)$ that is direct and IC can be constructed as 
\begin{align*}
    \optsignaling_{\state_i}(\action) 
    & \gets \sum\nolimits_{\prediction\in[0, 1]: \bestr(\prediction) = \action} \optpredictor_i(\prediction), \quad 
    && \action\in\actionSpace, i\in[\numData] \\
    \optbias(\action) 
    & \gets \sum\nolimits_{\prediction\in[0, 1]: \bestr(\prediction) = \action} \sum\nolimits_{i\in[\numData]} \prior_i \optpredictor_i(\prediction) (\prediction - \truemean_i), \quad 
    && \action\in\actionSpace
\end{align*}
where $\bestr(\prediction)$ is defined in Eqn.~\eqref{eq:agent bestr}.
\end{theorem}

We remark that in the corresponding Bayesian persuasion instance constructed in \Cref{thm:instance equivalence}, the receiver's utility $\receiverU(\action, \state)$ is a linear function over the states $\state\in\stateSpaceBP$, while the sender's utility could be general. 
Combining \Cref{thm:instance equivalence} and \Cref{cor:actionRem BP}, we are ready to provide \Cref{algo:poly-time} and prove \Cref{thm:opt general}. 

\begin{algorithm}[ht]
    \caption{LP-based polynomial-time algorithm for solving optimal predictor with $\normexponent \in \{1, \infty\}$}
    \begin{algorithmic}[1]
    \State \textbf{Input:} persuasive calibration instance $(\stateSpace, \prior, \senderU, \receiverU, \actionSpace, (\caliBudget, \tnorm))$ with $\normexponent\in\{1, \infty\}$. 
    \State 
    Solve linear program \ref{eq:opt general actionRem} using the persuasion instance defined in \Cref{thm:instance equivalence} and let $(\optsignaling, \optbias)$ be its solution.
    \State 
    Use the optimal solution $(\optsignaling, \optbias)$ to 
    construct the predictor $\optpredictor$ based on Eqn.~\eqref{eq:construct f opt}.

    \State \textbf{Output:} predictor $\optpredictor$.
    \end{algorithmic}
    \label{algo:poly-time}
\end{algorithm}

\begin{proof}[Proof of \Cref{thm:opt general}]
In \Cref{algo:poly-time}, it first solves \ref{eq:opt general actionRem}, which is a linear program with size $\poly(\numData, \numAction)$, and then constructs the predictor based on the optimal solution of \ref{eq:opt general actionRem}. Both steps require $\poly(\numData, \numAction)$ running time and thus the algorithm is polynomial-time. Finally, the correctness of the algorithm follows \Cref{cor:actionRem BP} and \Cref{thm:instance equivalence}. This finishes the proof of \Cref{thm:opt general}.
\end{proof}
We conclude this section by presenting the proof for \Cref{thm:instance equivalence}.

\begin{proof}[Proof of \Cref{thm:instance equivalence}]
In this proof, we use $\Payoff{\actionpredictor, \bias}$ to denote the sender's payoff from the strategy $(\actionpredictor, \bias)$.
To prove the theorem statement, we show that: (1) every direct and IC strategy $(\actionpredictor, \bias)$  that satisfies the condition in Eqn.~\eqref{eq:aggregated IC} can be converted into an $(\caliBudget, \tnorm)$-calibrated $\predictor$ whose principal's payoff equals to the original sender's payoff;
(2) 
every $(\caliBudget, \tnorm)$-calibrated $\predictor$ can be converted into a direct and IC sender's strategy $(\actionpredictor, \bias)$ 
whose sender's payoff equals to original principal's payoff from $\predictor$.
Since $\stateSpaceBP = \stateSpace$, we will simplify the notation $\actionpredictor_{\truemean_i} = \actionpredictor_{i}$ for $i\in[\numData]$.

\xhdr{From sender's strategy $(\actionpredictor, \bias)$ to principal's predictor $\predictor$}
By \Cref{lem:revelation}, we consider a direct and IC strategy $(\actionpredictor, \bias)$.
Now given a pair $(\actionpredictor, \bias)$ satisfying the condition \eqref{eq:aggregated IC}, we define a prediction $\actionpredictortoPos(\action)$ for the action $a\in\supp(\actionpredictor)$ and construct predictor $\predictor$ as follows:
\begin{align*}
    \actionpredictortoPos(\action) \triangleq \frac{\sum\nolimits_{i\in[\numData]} \prior_i\actionpredictor_i(\action) \cdot \truemean_i + \bias(\action)}{\sum\nolimits_{i\in[\numData]} \prior_i\actionpredictor_i(\action)}~;
    \quad
    \predictor_i(\prediction) \gets 
    \sum\nolimits_{\action\in\actionSpace}\actionpredictor_i(\action)\cdot\indicator{\actionpredictortoPos(\action) = \prediction}, \quad
    i\in[\numData], ~~ \prediction\in[0, 1]
\end{align*}
Clearly, by condition \eqref{eq:receiver response}, we know that $\actionpredictortoPos(\action)\in[0, 1]$ for all $\action\in\actionSpace$.
By construction, it can be verified that $\predictor_i(\cdot)$ is a valid distribution over $[0, 1]$.
We next argue that the constructed $\predictor$ is $(\caliBudget, \tnorm)$-calibrated and it satisfies that $\Payoff{\predictor} = \Payoff{\actionpredictor, \bias}$. 
To see $\predictor$ is indeed $(\caliBudget, \tnorm)$-calibrated, we observe 
\begin{align*}
    \left(\ECEnorm{\predictor}\right)^\normexponent
    & = \sum\nolimits_{\prediction\in\supp(\predictor)} \sum\nolimits_{i\in[\numData]} \prior_i\predictor_i(\prediction) \cdot \left|\prediction - \frac{\sum\nolimits_{k\in[\numData]} \prior_k\predictor_k(\prediction)\truemean_k}{\sum\nolimits_{j\in[\numData]} \prior_j\predictor_j(\prediction)}\right|^{\normexponent}\\
    & = 
    \sum\nolimits_{\action\in\actionSpace} \sum\nolimits_{i\in[\numData]} \prior_i \actionpredictor_i(\action) \cdot \left|\actionpredictortoPos(\action) - \frac{\sum\nolimits_{k\in[\numData]} \prior_k\actionpredictor_k(a)\truemean_k}{\sum\nolimits_{j\in[\numData]} \prior_j\actionpredictor_j(a)}\right|^{\normexponent} \\
    & = 
    \sum\nolimits_{\action\in\actionSpace} \sum\nolimits_{i\in[\numData]} \prior_i \actionpredictor_i(\action) \cdot \left|\frac{\bias(\action)}{\sum\nolimits_{j\in[\numData]} \prior_j\actionpredictor_j(a)}\right|^{\normexponent}
    \le \caliBudget^\normexponent~.
\end{align*}
To see the equal payoff, we first notice that for every event $i\in[n]$, $\indirectsenderU_i(\actionpredictortoPos(\action)) = \senderUBP(\action,\state_i)$ for all $\action\in\actionSpace$. 
This is by the condition \eqref{eq:receiver response}.
Consequently, we have
\begin{align*}
    \Payoff{\predictor}
    = \sum\nolimits_{i\in[\numData]} \prior_i \cdot \sum\nolimits_{\prediction\in [0, 1]} \predictor_i(\prediction) \indirectsenderU_i(\prediction) 
    & = 
    \sum\nolimits_{i\in[\numData]} \prior_i \cdot \sum\nolimits_{\action\in\actionSpace} \predictor_i(\actionpredictortoPos(\action)) \indirectsenderU_i(\actionpredictortoPos(\action)) \\
    & = 
    \sum\nolimits_{i\in[\numData]} \prior_i \cdot \sum\nolimits_{\action\in\actionSpace} \sum\nolimits_{\action'\in\actionSpace:\actionpredictortoPos(\action') = \actionpredictortoPos(\action)}\actionpredictor_i(\action) \senderUBP(\action, \truemean_i)\\
    & = 
    \sum\nolimits_{i\in[\numData]} \prior_i \cdot \sum\nolimits_{\action\in\actionSpace} \actionpredictor_i(\action) \senderUBP(\action, \truemean_i) 
     = \Payoff{\actionpredictor, \bias}
\end{align*}

\xhdr{From principal's predictor $\predictor$ to sender's strategy $(\actionpredictor, \bias)$}
Given an $(\caliBudget, \tnorm)$-calibrated predictor $\predictor$, for any $\prediction\in \supp(\predictor)$, we let $\bestr(\prediction) \in \actionSpace$ denote the receiver's optimal action upon receiving a prediction $\prediction$.
We consider a signal space $\signalSpace \gets \actionSpace$, and define $\actionpredictor_i(\cdot)\in\Delta(\actionSpace)$ for $i\in[\numData]$ and  $\bias(\cdot): \actionSpace\rightarrow \R$ as follows
\begin{alignat*}{3}
    \actionpredictor_i(\action) 
    & \gets \sum\nolimits_{\prediction\in[0, 1]: \bestr(\prediction) = \action} \predictor_i(\prediction), \quad 
    && \action\in\actionSpace, i\in[\numData] \\
    \bias(\action) 
    & \gets \sum\nolimits_{\prediction\in[0, 1]: \bestr(\prediction) = \action} \sum\nolimits_{i\in[\numData]} \prior_i \predictor_i(\prediction) (\prediction - \truemean_i), \quad 
    && \action\in\actionSpace
\end{alignat*}
We next show that by construction, the above $\actionpredictor, \bias$ satisfy the condition \eqref{eq:aggregated IC}, and moreover, they yield same payoff. 
It is easy to see that, by construction, $\actionpredictor_i$ is a valid distribution over $\signalSpace$:
\begin{align*}
    \sum\nolimits_{\action\in\actionSpace} \actionpredictor_i(\action) = 
    \sum\nolimits_{\action\in\actionSpace} 
    \sum\nolimits_{\prediction\in[0, 1]: \bestr(\prediction) = \action} \predictor_i(\prediction)
    = \sum\nolimits_{\prediction\in[0, 1]} \predictor_i(\prediction) = 1, \quad i\in[\numData]~.
\end{align*}
We next argue the Eqn.~\eqref{eq:aggregated IC} is satisfied for $\actionpredictor, \bias$. 
\begin{align*}
    & ~ \sum\nolimits_{\action\in\actionSpace} \sum\nolimits_{i\in[\numData]}\prior_i \actionpredictor_i(\action) \cdot \left|\frac{\bias(\action)}{\sum\nolimits_{j\in[\numData]}\prior_j \actionpredictor_j(a)}\right|^\normexponent \\
    & =
    \sum\nolimits_{\action\in\actionSpace} \predictor(\action) \cdot \frac{1}{\predictor(\action)^\normexponent} \left|\sum\nolimits_{\prediction\in[0, 1]: \bestr(\prediction) = \action} \sum\nolimits_{k\in[\numData]} \prior_k \predictor_k(\prediction) (\prediction - \truemean_k)\right|^\normexponent\\
    & \le
    \sum\nolimits_{\action\in\actionSpace} \predictor(\action) \cdot \frac{1}{\predictor(\action)^\normexponent} \sum\nolimits_{\prediction\in[0, 1]: \bestr(\prediction) = \action} 
    \left|\sum\nolimits_{k\in[\numData]} \prior_k \predictor_k(\prediction) (\prediction - \truemean_k)\right|^\normexponent
    \tag{due to Triangle inequality}\\
    & \le
    \sum\nolimits_{\action\in\actionSpace}  \sum\nolimits_{\prediction\in[0, 1]: \bestr(\prediction) = \action} 
    \predictor(\prediction) \cdot \frac{1}{\predictor(\prediction)^\normexponent}
    \left|\sum\nolimits_{k\in[\numData]} \prior_k \predictor_k(\prediction) (\prediction - \truemean_k)\right|^\normexponent
    \tag{due to $\predictor(\action)\ge \predictor(\prediction)$}\\
    & = 
    \sum\nolimits_\prediction \sum\nolimits_{i\in[\numData]} \prior_i\predictor_i(\prediction) 
    \left|\prediction - \frac{\sum\nolimits_{i\in[\numData]} \prior_i \predictor_i(\prediction)  \truemean_i}{\sum\nolimits_{i\in[\numData]} \prior_i \predictor_i(\prediction)  }\right|^{\normexponent}
    \le \caliBudget^\normexponent~.
\end{align*}
where $\predictor(\action) \triangleq \sum\nolimits_{\prediction\in[0, 1]: \bestr(\prediction) = \action} \sum\nolimits_{i\in[\numData]} \prior_i \predictor_i(\prediction) \ge 
\sum\nolimits_{i\in[\numData]}
\prior_i \predictor_i(\prediction) \triangleq \predictor(\prediction)$ for all $\prediction$ with $\bestr(\prediction) = a$.
To see the mean-belief $\prediction(a) \in [0, 1]$,
we note that for any $\action\in\actionSpace$, 
\begin{align*}
    & \sum\nolimits_{i\in[\numData]} \prior_i\actionpredictor_i(\action) \cdot \truemean_i + \bias(\action) \\
    = ~ & 
    \sum\nolimits_{i\in[\numData]} \prior_i\truemean_i \sum\nolimits_{\prediction\in[0, 1]: \bestr(\prediction) = \action} \predictor_i(\prediction) + \sum\nolimits_{\prediction\in[0, 1]: \bestr(\prediction) = \action} \sum\nolimits_{i\in[\numData]} \prior_i \predictor_i(\prediction) (\prediction - \truemean_i) \\
    \overset{(a)}{=} ~ & 
    \sum\nolimits_{i\in[\numData]}\prior_i\cdot\sum\nolimits_{\prediction\in[0, 1]: \bestr(\prediction) = \action} \predictor_i(\prediction)\prediction
    \overset{(b)}{\le} 
    \sum\nolimits_{i\in[\numData]}\prior_i\cdot\actionpredictor_i(\action)~,
\end{align*}
where from equality (a), we know that $\sum\nolimits_{i\in[\numData]} \prior_i\actionpredictor_i(\action) \cdot \truemean_i + \bias(\action) \ge 0$ as we always have $q\ge 0$; and inequality (b) utilizes the fact that $\prediction\in[0, 1]$ and the definition of $\actionpredictor_i(\action)$.

To see the strategy $(\actionpredictor, \bias)$ is IC,
we recall that for any prediction $\prediction\in\{\prediction\in[0, 1]: \bestr(\prediction) = a\}$ where the receiver's optimal decision is taking action $\action\in\actionSpace$, we must have 
\begin{align*}
    \prediction \cdot \receiverU(\action, \outcome = 1) + 
    (1- \prediction) \cdot \receiverU(\action, \outcome = 0) \ge 
    \prediction \cdot \receiverU(\action', \outcome = 1) + 
    (1- \prediction) \cdot \receiverU(\action', \outcome = 0), \quad \forall \action'\in\actionSpace~.
\end{align*}
Rearranging, multiplying on both sides with $\prior_i\predictor_i(\prediction)$ and taking the summation on both sides over all $i\in[\numData]$, we then have for any $\prediction\in\{\prediction\in[0, 1]: \bestr(\prediction) = a\}$
\begin{align*}
    \sum\nolimits_{i\in[\numData]}\prior_i \predictor_i(\prediction) \prediction \cdot \receiverU_1(\action, \action') \ge 
    \sum\nolimits_{i\in[\numData]}\prior_i \predictor_i(\prediction) (\prediction- 1) \cdot \receiverU_0(\action, \action'), \quad \forall \action'\in\actionSpace
\end{align*}
Taking the summation on both sides over all $\prediction\in\{\prediction\in[0, 1]: \bestr(\prediction) = a\}$, we then have for any $\action'\in\actionSpace$,
\begin{align}
    \sum\limits_{i\in[\numData]}\prior_i \sum\limits_{\prediction\in[0, 1]: \bestr(\prediction) = \action} \predictor_i(\prediction) \prediction \cdot \receiverU_1(\action, \action') \ge 
    \sum\limits_{i\in[\numData]}\prior_i \sum\limits_{\prediction\in[0, 1]: \bestr(\prediction) = \action} \predictor_i(\prediction) (\prediction- 1) \cdot \receiverU_0(\action, \action')~. \label{ineq:IC interim}
\end{align}
By construction, we have
\begin{align*}
    \text{LHS of } \eqref{ineq:IC interim} & = 
    \left(\sum\nolimits_{i\in[\numData]} \prior_i\actionpredictor_i(\action) \cdot \truemean_i + \bias(\action)\right) \cdot \receiverU_1(\action, \action') \\
    \text{RHS of } \eqref{ineq:IC interim} & = 
    \left(\sum\nolimits_{i\in[\numData]} \prior_i\actionpredictor_i(\action) \cdot \truemean_i + \bias(\action) -  \sum\nolimits_{i\in[\numData]}\prior_i \sum\nolimits_{\prediction\in[0, 1]: \bestr(\prediction) = \action} \predictor_i(\prediction)\right) \cdot \receiverU_0(\action, \action') \\
    & = 
    \left(\bias(\action) - \sum\nolimits_{i\in[\numData]} \prior_i\actionpredictor_i(\action) \cdot (1-\truemean_i)\right) \cdot \receiverU_0(\action, \action')~,
\end{align*}
which implies the condition \eqref{eq:receiver response} for the receiver receiving a signal $\action$.
Moreover, we have
\begin{align*}
    \Payoff{\actionpredictor, \bias}
    & = \sum\nolimits_{i\in[\numData]} \prior_i \cdot \sum\nolimits_{\action\in \actionSpace} \actionpredictor_i(\action) \senderUBP(\action, \truemean_i) \\
    & = 
    \sum\nolimits_{i\in[\numData]} \prior_i \cdot \sum\nolimits_{\action\in \actionSpace} \sum\nolimits_{\prediction\in[0, 1]: \bestr(\prediction) = \action} \predictor_i(\prediction) \senderUBP(\action, \truemean_i)\\
    & = 
    \sum\nolimits_{i\in[\numData]} \prior_i \cdot \sum\nolimits_{\prediction\in[0, 1]} \predictor_i(\prediction) \indirectsenderU_i(\prediction) = \Payoff{\predictor}~.
\end{align*}
We thus finish the proof of \Cref{thm:instance equivalence}.
\end{proof}

\bibliography{mybib}

\begin{thebibliography}{61}
\providecommand{\natexlab}[1]{#1}
\providecommand{\url}[1]{\texttt{#1}}
\expandafter\ifx\csname urlstyle\endcsname\relax
  \providecommand{\doi}[1]{doi: #1}\else
  \providecommand{\doi}{doi: \begingroup \urlstyle{rm}\Url}\fi

\bibitem[Agrawal et~al.(2023)Agrawal, Feng, and Tang]{AFT-23}
Shipra Agrawal, Yiding Feng, and Wei Tang.
\newblock Dynamic pricing and learning with bayesian persuasion.
\newblock In Alice Oh, Tristan Naumann, Amir Globerson, Kate Saenko, Moritz Hardt, and Sergey Levine, editors, \emph{Advances in Neural Information Processing Systems 36: Annual Conference on Neural Information Processing Systems 2023, NeurIPS 2023, New Orleans, LA, USA, December 10 - 16, 2023}, 2023.
\newblock URL \url{http://papers.nips.cc/paper\_files/paper/2023/hash/b9c2e8a0bbed5fcfaf62856a3a719ada-Abstract-Conference.html}.

\bibitem[Alonso and C{\^a}mara(2016)]{BC-16}
Ricardo Alonso and Odilon C{\^a}mara.
\newblock Bayesian persuasion with heterogeneous priors.
\newblock \emph{Journal of Economic Theory}, 165:\penalty0 672--706, 2016.

\bibitem[Arieli et~al.(2023)Arieli, Babichenko, Smorodinsky, and Yamashita]{ABSY-23}
Itai Arieli, Yakov Babichenko, Rann Smorodinsky, and Takuro Yamashita.
\newblock Optimal persuasion via bi-pooling.
\newblock \emph{Theoretical Economics}, 18\penalty0 (1):\penalty0 15--36, 2023.

\bibitem[Arunachaleswaran et~al.(2025)Arunachaleswaran, Collina, Roth, and Shi]{ACRS-25}
Eshwar~Ram Arunachaleswaran, Natalie Collina, Aaron Roth, and Mirah Shi.
\newblock An elementary predictor obtaining distance to calibration.
\newblock In \emph{Proceedings of the 2025 Annual ACM-SIAM Symposium on Discrete Algorithms (SODA)}, pages 1366--1370. SIAM, 2025.

\bibitem[Babichenko and Barman(2017)]{BB-17}
Yakov Babichenko and Siddharth Barman.
\newblock Algorithmic aspects of private bayesian persuasion.
\newblock In \emph{8th Innovations in Theoretical Computer Science Conference}, 2017.

\bibitem[Balcan et~al.(2005)Balcan, Blum, Hartline, and Mansour]{BBHM-05}
Maria{-}Florina Balcan, Avrim Blum, Jason~D. Hartline, and Yishay Mansour.
\newblock Mechanism design via machine learning.
\newblock In \emph{46th Annual {IEEE} Symposium on Foundations of Computer Science {(FOCS} 2005), 23-25 October 2005, Pittsburgh, PA, USA, Proceedings}, pages 605--614. {IEEE} Computer Society, 2005.
\newblock \doi{10.1109/SFCS.2005.50}.
\newblock URL \url{https://doi.org/10.1109/SFCS.2005.50}.

\bibitem[Balseiro et~al.(2024)Balseiro, Besbes, and Castro]{BBC-24}
Santiago~R Balseiro, Omar Besbes, and Francisco Castro.
\newblock Mechanism design under approximate incentive compatibility.
\newblock \emph{Operations Research}, 72\penalty0 (1):\penalty0 355--372, 2024.

\bibitem[Blackwell(1953)]{B-53}
David Blackwell.
\newblock Equivalent comparisons of experiments.
\newblock \emph{The annals of mathematical statistics}, pages 265--272, 1953.

\bibitem[B{\l}asiok et~al.(2023)B{\l}asiok, Gopalan, Hu, and Nakkiran]{BGHN-23}
Jaros{\l}aw B{\l}asiok, Parikshit Gopalan, Lunjia Hu, and Preetum Nakkiran.
\newblock A unifying theory of distance from calibration.
\newblock In \emph{Proceedings of the 55th Annual ACM Symposium on Theory of Computing}, pages 1727--1740, 2023.

\bibitem[Blasiok et~al.(2023)Blasiok, Gopalan, Hu, and Nakkiran]{BGHN-23b}
Jaroslaw Blasiok, Parikshit Gopalan, Lunjia Hu, and Preetum Nakkiran.
\newblock When does optimizing a proper loss yield calibration?
\newblock \emph{Advances in Neural Information Processing Systems}, 36:\penalty0 72071--72095, 2023.

\bibitem[Cai et~al.(2012)Cai, Daskalakis, and Weinberg]{CDW-12}
Yang Cai, Constantinos Daskalakis, and S.~Matthew Weinberg.
\newblock Optimal multi-dimensional mechanism design: Reducing revenue to welfare maximization.
\newblock In \emph{53rd Annual {IEEE} Symposium on Foundations of Computer Science, {FOCS} 2012, New Brunswick, NJ, USA, October 20-23, 2012}, pages 130--139. {IEEE} Computer Society, 2012.
\newblock \doi{10.1109/FOCS.2012.88}.
\newblock URL \url{https://doi.org/10.1109/FOCS.2012.88}.

\bibitem[Cai et~al.(2021)Cai, Oikonomou, Velegkas, and Zhao]{COVZ-21}
Yang Cai, Argyris Oikonomou, Grigoris Velegkas, and Mingfei Zhao.
\newblock An efficient $\varepsilon$-bic to bic transformation and its application to black-box reduction in revenue maximization.
\newblock In \emph{Proceedings of the 2021 acm-siam symposium on discrete algorithms (soda)}, pages 1337--1356. SIAM, 2021.

\bibitem[Camara et~al.(2020)Camara, Hartline, and Johnsen]{CHJ-20}
Modibo~K. Camara, Jason~D. Hartline, and Aleck~C. Johnsen.
\newblock Mechanisms for a no-regret agent: Beyond the common prior.
\newblock In Sandy Irani, editor, \emph{61st {IEEE} Annual Symposium on Foundations of Computer Science, {FOCS} 2020, Durham, NC, USA, November 16-19, 2020}, pages 259--270. {IEEE}, 2020.
\newblock \doi{10.1109/FOCS46700.2020.00033}.
\newblock URL \url{https://doi.org/10.1109/FOCS46700.2020.00033}.

\bibitem[Casacuberta et~al.(2024)Casacuberta, Dwork, and Vadhan]{CDV-24}
S{\'\i}lvia Casacuberta, Cynthia Dwork, and Salil Vadhan.
\newblock Complexity-theoretic implications of multicalibration.
\newblock In \emph{Proceedings of the 56th Annual ACM Symposium on Theory of Computing}, pages 1071--1082, 2024.

\bibitem[Chen and Lin(2023)]{CL-23}
Yiling Chen and Tao Lin.
\newblock Persuading a behavioral agent: Approximately best responding and learning.
\newblock \emph{CoRR}, abs/2302.03719, 2023.
\newblock \doi{10.48550/ARXIV.2302.03719}.
\newblock URL \url{https://doi.org/10.48550/arXiv.2302.03719}.

\bibitem[Collina et~al.(2024{\natexlab{a}})Collina, Goel, Gupta, and Roth]{CGGR-24}
Natalie Collina, Surbhi Goel, Varun Gupta, and Aaron Roth.
\newblock Tractable agreement protocols.
\newblock \emph{arXiv preprint arXiv:2411.19791}, 2024{\natexlab{a}}.

\bibitem[Collina et~al.(2024{\natexlab{b}})Collina, Roth, and Shao]{CRS-24}
Natalie Collina, Aaron Roth, and Han Shao.
\newblock Efficient prior-free mechanisms for no-regret agents.
\newblock In \emph{Proceedings of the 25th ACM Conference on Economics and Computation}, pages 511--541, 2024{\natexlab{b}}.

\bibitem[Corrao and Dai(2023)]{CD-23}
Roberto Corrao and Yifan Dai.
\newblock Mediated communication with transparent motives.
\newblock In Kevin Leyton{-}Brown, Jason~D. Hartline, and Larry Samuelson, editors, \emph{Proceedings of the 24th {ACM} Conference on Economics and Computation, {EC} 2023, London, United Kingdom, July 9-12, 2023}, page 489. {ACM}, 2023.
\newblock \doi{10.1145/3580507.3597808}.
\newblock URL \url{https://doi.org/10.1145/3580507.3597808}.

\bibitem[Dagan et~al.(2024)Dagan, Daskalakis, Fishelson, Golowich, Kleinberg, and Okoroafor]{DDFM-24}
Yuval Dagan, Constantinos Daskalakis, Maxwell Fishelson, Noah Golowich, Robert Kleinberg, and Princewill Okoroafor.
\newblock Breaking the ${T}^{2/3}$ barrier for sequential calibration.
\newblock \emph{arXiv preprint arXiv:2406.13668}, 2024.

\bibitem[Dawid(1982)]{D-82}
A~Philip Dawid.
\newblock The well-calibrated bayesian.
\newblock \emph{Journal of the American statistical Association}, 77\penalty0 (379):\penalty0 605--610, 1982.

\bibitem[De~Clippel and Zhang(2022)]{DZ-22}
Geoffroy De~Clippel and Xu~Zhang.
\newblock Non-bayesian persuasion.
\newblock \emph{Journal of Political Economy}, 130\penalty0 (10):\penalty0 2594--2642, 2022.

\bibitem[Dughmi and Xu(2017)]{DX-17}
Shaddin Dughmi and Haifeng Xu.
\newblock Algorithmic persuasion with no externalities.
\newblock In \emph{Proceedings of the 18th ACM Conference on Economics and Computation}, pages 351--368, 2017.

\bibitem[Dworczak and Martini(2019)]{DM-19}
Piotr Dworczak and Giorgio Martini.
\newblock The simple economics of optimal persuasion.
\newblock \emph{Journal of Political Economy}, 127\penalty0 (5):\penalty0 1993--2048, 2019.

\bibitem[Fan et~al.(2023)Fan, Si, and Zhang]{FSZ-23}
Yewen Fan, Nian Si, and Kun Zhang.
\newblock Calibration matters: Tackling maximization bias in large-scale advertising recommendation systems.
\newblock In \emph{The Eleventh International Conference on Learning Representations, {ICLR} 2023, Kigali, Rwanda, May 1-5, 2023}. OpenReview.net, 2023.
\newblock URL \url{https://openreview.net/forum?id=wzlWiO\_WY4}.

\bibitem[Feng et~al.(2022)Feng, Tang, and Xu]{FTX-22}
Yiding Feng, Wei Tang, and Haifeng Xu.
\newblock Online bayesian recommendation with no regret.
\newblock In David~M. Pennock, Ilya Segal, and Sven Seuken, editors, \emph{{EC} '22: The 23rd {ACM} Conference on Economics and Computation, Boulder, CO, USA, July 11 - 15, 2022}, pages 818--819. {ACM}, 2022.
\newblock \doi{10.1145/3490486.3538327}.
\newblock URL \url{https://doi.org/10.1145/3490486.3538327}.

\bibitem[Feng et~al.(2024)Feng, Ho, and Tang]{FHT-24}
Yiding Feng, Chien{-}Ju Ho, and Wei Tang.
\newblock Rationality-robust information design: Bayesian persuasion under quantal response.
\newblock In David~P. Woodruff, editor, \emph{Proceedings of the 2024 {ACM-SIAM} Symposium on Discrete Algorithms, {SODA} 2024, Alexandria, VA, USA, January 7-10, 2024}, pages 501--546. {SIAM}, 2024.
\newblock \doi{10.1137/1.9781611977912.19}.
\newblock URL \url{https://doi.org/10.1137/1.9781611977912.19}.

\bibitem[Foster and Vohra(1998)]{FV-98}
Dean~P Foster and Rakesh~V Vohra.
\newblock Asymptotic calibration.
\newblock \emph{Biometrika}, 85\penalty0 (2):\penalty0 379--390, 1998.

\bibitem[Galperti(2019)]{P-19}
Simone Galperti.
\newblock Persuasion: The art of changing worldviews.
\newblock \emph{American Economic Review}, 109\penalty0 (3):\penalty0 996--1031, 2019.

\bibitem[Gan et~al.(2023)Gan, Han, Wu, and Xu]{GHWX-23}
Jiarui Gan, Minbiao Han, Jibang Wu, and Haifeng Xu.
\newblock Robust stackelberg equilibria.
\newblock In Kevin Leyton{-}Brown, Jason~D. Hartline, and Larry Samuelson, editors, \emph{Proceedings of the 24th {ACM} Conference on Economics and Computation, {EC} 2023, London, United Kingdom, July 9-12, 2023}, page 735. {ACM}, 2023.
\newblock \doi{10.1145/3580507.3597680}.
\newblock URL \url{https://doi.org/10.1145/3580507.3597680}.

\bibitem[Garg et~al.(2024)Garg, Jung, Reingold, and Roth]{GJRR-24}
Sumegha Garg, Christopher Jung, Omer Reingold, and Aaron Roth.
\newblock Oracle efficient online multicalibration and omniprediction.
\newblock In \emph{Proceedings of the 2024 Annual ACM-SIAM Symposium on Discrete Algorithms (SODA)}, pages 2725--2792. SIAM, 2024.

\bibitem[Gentzkow and Kamenica(2016)]{GK-16}
Matthew Gentzkow and Emir Kamenica.
\newblock A rothschild-stiglitz approach to bayesian persuasion.
\newblock \emph{American Economic Review}, 106\penalty0 (5):\penalty0 597--601, 2016.

\bibitem[Gneiting and Katzfuss(2014)]{GK-14}
Tilmann Gneiting and Matthias Katzfuss.
\newblock Probabilistic forecasting.
\newblock \emph{Annual Review of Statistics and Its Application}, 1\penalty0 (1):\penalty0 125--151, 2014.

\bibitem[Gneiting et~al.(2007)Gneiting, Balabdaoui, and Raftery]{GBR-07}
Tilmann Gneiting, Fadoua Balabdaoui, and Adrian~E Raftery.
\newblock Probabilistic forecasts, calibration and sharpness.
\newblock \emph{Journal of the Royal Statistical Society Series B: Statistical Methodology}, 69\penalty0 (2):\penalty0 243--268, 2007.

\bibitem[Gonczarowski and Weinberg(2021)]{GW-21}
Yannai~A Gonczarowski and S~Matthew Weinberg.
\newblock The sample complexity of up-to-$\varepsilon$ multi-dimensional revenue maximization.
\newblock \emph{Journal of the ACM (JACM)}, 68\penalty0 (3):\penalty0 1--28, 2021.

\bibitem[Gopalan et~al.(2022)Gopalan, Kalai, Reingold, Sharan, and Wieder]{GKRSW-22}
Parikshit Gopalan, Adam~Tauman Kalai, Omer Reingold, Vatsal Sharan, and Udi Wieder.
\newblock Omnipredictors.
\newblock In \emph{13th Innovations in Theoretical Computer Science Conference (ITCS 2022)}, pages 79--1. Schloss Dagstuhl--Leibniz-Zentrum f{\"u}r Informatik, 2022.

\bibitem[Gopalan et~al.(2023)Gopalan, Hu, Kim, Reingold, and Wieder]{GHKRW-23}
Parikshit Gopalan, Lunjia Hu, Michael~P Kim, Omer Reingold, and Udi Wieder.
\newblock Loss minimization through the lens of outcome indistinguishability.
\newblock In \emph{14th Innovations in Theoretical Computer Science Conference (ITCS 2023)}, pages 60--1. Schloss Dagstuhl--Leibniz-Zentrum f{\"u}r Informatik, 2023.

\bibitem[Gopalan et~al.(2024)Gopalan, Hu, and Rothblum]{GHR-24}
Parikshit Gopalan, Lunjia Hu, and Guy~N. Rothblum.
\newblock On computationally efficient multi-class calibration.
\newblock In \emph{Proceedings of Thirty Seventh Conference on Learning Theory}, volume 247, pages 1983--2026. PMLR, 2024.

\bibitem[Guo et~al.(2017)Guo, Pleiss, Sun, and Weinberger]{GPSW-17}
Chuan Guo, Geoff Pleiss, Yu~Sun, and Kilian~Q Weinberger.
\newblock On calibration of modern neural networks.
\newblock In \emph{International conference on machine learning}, pages 1321--1330. PMLR, 2017.

\bibitem[Guo and Shmaya(2021)]{GS-21}
Yingni Guo and Eran Shmaya.
\newblock Costly miscalibration.
\newblock \emph{Theoretical Economics}, 16\penalty0 (2):\penalty0 477--506, 2021.

\bibitem[Haghtalab et~al.(2023)Haghtalab, Podimata, and Yang]{HPY-23}
Nika Haghtalab, Chara Podimata, and Kunhe Yang.
\newblock Calibrated stackelberg games: Learning optimal commitments against calibrated agents.
\newblock \emph{Advances in Neural Information Processing Systems}, 36:\penalty0 61645--61677, 2023.

\bibitem[Hartline and Lucier(2010)]{HL-10}
Jason~D. Hartline and Brendan Lucier.
\newblock Bayesian algorithmic mechanism design.
\newblock In Moshe Dror and Greys Sosic, editors, \emph{Proceedings of the Behavioral and Quantitative Game Theory - Conference on Future Directions, {BQGT} '10, Newport Beach, California, USA, May 14-16, 2010}, page 19:1. {ACM}, 2010.
\newblock \doi{10.1145/1807406.1807425}.
\newblock URL \url{https://doi.org/10.1145/1807406.1807425}.

\bibitem[Hartline and Lucier(2015)]{HL-15}
Jason~D Hartline and Brendan Lucier.
\newblock Non-optimal mechanism design.
\newblock \emph{American Economic Review}, 105\penalty0 (10):\penalty0 3102--3124, 2015.

\bibitem[H{\'e}bert-Johnson et~al.(2018)H{\'e}bert-Johnson, Kim, Reingold, and Rothblum]{HKRR-18}
Ursula H{\'e}bert-Johnson, Michael Kim, Omer Reingold, and Guy Rothblum.
\newblock Multicalibration: Calibration for the (computationally-identifiable) masses.
\newblock In \emph{International Conference on Machine Learning}, pages 1939--1948. PMLR, 2018.

\bibitem[Hu and Wu(2024)]{HW-24}
Lunjia Hu and Yifan Wu.
\newblock Predict to minimize swap regret for all payoff-bounded tasks.
\newblock In \emph{2024 IEEE 65th Annual Symposium on Foundations of Computer Science (FOCS)}, pages 244--263. IEEE, 2024.

\bibitem[Kamenica and Gentzkow(2011)]{KG-11}
Emir Kamenica and Matthew Gentzkow.
\newblock Bayesian persuasion.
\newblock \emph{American Economic Review}, 101\penalty0 (6):\penalty0 2590--2615, 2011.

\bibitem[Kim et~al.(2019)Kim, Ghorbani, and Zou]{KGZ-19}
Michael~P Kim, Amirata Ghorbani, and James Zou.
\newblock Multiaccuracy: Black-box post-processing for fairness in classification.
\newblock In \emph{Proceedings of the 2019 AAAI/ACM Conference on AI, Ethics, and Society}, pages 247--254, 2019.

\bibitem[Kleinberg et~al.(2023)Kleinberg, Leme, Schneider, and Teng]{KLST-23}
Bobby Kleinberg, Renato~Paes Leme, Jon Schneider, and Yifeng Teng.
\newblock U-calibration: Forecasting for an unknown agent.
\newblock In \emph{The Thirty Sixth Annual Conference on Learning Theory}, pages 5143--5145. PMLR, 2023.

\bibitem[Kolotilin et~al.(2024)Kolotilin, Li, and Zapechelnyuk]{KLZ-24}
Anton Kolotilin, Hongyi Li, and Andriy Zapechelnyuk.
\newblock On monotone persuasion.
\newblock \emph{arXiv preprint arXiv:2412.14400}, 2024.

\bibitem[Kuleshov et~al.(2018)Kuleshov, Fenner, and Ermon]{KFE-18}
Volodymyr Kuleshov, Nathan Fenner, and Stefano Ermon.
\newblock Accurate uncertainties for deep learning using calibrated regression.
\newblock In \emph{International conference on machine learning}, pages 2796--2804. PMLR, 2018.

\bibitem[Lipnowski and Ravid(2020)]{LR-20}
Elliot Lipnowski and Doron Ravid.
\newblock Cheap talk with transparent motives.
\newblock \emph{Econometrica}, 88\penalty0 (4):\penalty0 1631--1660, 2020.

\bibitem[Noarov et~al.(2023)Noarov, Ramalingam, Roth, and Xie]{NRRX-23}
Georgy Noarov, Ramya Ramalingam, Aaron Roth, and Stephan Xie.
\newblock High-dimensional prediction for sequential decision making.
\newblock \emph{arXiv preprint arXiv:2310.17651}, 2023.

\bibitem[Qiao and Valiant(2021)]{QV-21}
Mingda Qiao and Gregory Valiant.
\newblock Stronger calibration lower bounds via sidestepping.
\newblock In \emph{Proceedings of the 53rd Annual ACM SIGACT Symposium on Theory of Computing}, pages 456--466, 2021.

\bibitem[Qiao and Zhao(2025)]{QZ-25}
Mingda Qiao and Eric Zhao.
\newblock Truthfulness of decision-theoretic calibration measures.
\newblock \emph{arXiv preprint arXiv:2503.02384}, 2025.

\bibitem[Qiao and Zheng(2024)]{QZ-24}
Mingda Qiao and Letian Zheng.
\newblock On the distance from calibration in sequential prediction.
\newblock In \emph{The Thirty Seventh Annual Conference on Learning Theory}, pages 4307--4357. PMLR, 2024.

\bibitem[Rahaman et~al.(2021)]{R-21}
Rahul Rahaman et~al.
\newblock Uncertainty quantification and deep ensembles.
\newblock \emph{Advances in neural information processing systems}, 34:\penalty0 20063--20075, 2021.

\bibitem[Ranjan and Gneiting(2010)]{RG-10}
Roopesh Ranjan and Tilmann Gneiting.
\newblock Combining probability forecasts.
\newblock \emph{Journal of the Royal Statistical Society Series B: Statistical Methodology}, 72\penalty0 (1):\penalty0 71--91, 2010.

\bibitem[Roth and Shi(2024)]{RS-24}
Aaron Roth and Mirah Shi.
\newblock Forecasting for swap regret for all downstream agents.
\newblock In \emph{Proceedings of the 25th ACM Conference on Economics and Computation}, pages 466--488, 2024.

\bibitem[Rothschild and Stiglitz(1978)]{RS-78}
Michael Rothschild and Joseph~E Stiglitz.
\newblock Increasing risk: I. a definition.
\newblock In \emph{Uncertainty in economics}, pages 99--121. Elsevier, 1978.

\bibitem[Xu(2020)]{xu-20}
Haifeng Xu.
\newblock On the tractability of public persuasion with no externalities.
\newblock In \emph{Proceedings of the 14th Annual ACM-SIAM Symposium on Discrete Algorithms}, pages 2708--2727. SIAM, 2020.

\bibitem[Yang and Zhang(2024)]{YZ-24}
Kunhe Yang and Hanrui Zhang.
\newblock Computational aspects of bayesian persuasion under approximate best response.
\newblock In \emph{The Thirty-eighth Annual Conference on Neural Information Processing Systems}, 2024.

\bibitem[Zhao et~al.(2021)Zhao, Kim, Sahoo, Ma, and Ermon]{ZKS-21}
Shengjia Zhao, Michael Kim, Roshni Sahoo, Tengyu Ma, and Stefano Ermon.
\newblock Calibrating predictions to decisions: A novel approach to multi-class calibration.
\newblock \emph{Advances in Neural Information Processing Systems}, 34:\penalty0 22313--22324, 2021.

\end{thebibliography}

\appendix

\section{A Win-Win Example}
\label{apx:win-win}
\label{apx:win-win example}
As the ECE budget $\caliBudget$ increases, the principal's expected utility under the optimal $\caliBudget$-calibrated predictor weakly increases, since the space of feasible predictors expands. However, from the agent's perspective, the prediction becomes less calibrated, which may negatively impact their expected utility. Despite this, due to the strategic interaction between the principal and the agent, the intuition that the agent's expected utility under the optimal $\caliBudget$-calibrated predictor weakly decreases as $\caliBudget$ increases is not necessarily correct. In this section, we present a ``win-win'' instance of our persuasive calibration problem, where both the principal's and the agent's expected utility under the optimal $\caliBudget$-calibrated predictor increase within a certain range of the ECE budget and are strictly higher than their expected utilities when the ECE budget is set to zero. In this section, we focus on $\ell_1$-norm ECE.

\begin{example}[Win-win instance]
\label{example:win win example}
    There are $\numData = 3$ events, with expected outcomes $\truemean_1 = 0.10001$, $\truemean_2 = 0.85$, $\truemean_3 = 1$, and event prior $\prior_1 = 0.25$, $\prior_2 = 0.5$, $\prior_3 = 0.25$. The agent has $\numAction = 4$ actions, whose induced utilities $\receiverU(\action,\outcome)$ for each outcome $\outcome\in\outcomeSpace$ satisfy 
        \begin{table}[H]
            \centering
            \begin{tabular}{c|c|c|c|c}
              & $\action = 1$   &  $\action = 2$ & $\action = 3$ & $\action = 4$ \\
              \hline
              $\outcome = 1$ & -9.9999  & 0 & 0.1 & 10 \\
              $\outcome = 0$ & 0.0001  & 0 & -0.9 & $-\infty$ \\
            \end{tabular}
            \caption{Agent's utility $\receiverU(\action, \outcome)$ in \Cref{example:win win example}.}
        \end{table}
        \noindent
        The principal's utility $\senderU_i(\action,\outcome)$ only depends on the agent's action. We shorthand it using $\senderU(\action)$ and define it as 
        \begin{table}[H]
            \centering
            \begin{tabular}{c|c|c|c}
             $\action = 1$   &  $\action = 2$ & $\action = 3$ & $\action = 4$ \\
              \hline
              5  & 0 & 1 & 2 \\
            \end{tabular}
            \caption{Principal's event-independent utility $\senderU(\action)$ in \Cref{example:win win example}.}
        \end{table}
        \noindent 
        By construction, the principal's indirect utility $\indirectsenderU(\prediction)$ of a given prediction $\prediction$ is 
        \begin{table}[H]
            \centering
            \begin{tabular}{c|c|c|c}
             $0\leq\prediction \leq \prediction\primed$   &  $\prediction\primed<\prediction<\prediction\doubleprimed$ & $\prediction\doubleprimed\leq \prediction< 1$ & $\prediction = 1$ \\
              \hline
              5  & 0 & 1 & 2 \\
            \end{tabular}
            \caption{Principal's indirect utility $\indirectsenderU(\prediction)$ in \Cref{example:win win example}.}
        \end{table}
        \noindent
        where $\prediction\primed = 0.00001$ (resp.\ $\prediction\doubleprimed = 0.9$) is the prediction such that the agent is indifferent between actions 1 and 2 (resp.\ 2 and 3), respectively.
\end{example}
We first provide some intuition behind the construction of \Cref{example:win win example}. In this instance, the agent has four actions. The second action is a safe action that generates zero utility for both outcome realizations. The first and fourth actions are risky actions, such that the agent will select the first action only when outcome zero ($\outcome = 1$) is predicted to occur almost for sure and the fourth action only when outcome one ($\outcome = 1$) is predicted to occur for sure. Otherwise, the agent selects either the second or third action. Additionally, the fourth action yields significantly higher utility under outcome one ($\outcome = 1$) than any other action-outcome pair. Hence, the agent achieves good utility if he occasionally receive a prediction $\prediction$ equal to one, and when he does, it is perfectly calibrated.

On the other hand, the principal prefers the agent to take the first action, which yields an induced utility of $\senderU(1) = 5$, 2.5 times larger than her second favorite action (i.e., the fourth action) with $\senderU(4) = 2$. Moreover, she also seeks to avoid the agent selecting the second action, which has an induced utility of zero.

We next describe the optimal $\caliBudget$-calibrated predictor and its induced expected utilities for two players under different ECE budgets. Their optimality can be verified analytically using \Cref{prop:opt verify} or numerically using the algorithm developed in \Cref{thm:opt general}, and is omitted. 

\begin{itemize}

\item \textbf{Case 1. ECE budget $\caliBudget\in[0,0.025]$} In this range of ECE budget, the optimal $\caliBudget$-calibrated predictor $\optpredictor$ has three possible predictions. When the realized event is 2 or 3, a deterministic prediction $\prediction = \prediction\doubleprimed$ is generated. When the realized event is 1, predictions $\prediction = \prediction\primed$ and $\prediction = \truemean_1$ are generated with probability $40\caliBudget$ and $1-40\caliBudget$, respectively.\footnote{Both definitions of $\prediction\primed$ and $\prediction\doubleprimed$ are given in \Cref{example:win win example}.} Specifically, $\optpredictor_1(\prediction\primed) = 40\caliBudget$, $\optpredictor_1(\truemean_1) = 1 - 40\caliBudget$, $\optpredictor_2(\prediction\doubleprimed) = \optpredictor_3(\prediction\doubleprimed) = 1$. By construction, only prediction $\prediction = \prediction\primed$ induces positive calibration error. The principal and agent's expected utilities are $50 \caliBudget + 0.75$ and $-9.999\caliBudget$, respectively. 

\item \textbf{Case 2. ECE budget $\caliBudget\in[0.025, 0.05]$} In this range of ECE budget, the optimal $\caliBudget$-calibrated predictor $\optpredictor$ has three possible predictions. When the realized event is 1, a deterministic prediction $\prediction = \prediction\primed$ is generated. When the realized event is 2, a deterministic prediction $\prediction = \prediction\doubleprimed$ is generated.
When the realized event is 3, predictions $\prediction = \prediction\doubleprimed$ and $\prediction = 1$ are generated with probability $2-40\caliBudget$ and $40\caliBudget-1$, respectively. Specifically, $\optpredictor_1(\prediction\primed) = 1$, $\optpredictor_2(\prediction\doubleprimed) = 1$, $ \optpredictor_3(\prediction\doubleprimed) = 2-40\caliBudget$, and $\optpredictor_3(1) = 40\caliBudget-1$. By construction, predictions $\prediction = \prediction\primed$ and $\prediction=\prediction\doubleprimed$ induce positive calibration error. The principal and agent's expected utilities are $10\caliBudget + 1.75$ and $99\caliBudget-2.72498$, respectively.

\item \textbf{Case 3. ECE budget $\caliBudget\in(0.05, 0.1]$} In this range of ECE budget, the optimal $\caliBudget$-calibrated predictor $\optpredictor$ has three possible predictions. When the realized event is 1, a deterministic prediction $\prediction = \prediction\primed$ is generated. When the realized event is 2, predictions $\prediction = \prediction\doubleprimed$ and $\prediction = 1$ are generated with probability $2-20\caliBudget$ and $20\caliBudget-1$, respectively.
When the realized event is 3, a deterministic prediction $\prediction = 1$ is generated. Specifically, $\optpredictor_1(\prediction\primed) = 1$, $\optpredictor_2(\prediction\doubleprimed) = 2-20\caliBudget$, $ \optpredictor_2(1) = 20\caliBudget-1$, and $\optpredictor_3(1) = 1$. By construction, all three predictions induce positive calibration error. The principal and agent's expected utilities are $10\caliBudget + 1.75$ and $-\infty$, respectively. 

\item \textbf{Case 4. ECE budget $\caliBudget\in(0.1,0.45]$} In this range of ECE budget, the optimal $\caliBudget$-calibrated predictor $\optpredictor$ has two possible predictions. When the realized event is 1, a deterministic prediction $\prediction = \prediction\primed$ is generated. When the realized event is 2, predictions $\prediction = \prediction\primed$ and $\prediction = \prediction= 1$ are generated with probability $2.85718\caliBudget-0.28572$ and $1.28572 - 2.85718 \caliBudget$, respectively.\footnote{In Cases 4 and 5, we round the number after the fifth decimal place.} When the realized event is 3, a deterministic prediction $\prediction = 1$ is generated. Specifically, $\optpredictor_1(\prediction\primed) = 1$, $ \optpredictor_2(\prediction\primed) = 1.28572 - 2.85718 \caliBudget$,
$\optpredictor_2(1) = 2.85718\caliBudget-0.28572$, and 
$\optpredictor_3(1) = 1$. By construction, both two predictions induce positive calibration error. The principal and agent's expected utilities are $4.28578\caliBudget+2.32142$ and $-\infty$, respectively. 

\item \textbf{Case 5. ECE budget $\caliBudget\in[0.45,0.7)$} In this range of ECE budget, the optimal $\caliBudget$-calibrated predictor $\optpredictor$ has two possible predictions. When the realized event is 1 or 2, a deterministic prediction $\prediction = \prediction\primed$ is generated. When the realized event is 3, predictions $\prediction = \prediction\primed$ and $\prediction = \prediction= 1$ are generated with probability $4\caliBudget-1.8$ and $2.8- 4\caliBudget$, respectively.  Specifically, $\optpredictor_1(\prediction\primed) = \optpredictor_2(\prediction\primed) = 1$, $ \optpredictor_3(\prediction\primed) = 4\caliBudget-1.8$, and  
$\optpredictor_3(1) = 2.8-4\caliBudget$. By construction, prediction $\prediction=\prediction\primed$ induces positive calibration error. The principal and agent's expected utilities are $3\caliBudget+2.9$ and $-\infty$, respectively. 

\item \textbf{Case 6. ECE budget $\caliBudget\in[0.7,\infty)$} In this range of ECE budget, the optimal $\caliBudget$-calibrated predictor $\optpredictor$ generates prediction $\prediction =\prediction\primed$ regardless of the realized event. Specifically, $\optpredictor_1(\prediction\primed) = \optpredictor(\prediction\primed) = \optpredictor_3(\prediction\primed) = 1$. The principal and agent's expected utilities are $5$ and $-\infty$, respectively. 
\end{itemize}

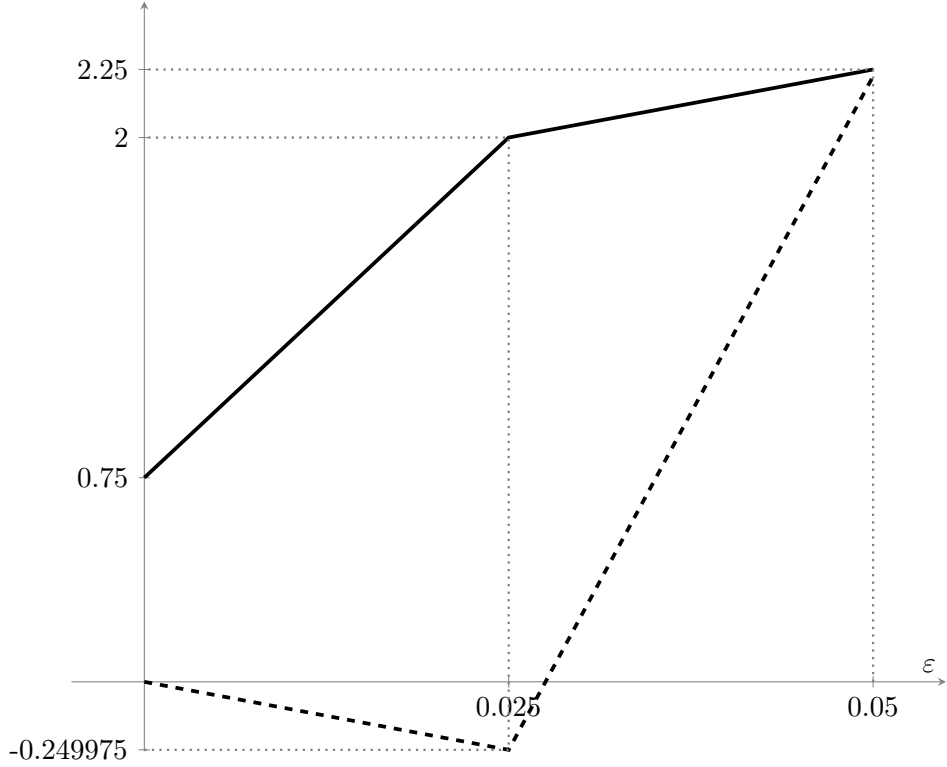
\begin{figure}
    \centering
    \begin{tikzpicture}
\begin{axis}[
axis line style=gray,
axis lines=middle,
        ytick = {-0.249975,0.75, 2, 2.25},
        yticklabels = {-0.249975, 0.75, 2, 2.25},
        xtick = {0.025, 0.05},
        xticklabels = {0.025, 0.05},
xlabel = {$\caliBudget$},
    scaled x ticks = false, % Disable automatic scaling
xmin=-0.005,xmax=0.055,ymin=-0.255,ymax=2.5,
width=0.8\textwidth,
height=0.7\textwidth]

\addplot[line width=0.5mm, dashed] coordinates {
(0.0, 0)
% (0.0005, -0.004999550000000012)
% (0.001, -0.009999100000000018)
% (0.0015, -0.01499865000000001)
% (0.002, -0.019998200000000015)
% (0.0025, -0.024997750000000013)
% (0.003, -0.02999730000000001)
% (0.0035, -0.03499685000000002)
% (0.004, -0.03999640000000001)
% (0.0045000000000000005, -0.044995950000000014)
% (0.005, -0.04999550000000002)
% (0.0055, -0.05499505000000001)
% (0.006, -0.0599946)
% (0.006500000000000001, -0.06499415000000003)
% (0.007, -0.06999370000000002)
% (0.0075, -0.07499325000000001)
% (0.008, -0.07999280000000003)
% (0.0085, -0.08499235000000002)
% (0.009000000000000001, -0.08999190000000001)
% (0.0095, -0.09499145)
% (0.01, -0.09999100000000002)
% (0.0105, -0.10499055000000002)
% (0.011, -0.10999010000000001)
% (0.0115, -0.11498965000000003)
% (0.012, -0.11998920000000002)
% (0.0125, -0.12498875000000001)
% (0.013000000000000001, -0.12998830000000003)
% (0.0135, -0.13498785000000002)
% (0.014, -0.13998740000000004)
% (0.0145, -0.14498695)
% (0.015, -0.14998650000000002)
% (0.0155, -0.15498605000000001)
% (0.016, -0.15998560000000003)
% (0.0165, -0.16498515000000002)
% (0.017, -0.16998470000000004)
% (0.0175, -0.17498425)
% (0.018000000000000002, -0.17998380000000003)
% (0.0185, -0.18498335000000002)
% (0.019, -0.1899829)
% (0.0195, -0.19498245000000003)
% (0.02, -0.19998200000000002)
% (0.0205, -0.20498155)
% (0.021, -0.2099811)
% (0.021500000000000002, -0.21498065000000002)
% (0.022, -0.21998020000000001)
% (0.0225, -0.22497975000000003)
% (0.023, -0.22997930000000003)
% (0.0235, -0.23497885)
% (0.024, -0.2399784)
% (0.0245, -0.24497795000000003)
(0.025, -0.2499775)
% (0.025500000000000002, -0.20047749999999998)
% (0.026000000000000002, -0.15097749999999996)
% (0.0265, -0.10147749999999997)
% (0.027, -0.051977499999999954)
% (0.0275, -0.002477499999999966)
% (0.028, 0.04702250000000002)
% (0.0285, 0.09652250000000007)
% (0.029, 0.14602250000000006)
% (0.029500000000000002, 0.1955225)
% (0.03, 0.24502250000000003)
% (0.0305, 0.2945225000000001)
% (0.031, 0.3440225)
% (0.0315, 0.39352250000000005)
% (0.032, 0.4430225000000001)
% (0.0325, 0.49252250000000003)
% (0.033, 0.5420225000000001)
% (0.0335, 0.5915225000000002)
% (0.034, 0.6410224999999999)
% (0.0345, 0.6905224999999999)
% (0.035, 0.7400225)
% (0.035500000000000004, 0.7895225000000001)
% (0.036000000000000004, 0.8390225000000001)
% (0.0365, 0.8885225000000001)
% (0.037, 0.9380225)
% (0.0375, 0.9875225000000001)
% (0.038, 1.0370225)
% (0.0385, 1.0865225)
% (0.039, 1.1360225000000002)
% (0.0395, 1.1855225)
% (0.04, 1.2350225)
% (0.0405, 1.2845225)
% (0.041, 1.3340225)
% (0.0415, 1.3835225000000002)
% (0.042, 1.4330225000000003)
% (0.0425, 1.4825225)
% (0.043000000000000003, 1.5320224999999998)
% (0.043500000000000004, 1.5815225000000002)
% (0.044, 1.6310224999999998)
% (0.0445, 1.6805225000000001)
% (0.045, 1.7300225)
% (0.0455, 1.7795225)
% (0.046, 1.8290225000000002)
% (0.0465, 1.8785224999999999)
% (0.047, 1.9280225000000002)
% (0.0475, 1.9775225)
% (0.048, 2.0270225)
% (0.0485, 2.0765225)
% (0.049, 2.1260225)
% (0.0495, 2.1755225)
(0.05, 2.2250225)
};

\addplot[line width=0.5mm] coordinates {
(0.0, 0.75)
% (0.0005, 0.775)
% (0.001, 0.8)
% (0.0015, 0.825)
% (0.002, 0.85)
% (0.0025, 0.875)
% (0.003, 0.9)
% (0.0035, 0.925)
% (0.004, 0.95)
% (0.0045000000000000005, 0.975)
% (0.005, 1.0)
% (0.0055, 1.025)
% (0.006, 1.05)
% (0.006500000000000001, 1.075)
% (0.007, 1.1)
% (0.0075, 1.125)
% (0.008, 1.15)
% (0.0085, 1.175)
% (0.009000000000000001, 1.2)
% (0.0095, 1.225)
% (0.01, 1.25)
% (0.0105, 1.275)
% (0.011, 1.3)
% (0.0115, 1.325)
% (0.012, 1.35)
% (0.0125, 1.375)
% (0.013000000000000001, 1.4)
% (0.0135, 1.425)
% (0.014, 1.45)
% (0.0145, 1.475)
% (0.015, 1.5)
% (0.0155, 1.525)
% (0.016, 1.55)
% (0.0165, 1.575)
% (0.017, 1.6)
% (0.0175, 1.625)
% (0.018000000000000002, 1.65)
% (0.0185, 1.675)
% (0.019, 1.7)
% (0.0195, 1.725)
% (0.02, 1.75)
% (0.0205, 1.775)
% (0.021, 1.8)
% (0.021500000000000002, 1.825)
% (0.022, 1.85)
% (0.0225, 1.875)
% (0.023, 1.9)
% (0.0235, 1.925)
% (0.024, 1.95)
% (0.0245, 1.975)
(0.025, 2.0)
% (0.025500000000000002, 2.005)
% (0.026000000000000002, 2.01)
% (0.0265, 2.015)
% (0.027, 2.02)
% (0.0275, 2.025)
% (0.028, 2.03)
% (0.0285, 2.035)
% (0.029, 2.04)
% (0.029500000000000002, 2.045)
% (0.03, 2.05)
% (0.0305, 2.055)
% (0.031, 2.06)
% (0.0315, 2.065)
% (0.032, 2.07)
% (0.0325, 2.075)
% (0.033, 2.08)
% (0.0335, 2.085)
% (0.034, 2.09)
% (0.0345, 2.095)
% (0.035, 2.1)
% (0.035500000000000004, 2.105)
% (0.036000000000000004, 2.11)
% (0.0365, 2.115)
% (0.037, 2.12)
% (0.0375, 2.125)
% (0.038, 2.13)
% (0.0385, 2.135)
% (0.039, 2.14)
% (0.0395, 2.145)
% (0.04, 2.15)
% (0.0405, 2.155)
% (0.041, 2.16)
% (0.0415, 2.165)
% (0.042, 2.17)
% (0.0425, 2.175)
% (0.043000000000000003, 2.18)
% (0.043500000000000004, 2.185)
% (0.044, 2.19)
% (0.0445, 2.195)
% (0.045, 2.2)
% (0.0455, 2.205)
% (0.046, 2.21)
% (0.0465, 2.215)
% (0.047, 2.22)
% (0.0475, 2.225)
% (0.048, 2.23)
% (0.0485, 2.235)
% (0.049, 2.24)
% (0.0495, 2.245)
(0.05, 2.25)
};

\addplot[dotted, gray, line width=0.3mm] coordinates {(0, -0.249975) (0.025, -0.249975) (0.025, 2) (0, 2)};
\addplot[dotted, gray, line width=0.3mm] coordinates {(0.05, 0) (0.05, 2.25) (0, 2.25)};

\end{axis}

\end{tikzpicture}
    \caption{The solid (resp.\ dashed) line is the expected utility of the principal (resp.\ agent) under the optimal $\caliBudget$-calibrated predictor with ECE budget $\caliBudget\in[0, 0.05]$ in \Cref{example:win win example}.}
    \label{fig:win-win example}
\end{figure}

In summary, under the optimal $\caliBudget$-calibrated predictor, while the principal's expected utility is always weakly increasing in the ECE budget $\caliBudget$, the agent's expected utility first decreases for $\caliBudget \in [0, 0.025]$, then increases for $\caliBudget \in [0.025, 0.05]$, and finally drops to negative infinity. Among all ECE budgets $\caliBudget$, the highest expected utility of the agent as well as the highest expected welfare are attained at $\caliBudget = 0.05$. 
Also see \Cref{fig:win-win example} for an illustration.

The non-monotonicity aligns with the aforementioned intuition behind the construction of \Cref{example:win win example}. In particular, when the ECE budget $\caliBudget$ is smaller than 0.025, the optimal predictor only miscalibrates the prediction under the first event while fully mixing the second and third events with a calibrated prediction. Since the second and third events are mixed, the agent selects the third action and never chooses the fourth action, which is risky but high-payoff.  
In contrast, when the ECE budget $\caliBudget$ is between 0.025 and 0.05, the optimal predictor not only miscalibrates the prediction under the first event but also miscalibrates the prediction under the second event. Moreover, when the third event is realized, there is some probability that the predictor generates a calibrated prediction $\prediction = 1$, in which case the agent knows that outcome one occurs for sure, takes the fourth action, and enjoys its high payoff.  
Finally, when the ECE budget $\caliBudget$ exceeds 0.05, the prediction $\prediction = 1$ generated by the optimal predictor becomes non-calibrated, causing the agent who takes the fourth action to receive a negative infinite payoff.

\section{The Failure of Two Natural Programs for FPTAS Construction}
\label{apx:Failure}

In this section, we present the failure of some natural attempts to formulate the principal's problem as a tractable program. 

\xhdr{Optimizing over the space  $\predicSpace_{(\caliBudget, \tnorm)}$}
First, we can focus on program that only involves the predictor $\predictor$ as the variables:
\begin{align}
    % \label{eq:decoupled opt general}
    \arraycolsep=5.4pt\def\arraystretch{1}
    % \tag{$\textsc{LP-TwoStep}^+$}
    &\begin{array}{rlll}
    \max
    \limits_{\predictor\ge 0} 
    ~ &
    \displaystyle 
    \sum\nolimits_{i\in[\numData]}\prior_i \cdot 
    \int_0^1\predictor_i(\prediction)\indirectsenderU_i(\prediction)~\d  \prediction
    \quad & \text{s.t.} &
    \vspace{1mm}
    \\
    & 
    \displaystyle
    \left(\int_0^1 \sum\nolimits_{i\in[\numData]} \prior_i \predictor_i(\prediction) \left|\prediction - \frac{\sum\nolimits_{i\in[\numData]} \prior_i \predictor_i(\prediction) \truemean_i}{\sum\nolimits_{i\in[\numData]} \prior_i \predictor_i(\prediction)} \right|^\normexponent\, \d\prediction\right)^{\frac{1}{\normexponent}} \le\caliBudget,
    & 
    \vspace{1mm}
    \\
    & 
    \displaystyle
    \int_0^1 \predictor_i(\prediction)\, \d\prediction = 1,
    & i\in[\numData]
    \vspace{1mm}
    \end{array}
\end{align}
where the first constraint is due to the definition of $\tnorm$-norm ECE.
It is easy to see that the above program is infinite-dimensional and it is a non-linear program. 
It remains unclear if this program is convex and if there exists efficient algorithm that can solve the above program. 

\xhdr{Optimizing over fully decoupled event-dependent post-processing plan}
The next natural program is considering the following  event-dependent post-processing plan $\miscali_i(\bayesprediction, \prediction)$ that miscalibrates a true expected outcome $\bayesprediction$ to a prediction $\prediction$.
However, to ensure that $\miscali_i(\bayesprediction, \prediction)$ indeed represents miscalibrating the true expected outcome $\bayesprediction$ to $\prediction$, we need to ensure that for every $\miscali_i(\bayesprediction, \prediction) > 0$, we must have $\truePosterior(\prediction) = \bayesprediction$. 
That is, we must have the following condition: for every $\miscali_i(\bayesprediction, \prediction) > 0$, we have 
\begin{align*}
    \bayesprediction = 
    \frac{ \sum\nolimits_{i\in[\numData]} \prior_i \int_0^1\miscali_i(\bayesprediction', \prediction)\,\d\bayesprediction' \cdot \truemean_i}{  \sum\nolimits_{i\in[\numData]} \prior_i \int_0^1\miscali_i(\bayesprediction', \prediction)\,\d\bayesprediction'}~.
\end{align*}
Equivalently, we can express this condition as follows: for every $\bayesprediction, \prediction\in[0, 1]$, we have
\begin{align*}
    \left(\bayesprediction - 
    \frac{ \sum\nolimits_{i\in[\numData]} \prior_i \int_0^1\miscali_i(\bayesprediction', \prediction)\,\d\bayesprediction' \cdot \truemean_i}{  \sum\nolimits_{i\in[\numData]} \prior_i \int_0^1\miscali_i(\bayesprediction', \prediction)\,\d\bayesprediction'}\right)
    \cdot \sum\nolimits_{i\in[\numData]} \prior_i\miscali_i(\bayesprediction, \prediction) = 0~.
\end{align*}
With the above observation, we can generalize \ref{eq:decoupled opt} to the following program that involves with variables $(\miscali_i(\cdot, \cdot))_{i\in[\numData]}$:
\begin{align}
    \arraycolsep=5.4pt\def\arraystretch{1}
    % \tag{$\decoupledprog$}
    % \tag{$\textsc{LP-TwoStep}$}
    &\begin{array}{rlll}
    \max
    \limits_{\miscali_i \ge 0} 
    ~ &
    \displaystyle 
    \sum\nolimits_{i\in[\numData]}\prior_i\cdot 
    \int_0^1\int_0^1
    \miscali_i(\bayesprediction, \prediction) \cdot \indirectsenderU_i(\prediction)~\d  \prediction \d  \bayesprediction
    \quad & \text{s.t.} &
    \vspace{1mm}
    \\
    & 
    \displaystyle 
    \sum\nolimits_{i\in[\numData]}\prior_i\cdot \int_0^1\int_0^1
    \miscali_i(\bayesprediction, \prediction) \cdot |\bayesprediction-\prediction|^\normexponent~ \d\prediction \d  \bayesprediction
    \le \caliBudget^\normexponent
    & 
    \vspace{1mm}
    \\
    & 
    \displaystyle 
    \left(\bayesprediction - 
    \frac{ \sum\nolimits_{i\in[\numData]} \prior_i \int_0^1\miscali_i(\bayesprediction', \prediction)\,\d\bayesprediction' \cdot \truemean_i}{  \sum\nolimits_{i\in[\numData]} \prior_i \int_0^1\miscali_i(\bayesprediction', \prediction)\,\d\bayesprediction'}\right)
    \cdot \sum\nolimits_{i\in[\numData]} \prior_i\miscali_i(\bayesprediction, \prediction) = 0
    & \bayesprediction, \prediction\in[0, 1], 
    \vspace{1mm}
    \\
    & 
    \displaystyle 
    \int_0^1\int_0^1
    \miscali_i(\bayesprediction, \prediction) ~\d  \prediction \d  \bayesprediction = 1
    & i\in[\numData], 
    \vspace{1mm}
    \end{array}
\end{align}
Clearly, the second constraint here is not a linear constraint, and it remains unclear if there exists an efficient solution to solve the above program.

\end{document}